\documentclass[a4paper,draftcls,12pt,onecolumn,journal]{IEEEtran}
\ifCLASSINFOpdf
\else
\fi
\usepackage[cmex10]{amsmath}
\usepackage{amsthm}
\usepackage{multirow,tabularx}
\usepackage{setspace} 
\usepackage[CaptionAfterwards]{fltpage}
\newtheorem{prop}{Proposition}
\theoremstyle{remark}

\usepackage{algorithm,latexsym,mathrsfs,graphicx}
\usepackage{algpseudocode}

\usepackage{ccaption}

\usepackage[normalem]{ulem}
\usepackage[font=normalsize,labelfont=sf,textfont=sf]{caption}
\DeclareCaptionLabelFormat{andtable}{#1~#2  \&  \tablename~\thetable}
\AtBeginCaption{\doublespacing}

\makeatletter
\newcommand\footnoteref[1]{\protected@xdef\@thefnmark{\ref{#1}}\@footnotemark}
\makeatother

\usepackage{array}

\ifCLASSOPTIONcompsoc
  \usepackage[caption=false,font=normalsize,labelfont=sf,textfont=sf]{subfig}
\else
  \usepackage[caption=false,font=normalsize,labelfont=sf,textfont=sf]{subfig}
\fi

\begin{document}

\title{Sparsity-aware Possibilistic Clustering Algorithms}
%
%
%
%

\DeclareRobustCommand*{\IEEEauthorrefmark}[1]{%
  \raisebox{0pt}[0pt][0pt]{\textsuperscript{\footnotesize\ensuremath{#1}}}}

%

\author{\IEEEauthorblockN{Spyridoula~D.~Xenaki\IEEEauthorrefmark{1,}\IEEEauthorrefmark{2},
Konstantinos~D.~Koutroumbas\IEEEauthorrefmark{1} and
Athanasios~A.~Rontogiannis\IEEEauthorrefmark{1}}

\IEEEauthorblockA{\IEEEauthorrefmark{1}{\normalsize Institute for Astronomy, Astrophysics, Space Applications and Remote Sensing (IAASARS), National Observatory of Athens, Penteli, GR-15236 Greece}}\\
\IEEEauthorblockA{\IEEEauthorrefmark{2}{\normalsize Department of Informatics and Telecommunications, National \& Kapodistrian University of Athens, GR-157 84, Ilissia, Greece}}}

\IEEEtitleabstractindextext{%
\begin{abstract}
In this paper two novel possibilistic clustering algorithms are presented, which utilize the concept of sparsity. The first one, called sparse possibilistic c-means, exploits sparsity and can deal well with closely located clusters that may also be of significantly different densities. The second one, called sparse adaptive possibilistic c-means, is an extension of the first, where now the involved parameters are dynamically adapted. The latter can deal well with even more challenging cases, where, in addition to the above, clusters may be of significantly different variances. More specifically, it provides improved estimates of the cluster representatives, while, in addition, it has the ability to estimate the actual number of clusters, given an overestimate of it. Extensive experimental results on both synthetic and real data sets support the previous statements.
\end{abstract}

}

\maketitle

\IEEEdisplaynontitleabstractindextext

%
\IEEEpeerreviewmaketitle

\section{Introduction}
%
%
%
%
%
%
%
%
%
%
%

\IEEEPARstart{C}{lustering} is a well established data analysis method that has been extensively used in various applications during the last decades. It is applied on a certain set of entities and it aims at grouping ``similar'' entities to the same groups (clusters) (e.g. \cite{Theo09}). In most of the practical applications, each entity is represented by a set of $l$ measurements, which form its corresponding $l$-dimensional {\em feature vector}. Equivalently, each entity is represented by a point (vector) in the $l$-dimensional space. The set of all feature vectors (also called {\it data vectors}) is called {\em data set}.

A major effort in the clustering bibliography has been devoted to the identification of compact and hyperellipsoidally shaped clusters. Usually, each such cluster is represented by a vector called {\em cluster representative} or simply {\em representative}, which lies in the same $l$-dimensional space with the data and it is desirable to be located to the ``center'' of the cluster. 
One way to achieve this is to initialize the representatives at some (e.g. random) locations and gradually move them to the centers of the clusters formed by the data vectors. This is usually carried out via algorithms that iteratively optimize suitably defined cost functions, called {\em cost function optimization clustering algorithms}. Celebrated algorithms of this kind are (a) the {\em k-means}, e.g. \cite{Hart79}, where each data vector belongs exclusively to a single cluster, (b) the {\em fuzzy c-means} ({\em FCM}), e.g. \cite{Bezd80},\cite{Bezd81}, where each data vector is shared among two or more clusters and (c) {\em possibilistc c-means} algorithms ({\em PCMs}), e.g. \cite{Kris93}, \cite{Kris96}, \cite{Pal05}, \cite{Yang05}, \cite{Tree13}, \cite{Theo09}, where the {\em compatibility} of each data vector with the clusters is considered.

Some significant features that both the k-means and FCM share are: (a) the interrelation of the updating equations of the representatives, (b) the requirement for a priori knowledge of the exact number of clusters $m$ underlying in the data set, (c) the {\em imposition} of a clustering structure on the data set\footnote{In the sense that the algorithms will split the data set to $m$ distinct clusters irrespectively of the actual number of clusters that underlie in the data set.} and (d) the vulnerability to noisy data and outliers. In contrast to the above, in PCMs the updating of representatives is carried out independently from each other and each representative is moved towards its closest physical cluster. Thus, PCMs do not impose a clustering structure on the data set, in the sense that they will not necessarily end up with $m$ distinct clusters. Actually, only a crude a priori knowledge of the exact number of clusters is required. In the case where $m$ is less than the actual number of clusters the algorithm will identify at least some physical clusters, while in the opposite case, it has the ability to recover all physical clusters with some duplicates \cite{Barn96}. Finally, PCMs are more robust to noisy data or outliers \cite{Pal05}. However, PCMs are sensitive to the values of some specific parameters, whose choice is not always obvious.

In the present work, we focus on PCM. More specifically, we extent the classical PCM algorithm, proposed in \cite{Kris96}, in two stages. First, given that, in practice, each data vector is compatible with only a few or even {\it none} clusters, a suitable sparsity constraint is imposed on the vector containing the degrees of compatibility of each data vector with the clusters, giving rise to the Sparse PCM (SPCM) algorithm. SPCM exhibits increased immunity to data points that may be considered as noise or outliers by not allowing them, in principle, to contribute to the estimation of the cluster representatives. As a consequence, SPCM concludes to more accurate estimates for the cluster representatives, especially in noisy enviroments. Moreover, in difficult cases, where the physical clusters underlying in the data set under study are very closely located to each other, SPCM has the ability to allow only the data points that are very close to the current location of the representatives to contribute to the estimation of the next location of the latter. As a result, SPCM is, in principle, capable of identifying very closely located clusters of possibly various densities. However, the requirement of the estimation of the specific parameters involved in all PCMs still remains. 

It is worth noting that the proposed method is not the only one that introduces the sparsity idea in clustering. Other methods that introduce sparsity in the, so-called, outlier domain have also been proposed in the past (e.g. \cite{Gian12}, \cite{Yang08}). Also in \cite{Inok07}, \cite{Hama12}, two variants of possibilistic clustering that impose sparsity constraints, adopting the $l_1$ norm, are proposed. In \cite{Hama12} the clusters are recovered in a sequential manner, in contrast to \cite{Inok07}, where clusters are recovered simultaneously.

In order to deal with the problem of the estimation of the parameters involved in PCMs, the SPCM is further extended using the rationale proposed in \cite{Xen15}, based on which these parameters are properly adjusted during the execution of the algorithm. Such an extension gives rise to the so called Sparse Adaptive PCM (SAPCM) algorithm\footnote{A preliminary version of SAPCM is presented in \cite{Xen14}.}. A consequence of this parameter adjustment is that, given an overestimate of the true number of clusters, the algorithm has (in principle) the ability to reduce it gradually towards the true number of clusters, i.e., the algorithm is equipped with the ability to estimate by itself the actual number of clusters as well as the clusters themselves.

The rest of the paper is organized as follows. In Section II, a brief description of PCM algorithms is given. In Section III, the proposed Sparse PCM (SPCM) clustering algorithm is fully presented, whereas in Section IV the new SAPCM clustering algorithm is described and its properties are analyzed. In Section V, the performance of both SPCM and SAPCM is tested against several related state-of-the-art algorithms. Finally, concluding remarks are provided in Section VI. 
 
%
%
%
%
%

\section{A Brief Review of PCM}

Let $X=\{\mathbf{x}_i \in \Re^\ell, i=1,...,N\}$ be a set of $N$, $l$-dimensional data vectors and $\Theta=\{\boldsymbol{\theta}_j \in \Re^\ell, j=1,...,m\}$ be a set of $m$ vectors that will be used for the representation of the clusters formed in $X$. Let $U=[u_{ij}], i=1,...,N, j=1,...,m$ be an {$N\times m$} matrix whose $(i,j)$ element stands for the so-called {\it degree of compatibility} of $\mathbf{x}_i$ with the $j$th cluster, denoted by $C_j$ and represented by the vector $\boldsymbol{\theta}_j$. Let also ${\mathbf{u}_i}^T=[u_{i1},...,u_{im}]$ be the vector containing the elements of the $i$th row of $U$. In what follows we consider only Euclidean norms, denoted by $\|\cdot\|$.

According to \cite{Kris93}, \cite{Kris96}, the $u_{ij}$'s should satisfy the conditions, (a) ${u_{ij} \in [0,1]}$, (b) $\max_{j=1,...,m}$ $u_{ij}>0$ and (c) $0<\sum\limits_{i=1}^N u_{ij}<N$. As it has been stated earlier, the strategy of a possibilistic algorithm is to move the vectors $\boldsymbol{\theta}_j$'s to regions that are dense in data points of $X$. This is carried out via the minimization of, among others, the following objective function \cite{Kris96}:
\begin{equation}
J_{PCM}(\Theta,U)=\sum\limits_{i=1}^N \sum\limits_{j=1}^m u_{ij}\|\mathbf{x}_i-\boldsymbol{\theta}_j\|^2 + \sum\limits_{j=1}^m \gamma_j \sum\limits_{i=1}^N (u_{ij}\ln u_{ij}-u_{ij})
\label{Jpcm}
\end{equation}
with respect to $\boldsymbol{\theta}_j$'s and $u_{ij}$'s, where $\gamma_j$'s are positive parameters, each one associated with a cluster. More specifically, each $\gamma_j$ indicates the degree of ``influence" of $C_j$ around its representative $\boldsymbol{\theta}_j$; the smaller (greater) the value of $\gamma_j$, the smaller (greater) the influence of cluster $C_j$ around $\boldsymbol{\theta}_j$. Also, $\gamma_j$'s are kept fixed during the execution of the algorithm. One way to estimate $\gamma_j$ is to run the FCM algorithm first and after its convergence, to set 
\begin{equation}
\label{Ketaj}
\gamma_j=B \frac{ \sum_{i=1}^N u^{FCM}_{ij} \|\mathbf{x}_i-\boldsymbol{\theta}_j\|^2 }{ \sum_{i=1}^N u^{FCM}_{ij} }, \ \ \ j=1,\ldots,m
\end{equation}
where usually $B$ is set equal to 1. However, since a prerequisite for the FCM to provide good clustering results is the accurate knowledge of the number of clusters (which is rarely the case in practice), the estimates for $\gamma_j$'s are, in most cases, not very accurate. Consequently, this usually leads to poor results, especially for more demanding data sets.

Minimizing $J_{PCM}(\Theta,U)$ with respect to $u_{ij}$ and $\boldsymbol{\theta}_j$ leads to the following two coupled updating equations,
%
\begin{center}
\begin{minipage}{.4\linewidth}
\begin{equation}
	u_{ij}=\exp\left(-\frac{\|\mathbf{x}_i-\boldsymbol{\theta}_j\|^2}{\gamma_j}\right)
\label{uij}
\end{equation}
\end{minipage}
\begin{minipage}{.3\linewidth}
\begin{equation}
	\boldsymbol{\theta}_j=\frac{\sum_{i=1}^N u_{ij} \mathbf{x}_i}{\sum_{i=1}^N u_{ij}}
\label{theta}
\end{equation}
\end{minipage}
\end{center}
Thus, PCM iterates between these two equations, giving at each iteration updated estimations for $u_{ij}$'s and $\boldsymbol{\theta}_j$'s, until a specific termination criterion is met. Note from eq.~\eqref{theta} that {\it all} data vectors contribute to the estimation of each one of the representatives. However, the farthest ones from a specific $\boldsymbol{\theta}_j$ contribute less, since the corresponding $u_{ij}$'s are smaller for these vectors, as eq.~\eqref{uij} indicates. Obviously, the estimates of the $u_{ij}$'s highly affect the estimation accuracy in the computation of $\boldsymbol{\theta}_j$'s from eq.~\eqref{theta}. It is clear that this alternate updating between $u_{ij}$'s and $\boldsymbol{\theta}_j$'s in PCM moves each representative towards the center of its closest dense in data region. In this sense, we say that PCM recovers the physical clusters. In addition, the update of $u_{ij}$'s is highly dependent on the parameters $\gamma_j$'s (a fact that is further magnified through the presence of the $\exp(\cdot)$ function), thus making imperative an accurate assessment of the latter. At this point, it is worth emphasizing the crucial role of the initialization of $\boldsymbol{\theta}_j$'s. Specifically, we would like to place initially at least one representative in each dense region (cluster) and hope that PCM will lead each such representative to the center of the dense region where it was initially placed.

As it has been mentioned earlier, PCM does not require exact prior knowledge of the number of clusters $m$ in $X$, but, rather, a crude estimation of it. In the case where $m$ is underestimated the algorithm will reveal at least some physical clusters, while if $m$ is overestimated, the algorithm will (potentially) recover all physical clusters, however with some duplicates. Thus, after the convergence of PCM, one should identify and remove these duplicates. 

\section{Introducing Sparsity - The Sparse PCM (SPCM)}
A notable feature of the PCM algorithm is that {\it all} data vectors contribute to the updating of the representatives (see eq.~\eqref{theta}) since, from eq.~\eqref{uij}, we have that all $u_{ij}$'s are positive. When the physical clusters are well seperated from each other, the updating of a specific $\boldsymbol{\theta}_j$ will only slightly be affected by distant from it data points. However, in the case where the physical clusters are closely located to each other and have different densities, the affection of $\boldsymbol{\theta}_j$ from data points that belong to other physical clusters will be increased. Moreover, the affection will be higher for a representative in the sparser cluster. This may drive its representative towards the center of the denser cluster, failing thus to identify the sparser cluster. However, even if this does not happen, the corresponding final estimates of $\boldsymbol{\theta}_j$'s will not represent accurately the physical cluster centers. The previous arguments are illustrated in the following two examples.

\begin{figure}[htpb!]
\centering
\subfloat[Data set of Example 1]{\includegraphics[width=0.4\textwidth]{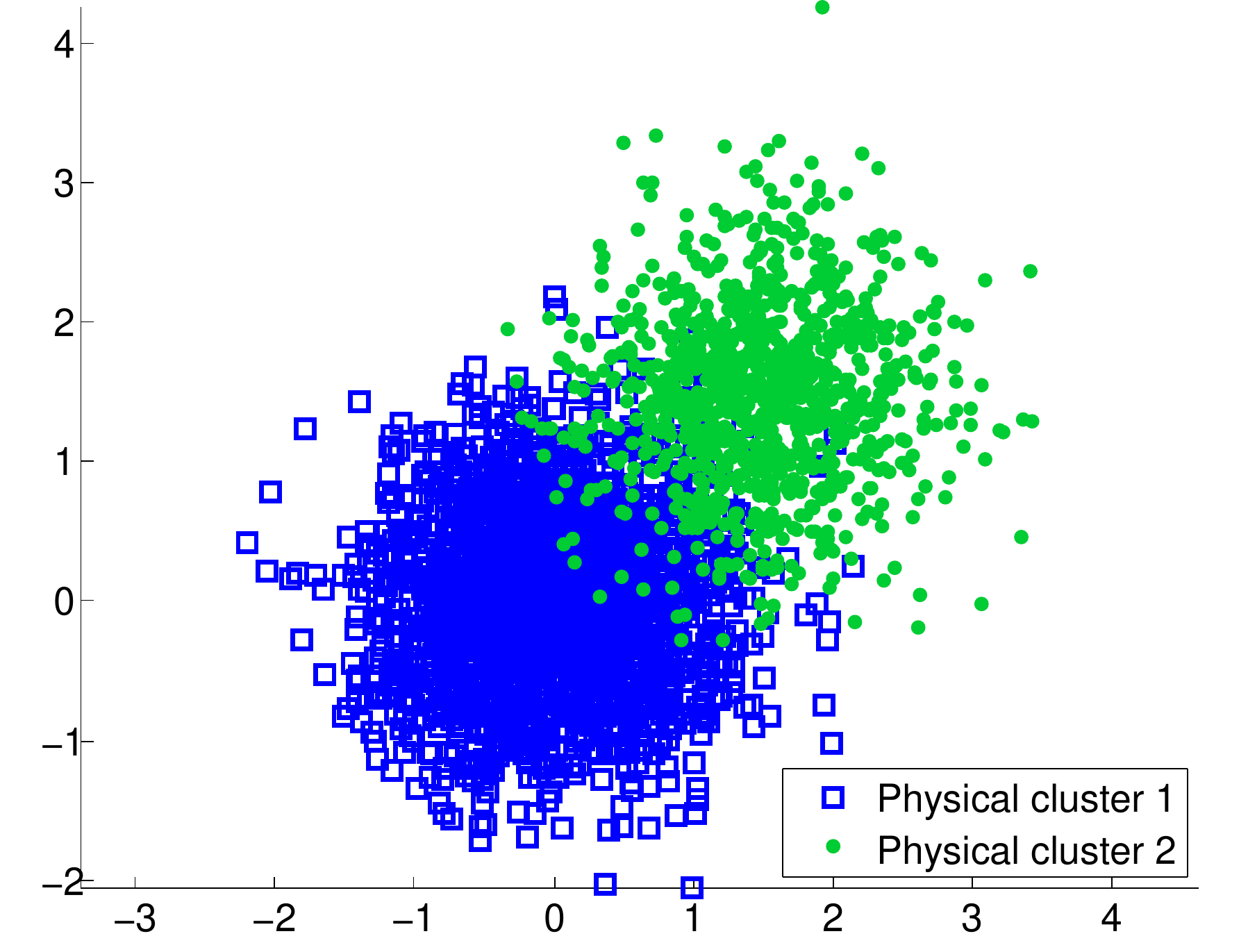}\label{ex1data}}
\hfil
\centering
\subfloat[PCM]{\includegraphics[width=0.4\textwidth]{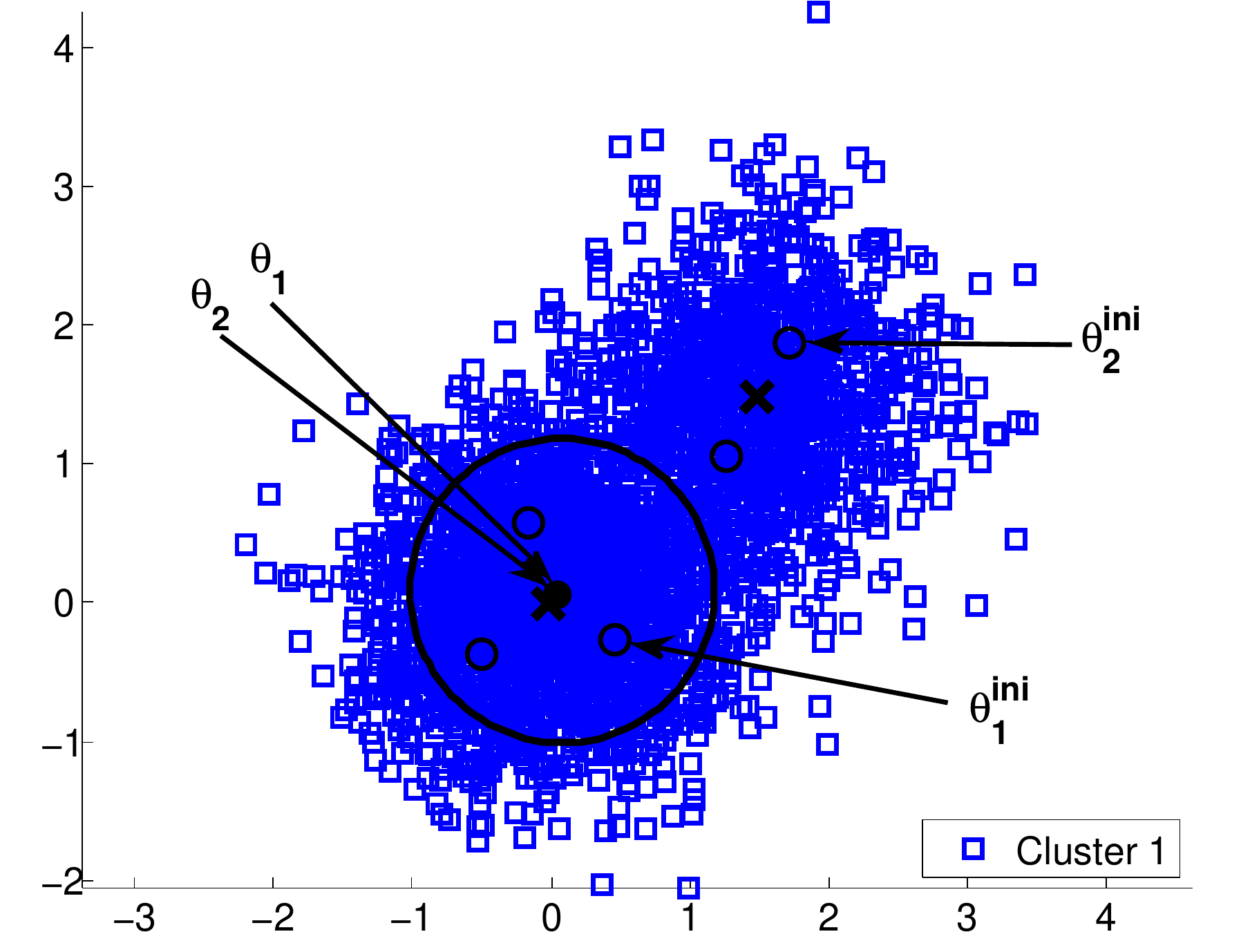}\label{ex1-PCM}}
\hfil
\centering{\caption{(a) The data set of Example 1 and (b) the clustering result of Example 1 for PCM with $m=5$. Open circles stand for the initial location of the representatives, closed circles represent the locations of the representatives ($\boldsymbol{\theta}_j$'s) after the convergence of the algorithm and crosses represent the true center of the clusters ($\mathbf{c}_j$'s). In addition, the circles centered at each $\boldsymbol{\theta}_j$ and having radius $\sqrt{\gamma_j}$ (after the convergence of the algorithm) are also drawn.}\label{example1PCM}}
\end{figure}

{\bf Example 1:} Consider a two-dimensional data set $X$ consisting of $N=3000$ points, where two physical clusters $C_1$ and $C_2$ are formed. The clusters are modelled by normal distributions with means $\mathbf{c}_1=[0, 0]^T$ and $\mathbf{c}_2=[1.5, 1.5]^T$, respectively, while their covariance matrices are both set to $0.4\cdot I_2$, where $I_2$ is the $2\times 2$ identity matrix. A number of 2000 points of $X$ is generated by the first distribution and 1000 points are generated by the second one. Note that the clusters share the same covariance matrix, they are located very close to each other and they have different densities, as shown in Fig.~\ref{ex1data}. The clustering result of the PCM, executed for $m=5$ clusters, is shown in Fig.~\ref{ex1-PCM}. Apparently, PCM failed to uncover the sparser cluster. To see qualitatively why this happens let us focus on $\boldsymbol{\theta}_1$ and $\boldsymbol{\theta}_2$ in Fig.~\ref{ex1-PCM}. As it can be seen, $\boldsymbol{\theta}_2$ was finally attracted towards $C_1$, although it was initially placed in $C_2$. This happens because in the process of determining the next location of $\boldsymbol{\theta}_2$, the many small contributions from the data points of $C_1$ gradually prevail over the less but larger contributions from the data points of $C_2$ (see eqs.~\eqref{uij},~\eqref{theta}).

\begin{figure}[htpb!]
\centering
\subfloat[Data set of Example 2]{\includegraphics[width=0.4\textwidth]{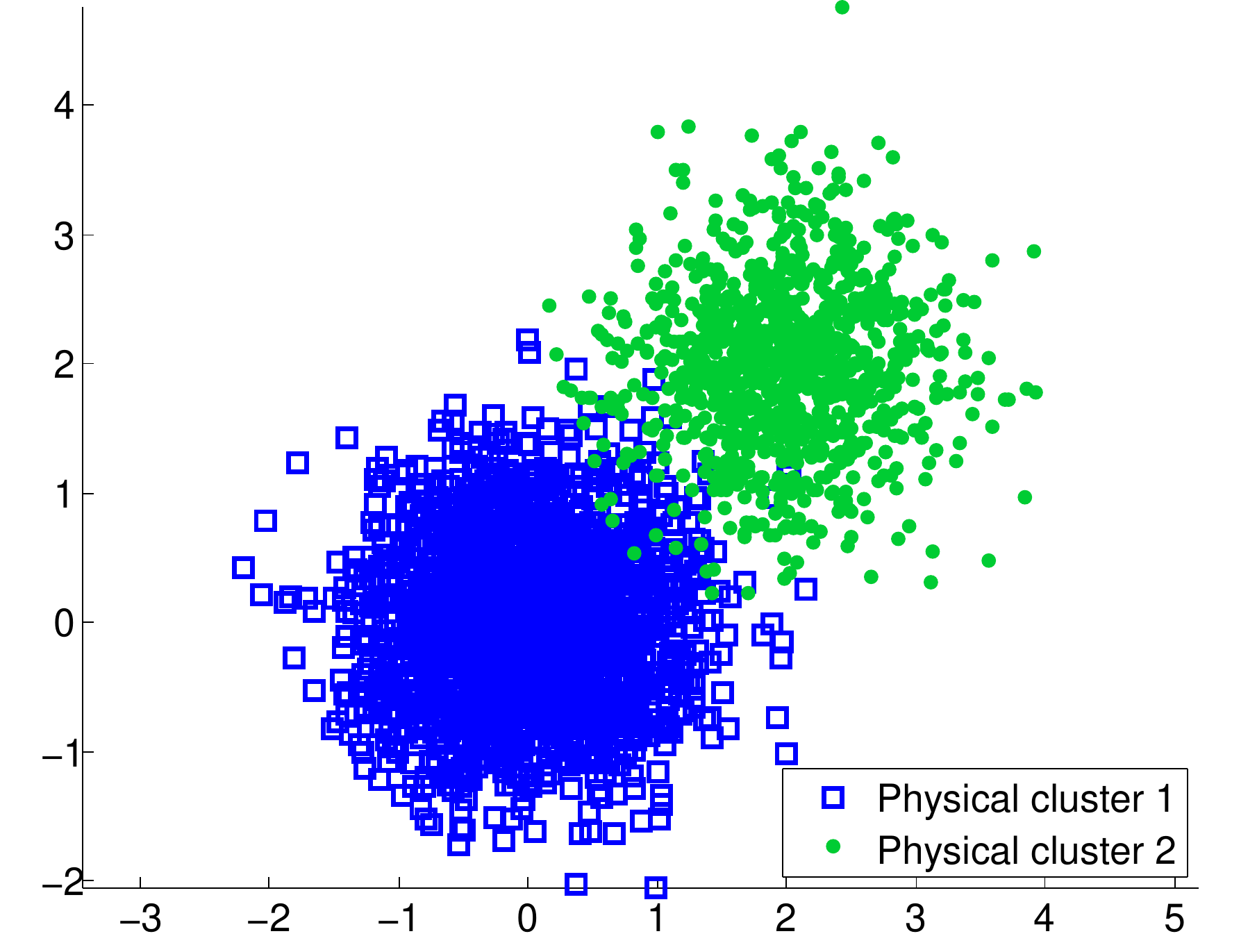}\label{ex2data}}
\hfil
\centering
\subfloat[PCM]{\includegraphics[width=0.4\textwidth]{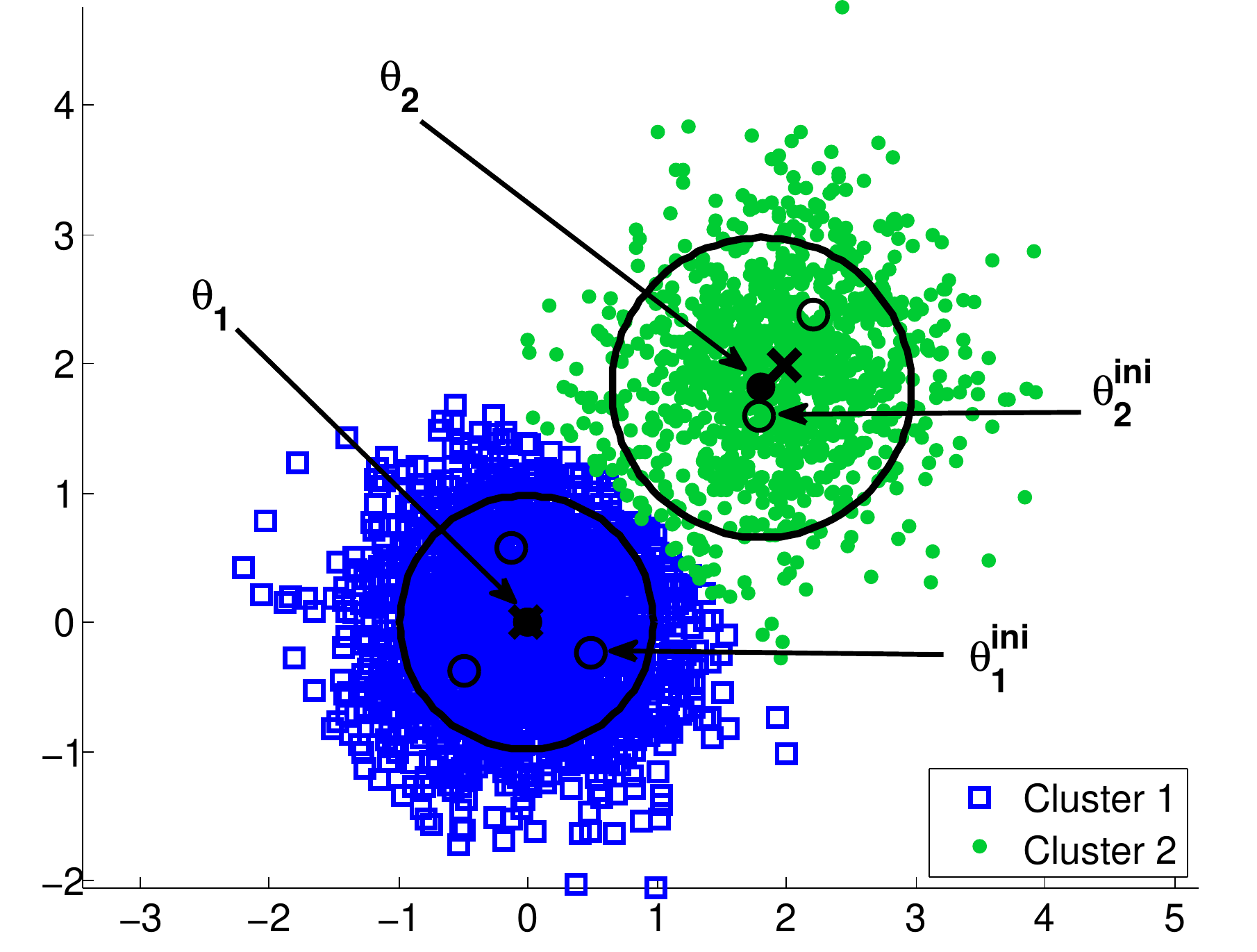}\label{ex2-PCM}}
\hfil
\centering{\caption{(a) The data set of Example 2 and (b) the clustering result of Example 2 for PCM with $m=5$. See also Fig.~\ref{example1PCM} caption.}\label{example2PCM}}
\end{figure}

{\bf Example 2:} Consider now the same two-dimensional data set of Example 1, where now the two normal distributions are more distant from each other with means $\mathbf{c}_1=[0, 0]^T$ and $\mathbf{c}_2=[2, 2]^T$, respectively (see Fig.~\ref{ex2data}). As is shown in Fig.~\ref{ex2-PCM}, PCM now succeeds in identifying both clusters. It seems that, in determining the next location of $\boldsymbol{\theta}_2$, the many small contributions from the data points of $C_1$ did not succeed to prevail over the less but larger contributions from the data points of $C_2$. However, the final estimates of the true centers (means of the Gaussians) are not very accurate, as shown qualitatively in Fig.~\ref{ex2-PCM} and established quantitatively later in Table~\ref{table:Example12}.

One way to face situations, such as those encountered in Examples 1 and 2, is to suppress the contribution in the updating of representatives from data points that are distant from it. Focusing on a specific representative $\boldsymbol{\theta}_j$, this can be achieved by setting $u_{ij}=0$ for data points $\mathbf{x}_i$ that are distant from it. Recalling that $\mathbf{u}_i^T=[u_{i1},\ldots,u_{im}]$, $i=1,\ldots,N$, this is tantamount to impossing sparsity on $\mathbf{u}_i$, i.e., forcing the corresponding data point $\mathbf{x}_i$ to contribute only to its (currently) closest representatives. To incorporate sparsity in PCM, we augment the cost function $J_{PCM}$ of eq.~\eqref{Jpcm}, as follows,
\begin{equation}
J_{SPCM}(\Theta,U)=\sum_{j=1}^m \left[ \sum_{i=1}^N u_{ij} \|\mathbf{x}_i-\boldsymbol{\theta}_j\|^2 + \gamma_j \sum_{i=1}^N (u_{ij} \ln u_{ij} - u_{ij})\right] + \lambda \sum_{i=1}^N \|\mathbf{u}_i\|_p^p, \ u_{ij}> 0 \ \footnote{This is a prerequisite in order for the $\ln u_{ij}$ to be well-defined. However, in the sequel, when refering to $\ln u_{ij}$ for $u_{ij}=0$, we mean $\lim\limits_{u_{ij}\rightarrow 0^+}\ln u_{ij}$.},\label{Jsapcm}
\end{equation}
where $\|\mathbf{u}_i\|_p$ is the $\ell_p$-norm of vector $\mathbf{u}_i$ ($p \in (0, 1)$); thus, $\|\mathbf{u}_i\|_p^p=\sum_{j=1}^m u_{ij}^p$. The last term in eq.~(\ref{Jsapcm}) is expected to induce sparsity on each one of the vectors $\mathbf{u}_i$, while $\lambda$ ($\geq 0$) is a regularization parameter that controls the degree of the imposed sparsity. The selection of the parameter $\lambda$, which remains constant during the execution of the algorithm, is discussed in subsection~\ref{subseclambda}. It is clear that by setting $\lambda=0$, we end up with the cost function which is associated with the classical PCM (eq.~\eqref{Jpcm}). The algorithm resulting by the minimization of $J_{SPCM}(\Theta,U)$ is called sparse possibilistic c-means (SPCM) clustering algorithm.
%

We describe next in detail the various stages of the algorithm. Specifically, we first describe the way its parameters are initialized. Next, the updating of $u_{ij}$'s and $\boldsymbol{\theta}_j$'s is considered. Note that the updating of $\boldsymbol{\theta}_j$'s is the same as in classical PCM, while, the updating of $u_{ij}$'s is quite different. Although the latter is more complicated than in the classical PCM, proposed in \cite{Kris96}, at the same time, it is far more simpler than the updating in other problems where sparsity is induced through the $\ell_p$-norm with $0<p<1$.

\subsection{Initialization in SPCM}
First, we make an overestimation, denoted by $m_{ini}$, of the true number of clusters $m$, underlying in the data set. Regarding $\boldsymbol{\theta}_j$'s, their initialization drastically affects the final clustering result in PCM. Thus, a good starting point for them is of crucial importance. Ideally, we would like to have at least one representative in the region of each physical cluster. To this end, the initialization of $\boldsymbol{\theta}_j$'s is carried out using the final cluster representatives obtained from the FCM algorithm, when the latter is executed with $m_{ini}$ clusters. Taking into account that FCM is likely to drive the representatives to ``dense in data" regions (since $m_{ini}>m$), we have a good probability of at least one of the initial $\boldsymbol{\theta}_j$'s to be placed in each dense region (cluster) of the data set.

After the initialization of $\boldsymbol{\theta}_j$'s, we initialize $\gamma_j$'s as in eq.~\eqref{Ketaj} for $B=1$.
%
%

\subsection{Updating of $\boldsymbol{\theta}_j$'s and $u_{ij}$'s in SPCM}
Minimization of $J_{SPCM}(\Theta,U)$ with respect to $\boldsymbol{\theta}_j$ leads to the same updating equation as in the original PCM scheme (eq.~(\ref{theta})), since the last term added to the cost function does not depend on $\boldsymbol{\theta}_j$'s. It is only the updating of $u_{ij}$'s that will be modified, in the light of the last term of $J_{SPCM}(\Theta,U)$. Taking the derivative of $J_{SPCM}(\Theta,U)$ with respect to $u_{ij}$, we obtain
\begin{equation}
\frac{\partial J_{SPCM}(\Theta,U)}{\partial u_{ij}} \equiv f(u_{ij})=d_{ij}+\gamma_j\ln{u_{ij}}+\lambda pu_{ij}^{p-1},
\label{f2u}
\end{equation}
where $d_{ij}=\|\mathbf{x}_i-\boldsymbol{\theta}_j\|^2$. Obviously, $\frac{\partial J_{SPCM}(\Theta,U)}{\partial u_{ij}}=0$ is equivalent to $f(u_{ij})=0$, the solution of which will give the requested $u_{ij}$. This equation can be solved based on the following propositions.

\begin{prop}
$f(u_{ij})$ does not become zero for $u_{ij} \in (-\infty,0) \cup (1,+\infty)$.
\label{propfzero}
\end{prop}
\begin{proof}
It is clear that if $u_{ij}\in (1, +\infty)$, all terms in eq.~(\ref{f2u}) are strictly positive and, as a consequence, $f(u_{ij})$ is positive. Moreover, $u_{ij}\in (-\infty, 0)$ is meaningless, since in this case $\ln u_{ij}$ is not defined.
\end{proof}
\begin{prop}
The stationary points of $f(u_{ij})$ are $\hat{u}_{ij}={\left[{\frac{\lambda}{\gamma_j}p(1-p)}\right]}^{\frac{1}{1-p}}$ and $\tilde{u}_{ij}=+\infty$ \footnote{\label{ftnote}The proofs of Propositions 2 to 6 are given in Appendix A.}.
\label{propstatpoint}
\end{prop}
%
%
\begin{prop}
The unique minimum of $f(u_{ij})$ appears at $\hat{u}_{ij}={\left[{\frac{\lambda}{\gamma_j}p(1-p)}\right]}^{\frac{1}{1-p}}$ \footnoteref{ftnote}.
\label{propuniqmin}
\end{prop}
%
%
\begin{prop}
If $f(\hat{u}_{ij})<0$ then $f(u_{ij})$ has exactly two solutions $u^1_{ij}$, $u^2_{ij}\in (0,1)$ with $u^1_{ij}<u^2_{ij}$ \footnoteref{ftnote}.\label{propfpos}
\end{prop}
%
%
\begin{prop}
If $f(u_{ij})=0$ has two solutions $u^1_{ij}$, $u^2_{ij}$ (with $u^1_{ij}<u^2_{ij}$), $J_{SPCM}(\Theta,U)$ exhibits a local minimum at the largest of them  ($u^2_{ij}$) \footnoteref{ftnote}.
\label{proplargsol}
\end{prop}
%
%
\begin{prop}
$J_{SPCM}(\Theta,U)$ exhibits its global minimum (with respect to $u_{ij}$) at $u^*_{ij}$, where\footnoteref{ftnote}: 
\begin{equation}
u^*_{ij}=\left\{\begin{matrix}
u^2_{ij}, & \text{if } f(\hat{u}_{ij})<0 \text{ and } u^2_{ij}>\left(\frac{\lambda(1-p)}{\gamma_j}\right)^{1/(1-p)}\\ 
0, & \text{otherwise}
\end{matrix}\right.
\label{globJ}
\end{equation}
\label{propglobsol}
\end{prop}
%

Based on the above propositions, we solve $f(u_{ij})=0$ as follows. First, we determine $\hat{u}_{ij}$ and check whether $f(\hat{u}_{ij})>0$. If this is the case, then $f(u_{ij})$ has no roots in $[0,1]$. Note that, in this case, it is $f(u_{ij})>0$ for all $u_{ij}\in (0, 1]$, since $f(\hat{u}_{ij})>0$. Thus, $J_{SPCM}$ is increasing with respect to $u_{ij}$ in $(0, 1]$. Consequently, in this case we set $u_{ij}=0$, {\it imposing sparsity}. In the rare case, where $f(\hat{u}_{ij})=0$, we set $u_{ij}=0$, as $\hat{u}_{ij}$ is the unique root of $f(u_{ij})=0$ and $f(u_{ij})>0$ for $u_{ij}\in (0,\hat{u}_{ij})\cup(\hat{u}_{ij},1]$. If $f(\hat{u}_{ij})<0$, then $f(u_{ij})=0$ has two solutions in $[0,1]$. In order to determine the largest of the solutions ($u^2_{ij}$), we apply the bisection method (see e.g. \cite{Corl77}) in the range $(\hat{u}_{ij},1]$, as $u^2_{ij}$ is greater than $\hat{u}_{ij}$ (see proof of Proposition \ref{proplargsol}). The bisection method is known to converge very rapidly to the optimum $u_{ij}$, that is, in our case, to the largest of the two solutions of $f(u_{ij})=0$ \footnote{Alternatively, any other method of this kind can also be used, e.g. \cite{Hous70}.}. Finally, we choose the global minimum of $J_{SPCM}$ (with respect to $u_{ij}$), as eq.~\eqref{globJ} indicates.

\subsection{Selection of the parameter $\lambda$}
\label{subseclambda}
As it follows from the previous analysis, considering a specific data point $\mathbf{x}_i$ and a cluster $C_j$, a necessary condition in order for the equation $f(u_{ij})=0$ to have a solution is $f(\hat{u}_{ij})<0$, which, taking into account eq.~\eqref{f2u} and solving with respect to $\lambda$ gives $\lambda<\frac{\gamma_j}{p(1-p)}\exp{\left(-1-\frac{d_{ij}(1-p)}{\gamma_j}\right)}$.
Consequently, selecting 
\begin{equation}
\lambda\geq\frac{\gamma_j}{p(1-p)}\exp{\left(-1-\frac{d_{ij}(1-p)}{\gamma_j}\right)},
\label{lambdafragma}
\end{equation}
the degree of compatibility $u_{ij}$ of a data point $\mathbf{x}_i$ with a cluster $C_j$ is set to 0, promoting sparsity. Aiming at retaining the smallest sized cluster, say $C_q$ (i.e., the cluster with $\gamma_q=\min\limits_{j=1,\ldots,m}\gamma_j$) until the termination of the algorithm (provided of course that at least one representative has been initially placed in it), a reasonable choice for $\lambda$ would be the one for which $u_{ij}$ becomes 0 for points $\mathbf{x}_i$ that lie at distance $d_{iq}$ greater than $\gamma_q$ from the representative $\boldsymbol{\theta}_q$. In this way, $\boldsymbol{\theta}_q$ will be less likely to be ``attracted" by nearby larger clusters, but, instead, is aided to remain in the region of the physical cluster where it was first placed. This is so because the cluster representative will be affected only by the data points that are very close to it (i.e., points with $d_{iq}<\gamma_q=\min\limits_{j=1,\ldots,m}\gamma_j$). 

To this end, applying inequality~\eqref{lambdafragma} for $d_{ij}$ and $\gamma_j$ equal to $\gamma_q=\min\limits_{j=1,\ldots,m}\gamma_j$, we end up with $\lambda\geq\frac{\gamma_q}{p(1-p)\mathrm{e}^{p-2}}$, where $\mathrm{e}$ is the base of natural logarithm. In practice, we select $\lambda$ as 
\begin{equation}
\lambda=K\frac{\min\limits_{j=1,\ldots,m}\gamma_j}{p(1-p)\mathrm{e}^{p-2}},
\label{lambda}
\end{equation}
where $0\ll K<1$, i.e., actually we allow non-zero $u_{ij}$'s for points that lie at distance a bit larger than $\gamma_q$ from  $\boldsymbol{\theta}_q$. 
%
%
In all the experiments of SPCM, we take $K=0.9$.

\subsection{The SPCM algorithm}
\label{subsecSPCM}
From the previous analysis, the SPCM algorithm can be summarized as follows.

\begin{algorithm}[htpb!]
\caption{ [$\Theta$, $\Gamma$, $U$] = SPCM($X$, $m_{ini}$)}
\label{alg:spcm}
\begin{algorithmic}[1]
\doublespacing
\Require {$X$, $m_{ini}$}
\State $t=0$
\State $m=m_{ini}$
\Statex {\Comment{Initialization of $\boldsymbol{\theta}_j$'s part}}
\State \textbf{Initialize:} $\boldsymbol{\theta}_j(t)$ {\it via FCM algorithm}
\Statex {\Comment{Initialization of $\gamma_j$'s part}}
   \State \textbf{Set:} $\gamma_j=\frac{ \sum_{i=1}^n u^{FCM}_{ij} \|\mathbf{x}_i-\boldsymbol{\theta}_j(t)\|^2}{ \sum_{i=1}^n u^{FCM}_{ij} }$, $j=1,...,m$ 
	\State \textbf{Set:} $\lambda=K\frac{\min\nolimits_{j=1,\ldots,m}\gamma_j}{p(1-p)\mathrm{e}^{p-2}}$, $K=0.9$ 
\Repeat 
\Statex {\Comment {Update $U$ part}}
\algstore{bkbreak}
\end{algorithmic}
\end{algorithm}

\begin{algorithm}[t!]
\begin{algorithmic}[1]
\algrestore{bkbreak}
\doublespacing
\State \textbf{Update:} $U(t)$ (as described in the text)
\Statex {\Comment{Update $\Theta$ part}}
\State \text{$\boldsymbol{\theta}_j(t+1)=\left.{\sum\limits_{i=1}^N u_{ij}(t)\mathbf{x}_i} \middle/ {\sum\limits_{i=1}^N u_{ij}(t)} \right.$, $j=1,...,m$} 

\State $t=t+1$ 
\Until{the change in $\boldsymbol{\theta}_j$'s between two successive iterations becomes sufficiently small}\\ 
\Return {$\Theta$, $\Gamma=\{\gamma_1,\ldots,\gamma_m\}$, $U$}
\end{algorithmic}
\end{algorithm}

In the sequel, we discuss how the exploitation of sparsity affects the clustering result in Examples 1 and 2, by comparing PCM and SPCM through the use of some quantitative indices. Specifically, in order to compare a clustering outcome with the true data label information, we use (a) the {\it Rand Measure} (RM) (e.g. \cite{Theo09}), which measures the degree of agreement between the obtained clustering and the physical clustering and can handle clusterings whose number of clusters may differ from the number of physical clusters, (b) the {\it Success Rate} (SR), which measures the percentage of the points that have been correctly labeled by an algorithm and (c) the mean of the Euclidean distances (MD) between the true center $\mathbf{c}_j$ of each physical cluster and its closest cluster representative ($\boldsymbol{\theta}_j$) obtained by each algorithm. In cases where a clustering algorithm ends up with a higher number of clusters than the actual one ($m_{final}>m$), only the $m$ cluster representatives that are closest to the true $m$ centers of the physical clusters, are taken into account in the determination of MD. On the other hand, in cases where $m_{final}<m$, the MD measure refers to the distances of the actual centers from their nearest cluster representatives. It is noted that lower MD values indicate more accurate determination of the cluster center locations.

\begin{figure}[htpb!]
\centering
\subfloat[SPCM]{\includegraphics[width=0.4\textwidth]{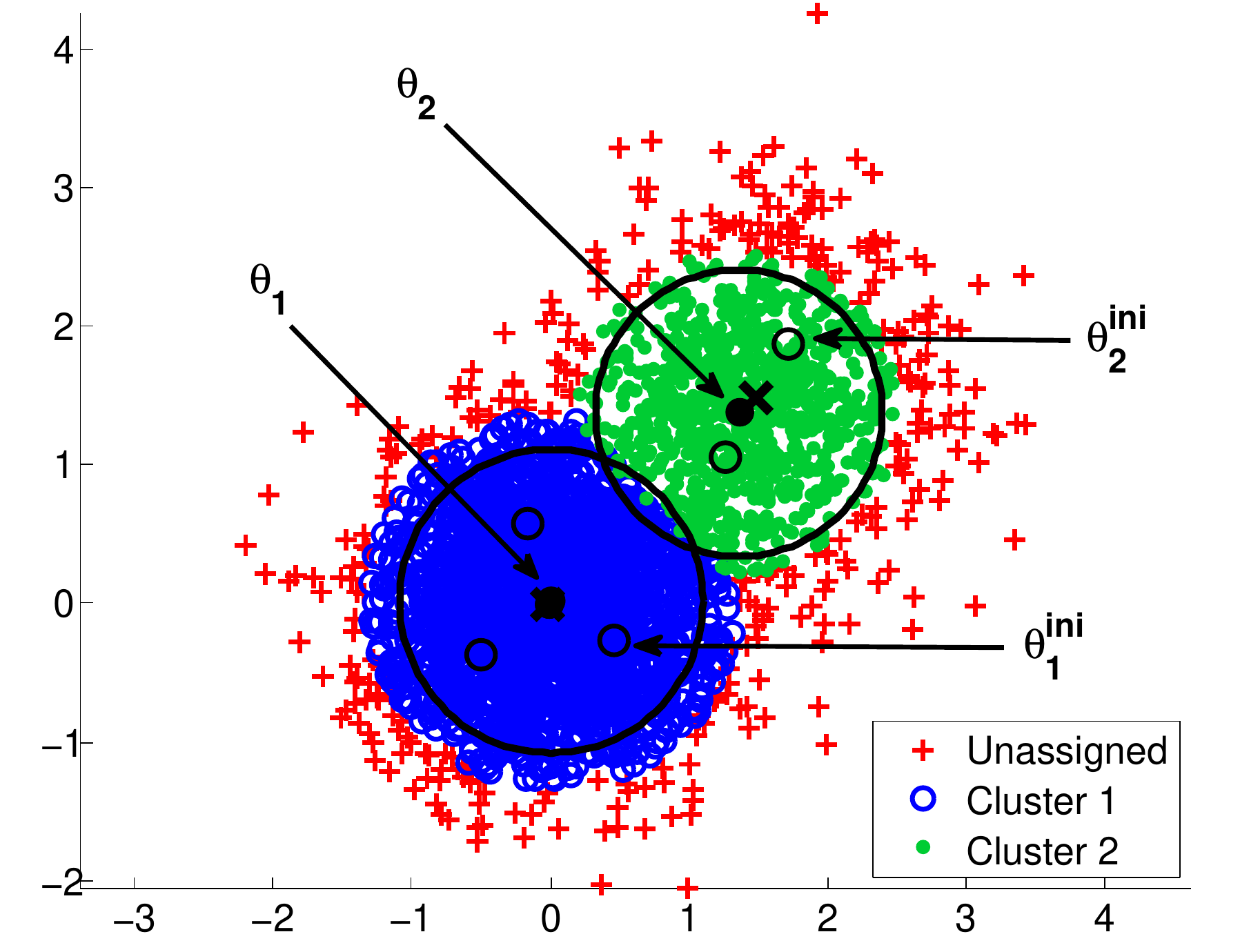}\label{ex1-SPCM}}
\hfil
\centering
\subfloat[SPCM]{\includegraphics[width=0.4\textwidth]{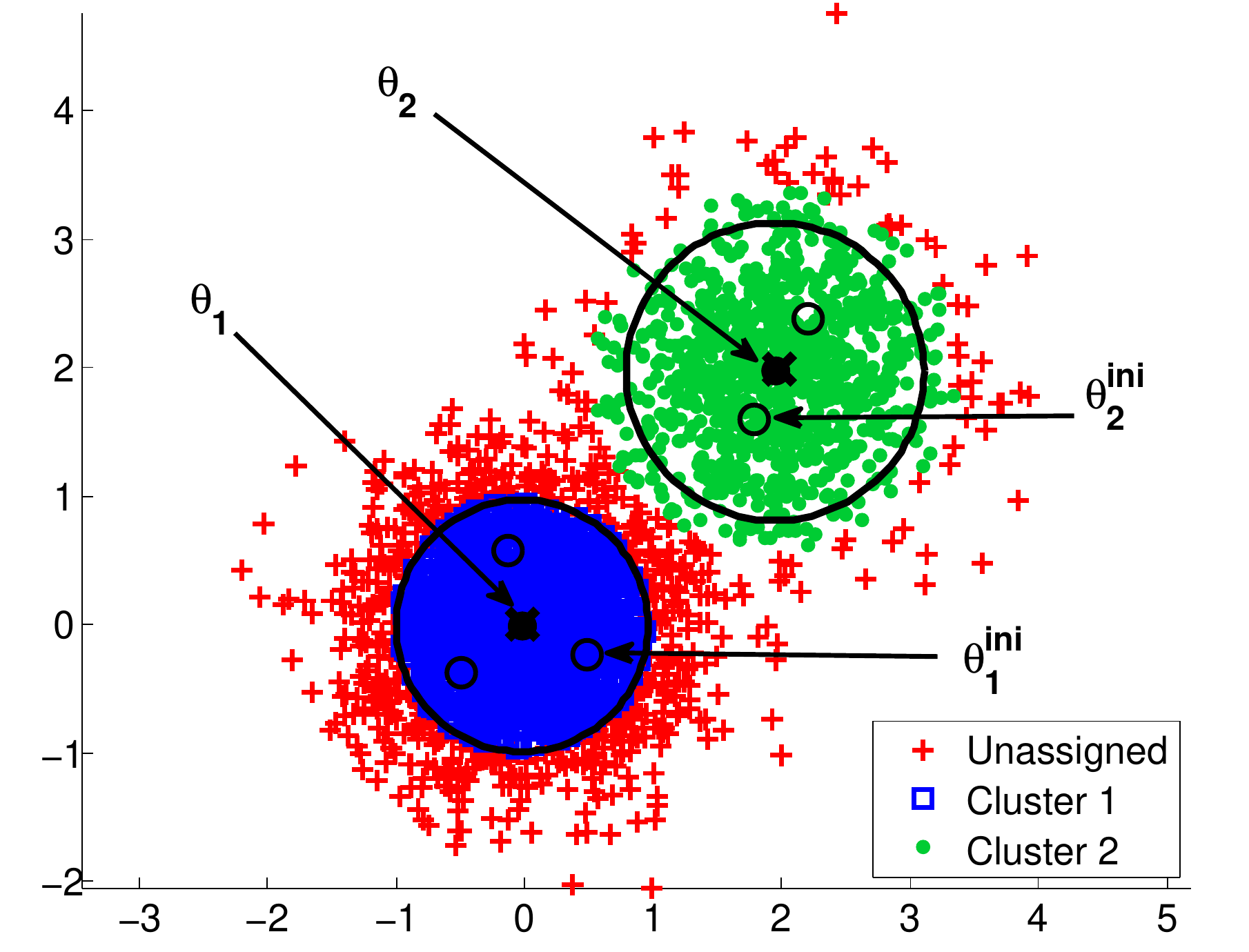}\label{ex2-SPCM}}
\hfil
\centering{\caption{The clustering results of SPCM for the data set of (a)  Example 1 with $m_{ini}=5$ and (b) Example 2 with $m_{ini}=5$. See also Fig.~\ref{example1PCM} caption.}\label{example12SPCM}}
\end{figure}

{\bf Example 1 (cont.):} Table~\ref{table:Example12} shows the clustering results of PCM and SPCM, where $m_{ini}$ and $m_{final}$ denote the initial and the final number of distinct clusters. Figs.~\ref{ex1-PCM} and \ref{ex1-SPCM} depict the performances of PCM and SPCM, respectively. 

As we have already seen, PCM fails to uncover the underlying clustering structure (as is clearly depicted quantitatively in Table~\ref{table:Example12}), whereas SPCM distinguishes the two physical clusters, since it annihilates the contributions of most of the points of $C_1$ ($C_2$) in the determination of the next location of $\boldsymbol{\theta}_2$ ($\boldsymbol{\theta}_1$) through the imposition of sparsity. This is also verified through the achieved satisfactory values of RM, SR and MD (see Fig.~\ref{ex1-SPCM} and Table~\ref{table:Example12}).

\begin{table}[htpb!]
\centering
\caption{Performance of PCM and SPCM for the data sets of Examples 1 and 2.}
{\small
\begin{tabular}{>{\arraybackslash}m{0.1\linewidth} | >{\centering\arraybackslash}m{0.15\linewidth} |>{\centering\arraybackslash}m{0.06\linewidth}| >{\centering\arraybackslash}m{0.07\linewidth} |>{\centering\arraybackslash}m{0.1\linewidth} |>{\centering\arraybackslash}m{0.1\linewidth} |>{\centering\arraybackslash}m{0.1\linewidth}}
\hline  
	& \centering Data Set & \centering $m_{ini}$ & \centering $m_{final}$ & \centering RM & \centering SR & {\centering MD} \\
\hline
PCM  	& Example 1 & 5 & 1 & 54.02 & 64.20 & 1.0271 \\
SPCM  & Example 1 & 5 & 2 & 91.16 & 95.37 & 0.0875 \\
\hline
PCM  & Example 2 & 5 & 2 & 95.44 & 97.67 & 0.1150 \\
SPCM & Example 2 & 5 & 2 & 96.21 & 98.07 & 0.0204 \\
\hline
\end{tabular}}
\label{table:Example12}
\end{table}

{\bf Example 2 (cont.):} 
Table~\ref{table:Example12} shows the clustering results of PCM and SPCM and Fig.~\ref{ex2-SPCM} depicts the performance of SPCM. As we have seen in this case, PCM is able to uncover the underlying clustering structure. However, SPCM manages to detect more accurately the true centers of the clusters, as the MD index indicates.

\section{The sparse adaptive PCM (SAPCM)}
Despite the fact that SPCM can handle successfully cases of closely located and different in density clusters, it still suffers from the problem of its ancestor PCM as far as the estimation of $\gamma_j$'s is concerned. Specifically, the estimation of $\gamma_j$'s is based on the outcomes of the FCM, which can be significantly affected by the possible presense of noise or outliers in the data, as well as by the possible differences in the variance of the clusters. Moreover, once they have been estimated they remain fixed during the execution of the algorithm. Thus, poor initial estimates of $\gamma_j$'s may lead SPCM to degraded performance. Furthermore, as is the case with all PCMs, SPCM may end up with coincident clusters (duplicates of the same cluster). This happens when more than one representatives are led to the center of the same physical cluster.

One way to deal with these issues is to allow $\gamma_j$'s to adapt as the algorithm evolves. This will allow the algorithm to track the changes occuring in the formation of clusters during its execution. Such a method has been proposed in \cite{Xen15}, where a PCM algorithm called adaptive PCM (APCM) was introduced. As shown in \cite{Xen15}, besides the above, APCM is able to determine the true number of clusters. In the sequel, we extend SPCM in order to incorporate the adaptation of $\gamma_j$'s by embedding the relevant mechanism of APCM. The resulting algorithm is called Sparse Adaptive PCM (SAPCM). As a consequence of the above, the algorithm inherits the ability to detecting automatically also the true number of physical clusters. Next, inspired by \cite{Xen15}, we describe how the parameters $\gamma_j$'s are adapted in SAPCM, so that starting from an overestimated number of clusters, to conclude to the true number of physical clusters.
%

The proposed SAPCM algorithm stems from the optimization of the cost function in eq.~\eqref{Jsapcm} where now $\gamma_j$ is defined as
\begin{equation}
\gamma_j=\frac{\hat{\eta}}{\alpha}\eta_j
\label{gamma}
\end{equation}
with $\eta_j$ being a measure of the mean absolute deviation of $C_j$ in its current form (to be defined rigorously in the next subsection), $\alpha$ is a user-defined positive parameter \cite{Xen15} and $\hat{\eta}$ is a constant defined as the minimum among all initial $\eta_j$'s, i.e., $\hat{\eta}=\min\limits_{j=1,\ldots,m_{ini}} \eta_j$, where $m_{ini}$ is the initial number of clusters.

\subsection{Initialization of $\gamma_j$'s}

In SAPCM, we initialize $\eta_j$'s as follows \cite{Xen15}:
\begin{equation}
\label{initetaj}
\eta_j=\frac{ \sum_{i=1}^N u^{FCM}_{ij} \|\mathbf{x}_i-\boldsymbol{\theta}_j\|}{ \sum_{i=1}^N u^{FCM}_{ij} }, \ \ \ j=1,\ldots,m_{ini}
\end{equation}
where $\boldsymbol{\theta}_j$'s and $u^{FCM}_{ij}$'s in eq.~(\ref{initetaj}) are the final parameter estimates obtained by FCM\footnote{An alternative initialization for $\gamma_j$'s is proposed in \cite{Xen13}.}. Combining eqs.~\eqref{gamma} and \eqref{initetaj}, the initialization of $\gamma_j$'s is completely defined.

It is worth noting that the above initialization of $\eta_j$'s, involves {\it Euclidean} instead of {\it squared Euclidean distances}, as is the case with the classical PCM algorithm. This gives the algorithm the agility to deal well with closely located clusters, for appropriate values of $\alpha$ \cite{Xen15}.

\subsection{Parameter adaptation in SAPCM}
This part of SAPCM is adopted by APCM \cite{Xen15} and refers to, (a) the adjustment of the number of clusters and (b) the adaptation of $\gamma_j$'s, which are two interrelated processes. In the sequel, for the sake of completeness, we describe in some detail the above characteristics. As far as the first is concerned, we proceed as follows. Let $label$ be a $N$-dimensional vector, whose $i$th component contains the index of the cluster which is most {\it compatible} with $\mathbf{x}_i$, that is the cluster $C_j$ for which $u_{ij}=\max_{r=1,...,m} u_{ir}$. Let $n_j$ denote the number of the data points $\mathbf{x}_i$, that are most compatible with the cluster $C_j$ and $\boldsymbol{\mu}_j$ be the mean vector of these data points. The adjustment (reduction) of the number of clusters is achieved by examining if the index $j$ of a cluster $C_j$ appears in the vector $label$. If this is the case (i.e. if there exists at least one vector $\mathbf{x}_i$ that is most compatible with $C_j$), $C_j$ is preserved. Otherwise, $C_j$ is eliminated (see {\it Possible cluster elimination} part in Algorithm~\ref{alg:sapcm}). 

Regarding the adaptation of $\gamma_j$'s at the iteration $t+1$ of the algorithm, we proceed as follows. Each parameter $\eta_j$ of a cluster $C_j$ is estimated as the mean absolute deviation of the most compatible data vectors to cluster $C_j$ (see {\it Adaptation of $\eta_j$'s} part in algorithm~\ref{alg:sapcm}), i.e., 
\begin{equation}
\eta_j(t+1)=\frac{1}{n_j(t)}\sum\nolimits_{\mathbf{x}_i:u_{ij}(t)=\max_{r=1,...,m(t+1)} u_{ir}(t)}^{} \|\mathbf{x}_i-\boldsymbol{\mu}_j(t)\|.
\label{adapteta}
\end{equation}

Note that, the proposed updating mechanism of $\eta_j$'s differs from others used in the classical PCM, as well as in many of its variants, in two distinctive points. First $\eta_j$'s are updated taking into account {\it only} the data vectors that are most compatible to cluster $C_j$ and not all the data points weighted by their corresponding coefficients $u_{ij}$. Second, the distances involved in the formula are between a data vector and the mean vector $\boldsymbol{\mu}_j$ of the most compatible points of the cluster; {\it not} from the representative $\boldsymbol{\theta}_j$, as in previous works (e.g. \cite{Kris93}, \cite{Zhan04}). This allows more accurate estimates for $\eta_j$'s (\cite{Xen15}). It is also noted that, in the (rare) case where there are two or more clusters, that are equally compatible with a specific $\mathbf{x}_i$, then $\mathbf{x}_i$ will contribute to the determination of the parameter $\eta$ of {\it only} one of them, which is chosen arbitrarily. The adaptation of the parameters $\gamma_j$'s results after combining eqs.~\eqref{gamma} and \eqref{adapteta}. For more details on the rationale behind the definition of $\gamma_j$'s see \cite{Xen15}.

\begin{figure}[htpb!]
\centering
{\includegraphics[width=0.55\textwidth]{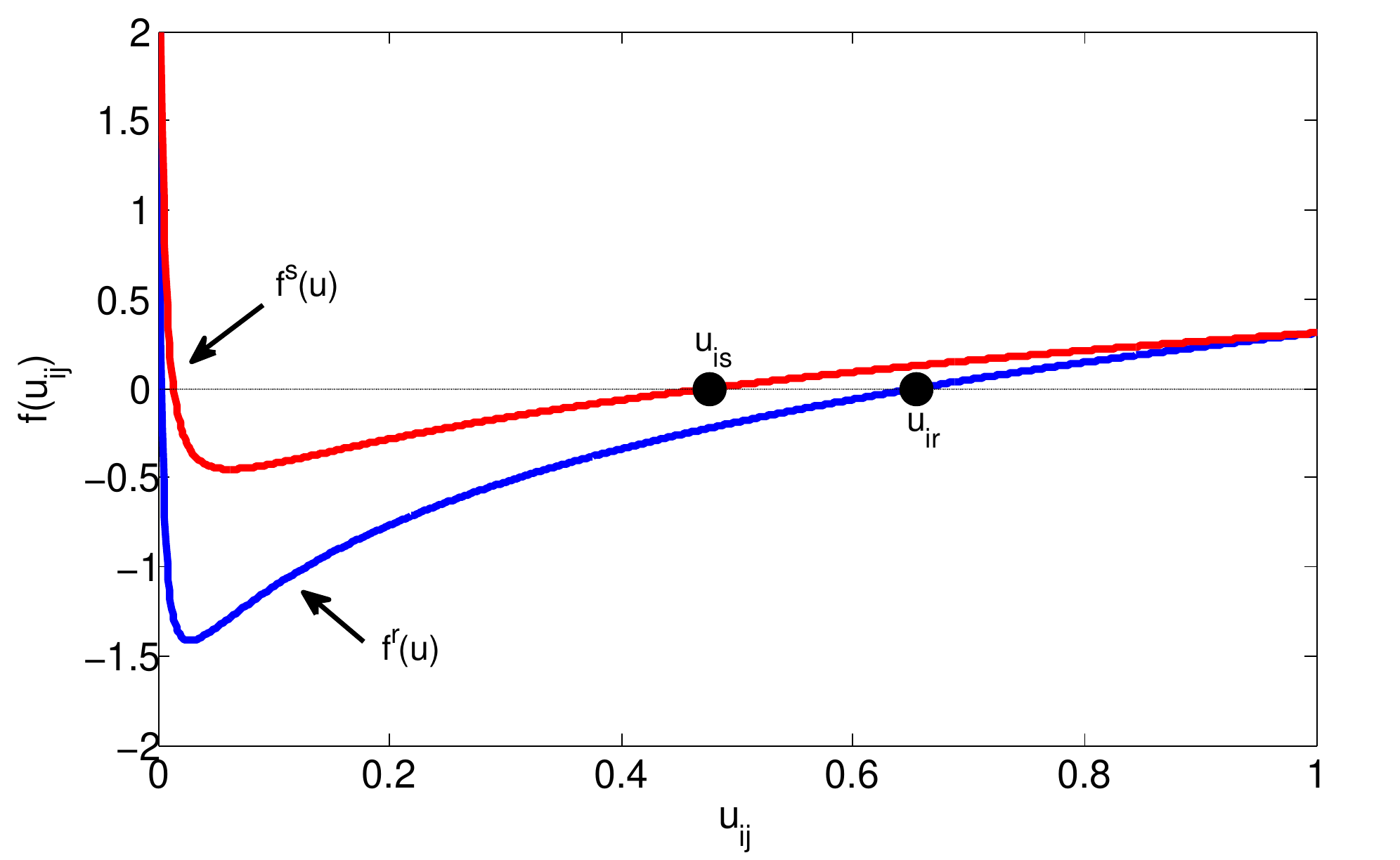}}
\hfil
\centering{\caption{Graphical presentation of $f^r(u)$ and $f^s(u)$ for constant $d$, $\lambda$ and $p$, with $\gamma_r > \gamma_s$. The largest of the two solutions of $f^r(u)=0$ and $f^s(u)=0$, $u_{ir}$ and $u_{is}$, are also shown, respectively. It is observed that $u_{ir}\geq u_{is}$.}\label{figfu}}
\end{figure}

Let us focus for a while on the immunity of the SAPCM algorithm to its initialization with an overestimated number of clusters. Taking into account (a) that all representatives are driven to dense in data regions, due to the possibilistic nature of SAPCM, (b) that the probability to select as representative at least one point in each dense region is increased, since the overestimated number of representatives are initially selected via FCM algorithm and (c) the mechanism for reducing the number of clusters, then, in principle, the number of the representatives which move to the same dense region will be reduced to a single one. In order to get some further insight on this issue, assume that two cluster representatives $\boldsymbol{\theta}_r$, $\boldsymbol{\theta}_s$ almost coincide, which, for a given $\mathbf{x}_i$ implies that $d_{ir}\simeq d_{is}\equiv d$, but let say that $\gamma_r > \gamma_s$. Consider also the functions $f^r(u)=d+\gamma_r\ln{u}+\lambda p u^{p-1}$ and $f^s(u)=d+\gamma_s\ln{u}+\lambda p u^{p-1}$ for $u\in(0,1]$. It is easy to see that $f^r(u)\leq f^s(u)$, for each $u\in(0,1]$. Assume now that both have positive solutions. It is easy to verify that $u_{ir}\geq u_{is}$, where $u_{ir}$ and $u_{is}$ are the largest of the two solutions of $f^r(u)=0$ and $f^s(u)=0$, respectively (see Fig.~\ref{figfu}). In the case where $u_{ir}=0$ then, trivially follows that $u_{is}=0$. Finally, if $u_{is}=0$ then $u_{ir}\geq u_{is}$. Thus, the influence of the cluster with the smaller $\gamma$ ($\gamma_s$) will be vanished by the influence of the one with the greater $\gamma$ ($\gamma_r$), in the sense that $u_{ir}>u_{is}$, for all data points $\mathbf{x}_i\in X$. As a consequence the index $s$ will not appear in the {\it label} vector and, thus, $C_s$ will be eliminated.

\subsection{The SAPCM algorithm}
The proposed SAPCM algorithm is summarized below (the choice of $\lambda$ is justified later).

\begin{algorithm}[htpb!]
\caption{ [$\Theta$, $\Gamma$, $U$, $label$] = SAPCM($X$, $m_{ini}$, $\alpha$)}
\label{alg:sapcm}
\begin{algorithmic}[1]
\doublespacing
\Require {$X$, $m_{ini}$, $\alpha$}
\State $t=0$
\State $m(t)=m_{ini}$
\Statex {\Comment{Initialization of $\boldsymbol{\theta}_j$'s part}}
\State \textbf{Initialize:} $\boldsymbol{\theta}_j(t)$ {\it via FCM algorithm}
\Statex {\Comment{Initialization of $\eta_j$'s part}}
   \State \textbf{Set:} $\eta_j(t)=\frac{ \sum_{i=1}^n u^{FCM}_{ij} \|\mathbf{x}_i-\boldsymbol{\theta}_j(t)\|}{ \sum_{i=1}^n u^{FCM}_{ij} }$, $j=1,...,m(t)$ 
	\State \textbf{Set:} $\hat{\eta}=\min\nolimits_{j=1,...,m(t)} \eta_j(t)$
	\State \textbf{Set:} $\gamma_j(t)=\hat{\eta}\eta_j(t)/\alpha$, $j=1,...,m(t)$ 
	\State \textbf{Set:} $\lambda(t)=K\frac{\min\nolimits_{j=1,\ldots,m(t)}\gamma_j(t)}{p(1-p)\mathrm{e}^{p-2}}$, $K=0.1$ 
\Repeat 
\Statex {\Comment {Update $U$ part}}
\State \textbf{Update:} $U(t)$ (as described in the text)
\Statex {\Comment{Update $\Theta$ part}}
\State \text{$\boldsymbol{\theta}_j(t+1)=\left.{\sum\limits_{i=1}^N u_{ij}(t)\mathbf{x}_i} \middle/ {\sum\limits_{i=1}^N u_{ij}(t)} \right.$, $j=1,...,m(t)$} 
\Statex {\Comment{Possible cluster elimination part}}
\For{$i \leftarrow 1 \textbf{ to } N$}
	\State \textbf{Determine:} $u_{ir}(t)=\max_{j=1,...,m(t)} u_{ij}(t)$
	\If{$u_{ir}(t) \neq 0$}
		\State \textbf{Set:} $label(i)=r$
\algstore{bkbreak}
\end{algorithmic}
\end{algorithm}

\begin{algorithm}[htpb!]
\begin{algorithmic}[1]
\algrestore{bkbreak}
\doublespacing
	\Else
		\State \textbf{Set:} $label(i)=0$
	\EndIf
\EndFor
\State $p=0$ {\it \ \ \ //number of removed clusters at iteration $t$}
\For{$j \leftarrow 1 \textbf{ to } m(t)$}
	\If{$j \notin label$}
		\State \textbf{Remove:} $C_j$ (and \textbf{renumber} accordingly $\Theta(t+1)$ and the columns of $U(t)$)
		\State $p=p+1$
	\EndIf
\EndFor
\State $m(t+1)=m(t)-p$
\Statex {\Comment{Adaptation of $\eta_j$'s part}}
\State $\eta_j(t+1)=\frac{1}{n_j(t)}\sum\nolimits_{\mathbf{x}_i:u_{ij}(t)=\max_{r=1,...,m(t+1)} u_{ir}(t)}^{} \|\mathbf{x}_i-\boldsymbol{\mu}_j(t)\|$, $j=1,...,m(t+1)$
\State \textbf{Set:} $\gamma_j(t+1)=\hat{\eta}\eta_j(t+1)/\alpha$, $j=1,...,m(t+1)$ 
	\State \textbf{Set:} $\lambda(t+1)=K\frac{\min\nolimits_{j=1,\ldots,m(t+1)}\gamma_j(t+1)}{p(1-p)\mathrm{e}^{p-2}}$, $K=0.1$ 
\State $t=t+1$ 
\Until{the change in $\boldsymbol{\theta}_j$'s between two successive iterations becomes sufficiently small}\\ 
\Return {$\Theta$, $\Gamma=\{\gamma_1,\ldots,\gamma_m\}$, $U$, $label$}
\end{algorithmic}
\end{algorithm}

In the sequel, we give some very demanding experimental set ups which exhibit the enhanced abilities of SAPCM compared to APCM.

\begin{figure}[htpb!]
\centering
\subfloat[Data set of Example 3]{\includegraphics[width=0.4\textwidth]{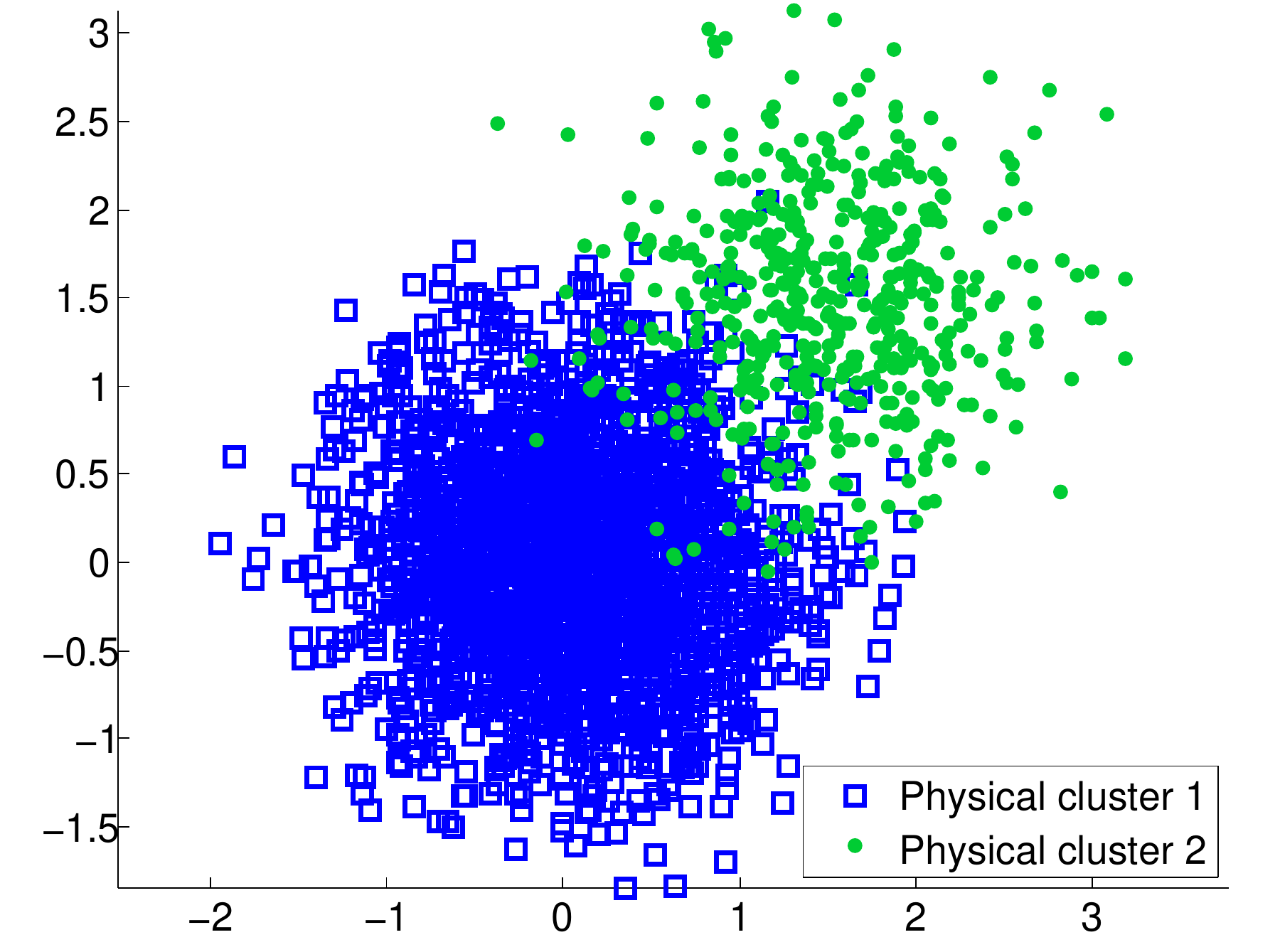}\label{ex3data}}
\hfil
\centering
\subfloat[Data set of Example 4]{\includegraphics[width=0.4\textwidth]{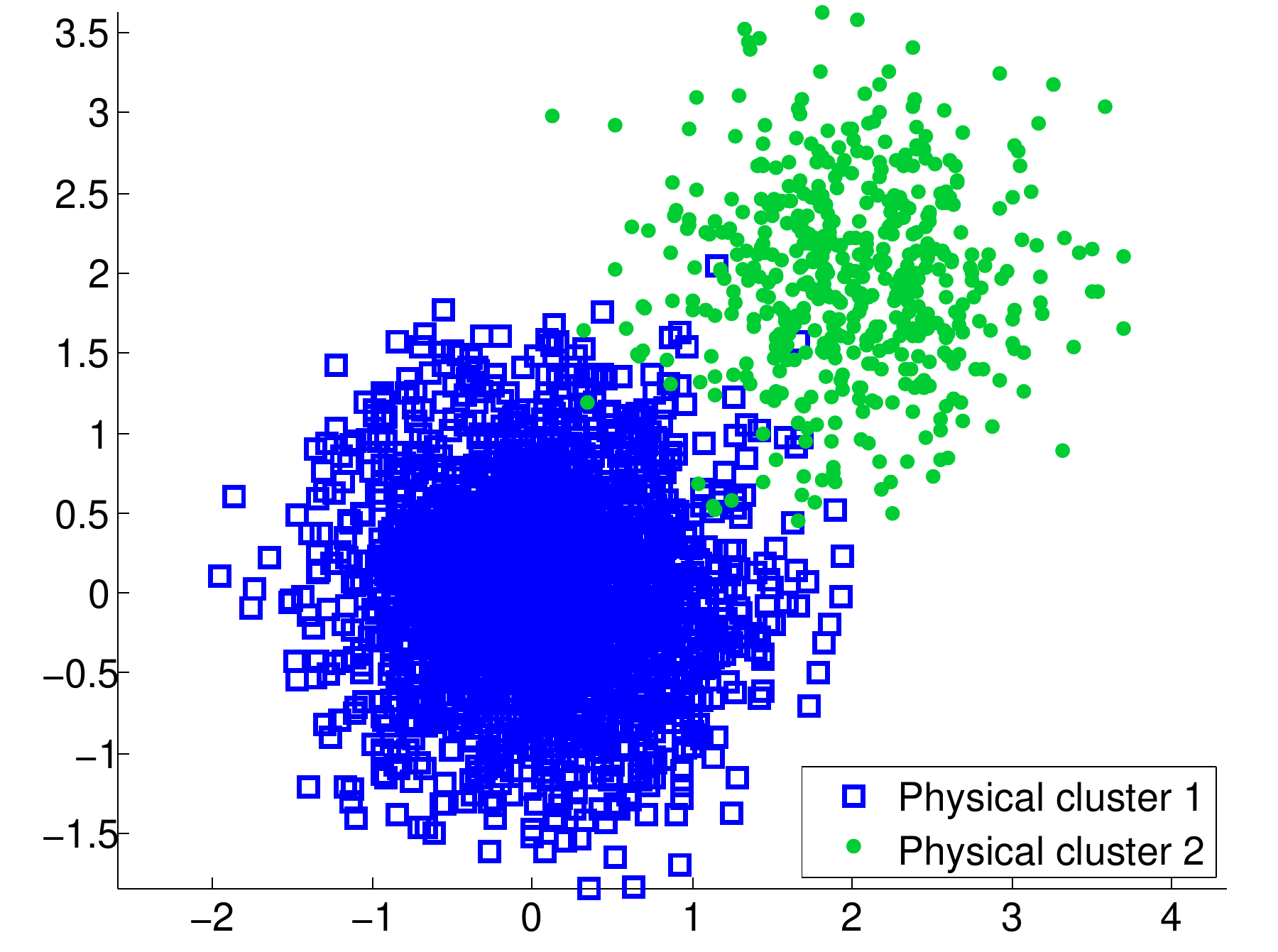}\label{ex4data}}
\hfil
\centering{\caption{(a) The data set of Example 3, (b) The data set of Example 4.}\label{example34}}
\end{figure}

{\bf Example 3:} Consider the set up of Example 1, where now $C_1$ and $C_2$ consist of 2000 and 500 points, respectively. Note that the clusters have the same variances yet even more different densities compared to the data set of Example 1, while at the same time they are located very close to each other, as shown in Fig.~\ref{ex3data}. Table~\ref{table:Example34} shows the clustering results of APCM and SAPCM and Figs.~\ref{ex3-APCM} and \ref{ex3-SAPCM} depict the performance of APCM and SAPCM, respectively, with their parameter $\alpha$ chosen as stated in the figure caption (after fine-tuning). As it can be deduced from Fig.~\ref{Example3} and Table~\ref{table:Example34}, APCM fails to uncover the underlying clustering structure, whereas SAPCM distinguishes the two physical clusters and achieves very satisfactory results in terms of RM, SR and MD. To see why this happens, let us focus on $\boldsymbol{\theta}_1$ and $\boldsymbol{\theta}_2$ in Figs.~\ref{ex3-APCM} and~\ref{ex3-SAPCM}. Clearly, APCM fails to recover $C_2$ since, in determining the next location of $\boldsymbol{\theta}_2$ the many small contributions from the points of $C_1$ gradually prevail over the less but larger contributions from the points of $C_2$. Note that this happens despite the fact that APCM adjusts dynamically the $\gamma_j$'s and it is oughted to the combination of (a) the strict positivity of all $u_{ij}$'s, (b) the very different cluster densities and (c) the closeness of the clusters. However, this is not the case for SAPCM, since the latter annihilates the contributions of the points of $C_1$ in the determination of the next location of $\boldsymbol{\theta}_2$, via the imposition of sparsity.

\begin{figure}[htpb!]
\centering
\subfloat[APCM]{\includegraphics[width=0.4\textwidth]{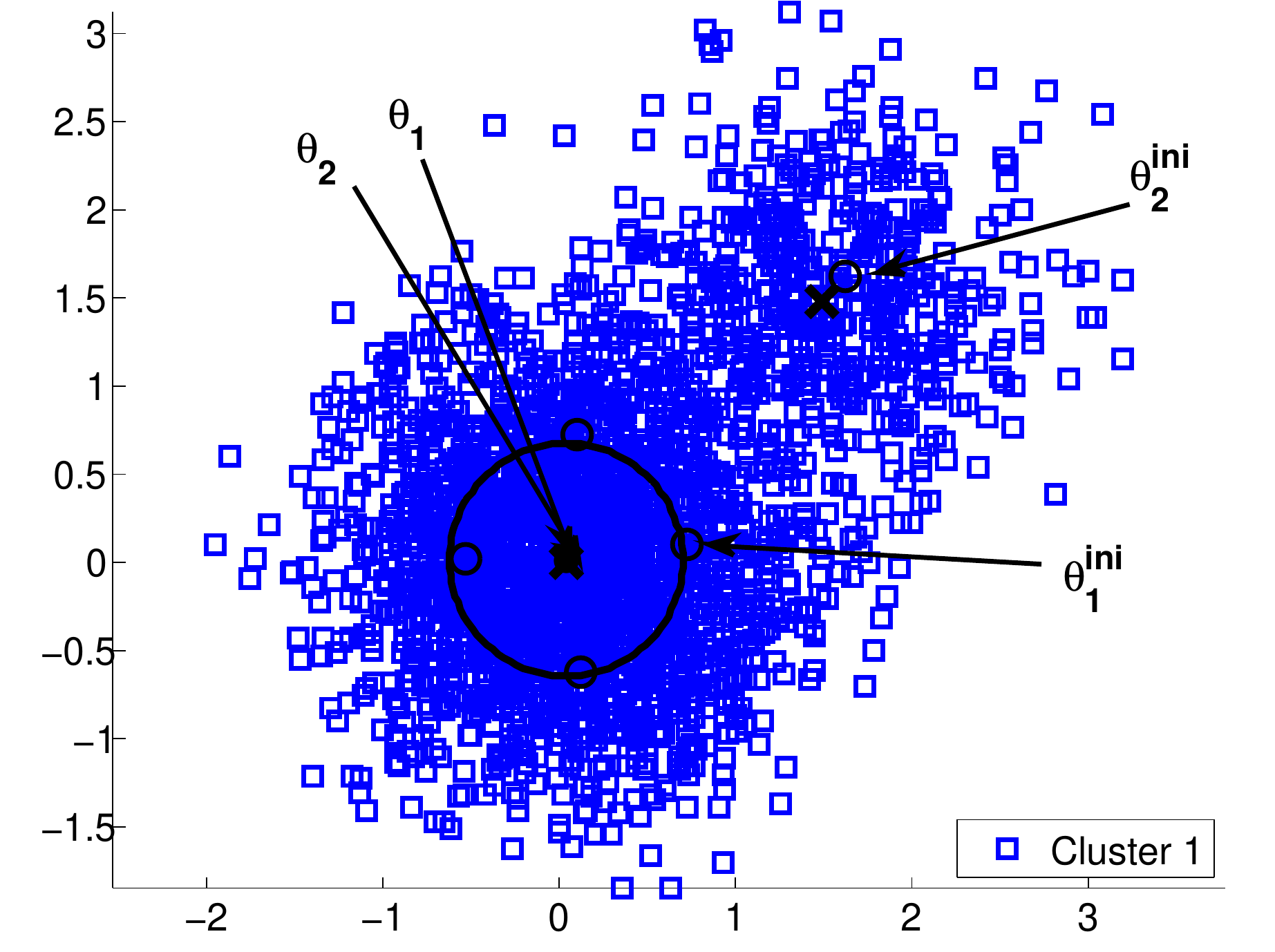}\label{ex3-APCM}}
\hfil
\centering
\subfloat[SAPCM]{\includegraphics[width=0.4\textwidth]{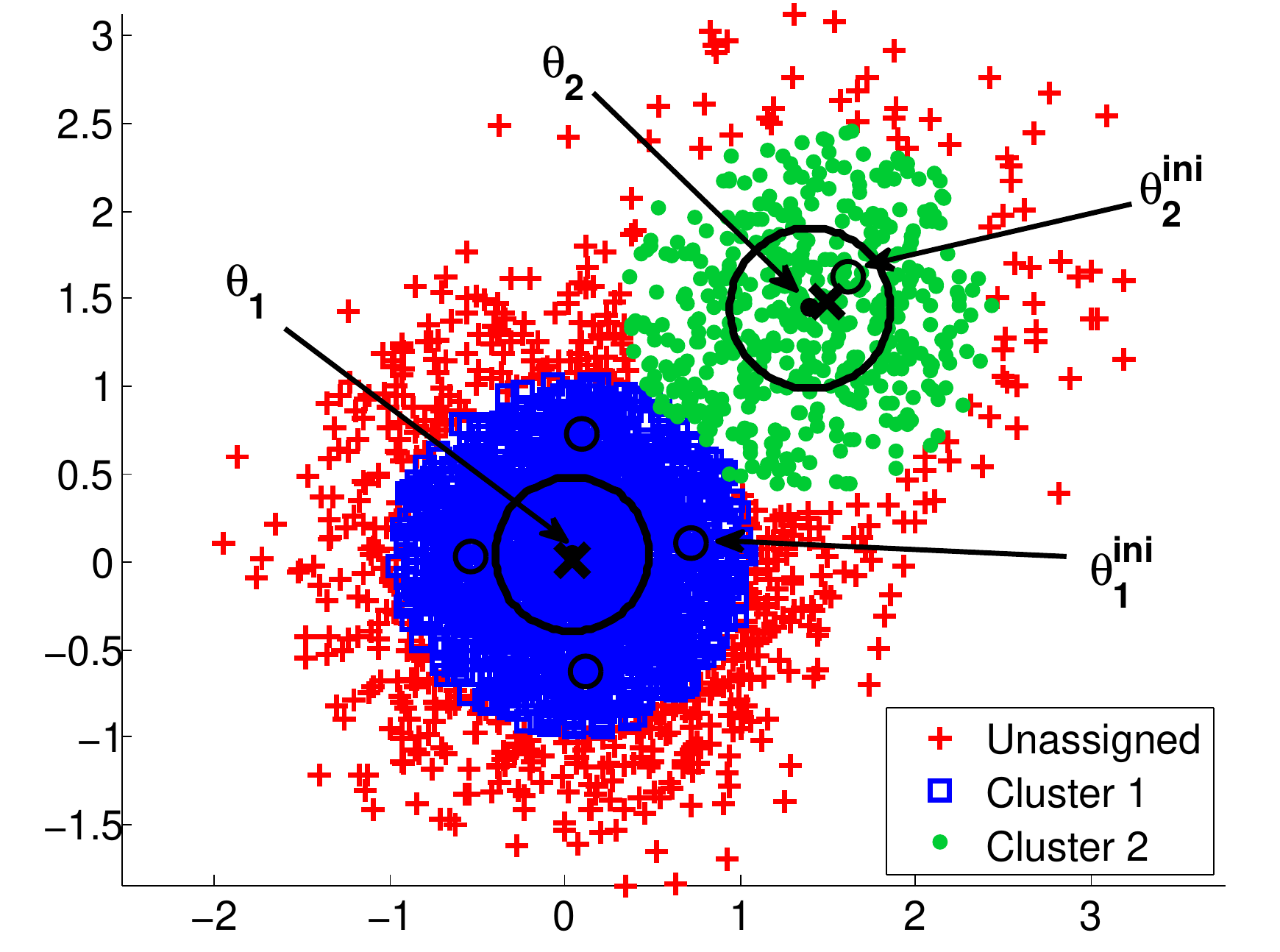}\label{ex3-SAPCM}}
\hfil
\centering{\caption{The clustering results of Example 3 for (a) APCM, $m_{ini}=5$ and $\alpha=1.6$ (b) SAPCM, $m_{ini}=5$ and $\alpha=2$. See also Fig.~\ref{example1PCM} caption.}\label{Example3}}
\end{figure}

\begin{figure}[htpb!]
\centering
\subfloat[APCM]{\includegraphics[width=0.4\textwidth]{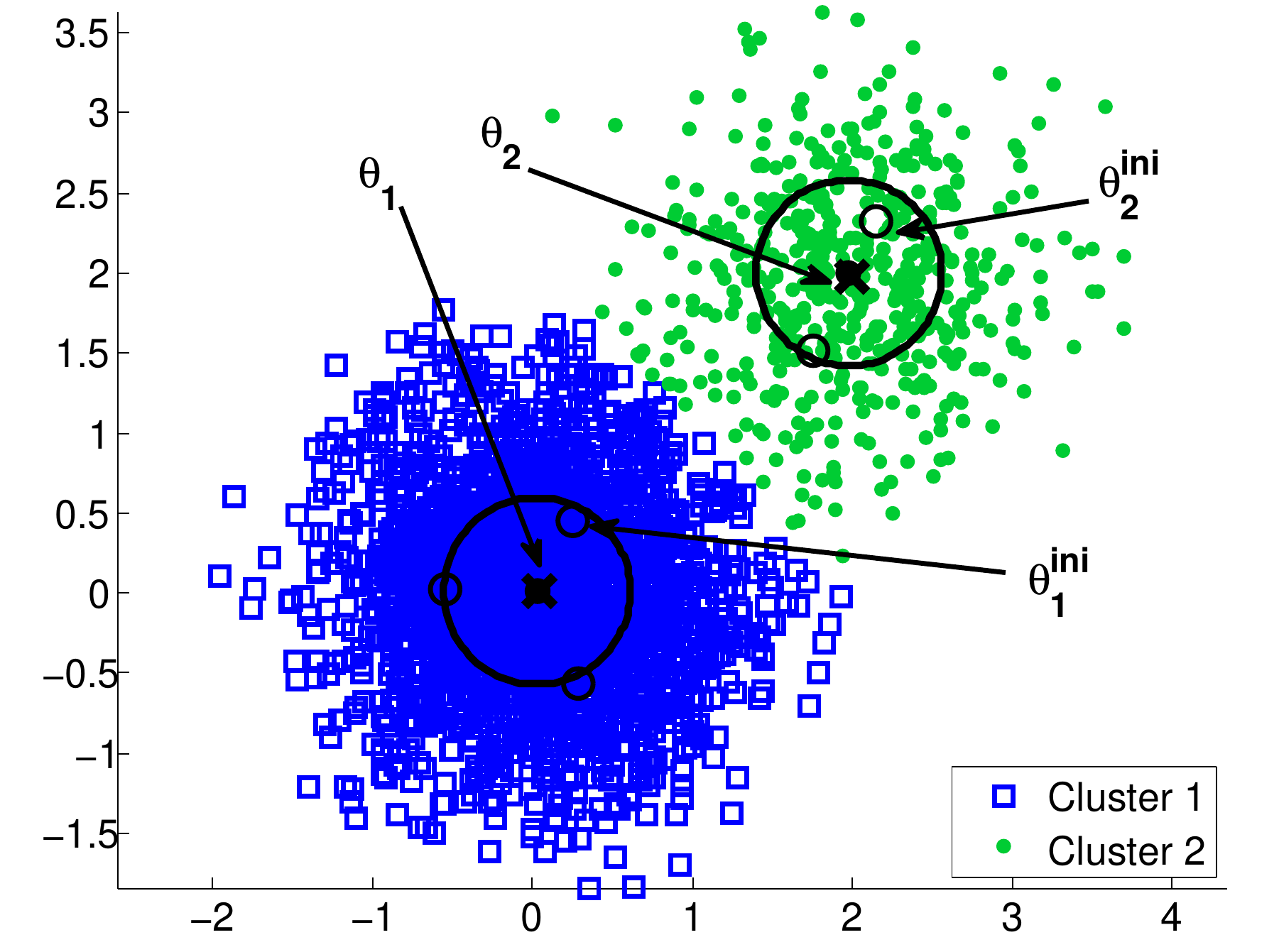}\label{ex4-APCM}}
\hfil
\centering
\subfloat[SAPCM]{\includegraphics[width=0.4\textwidth]{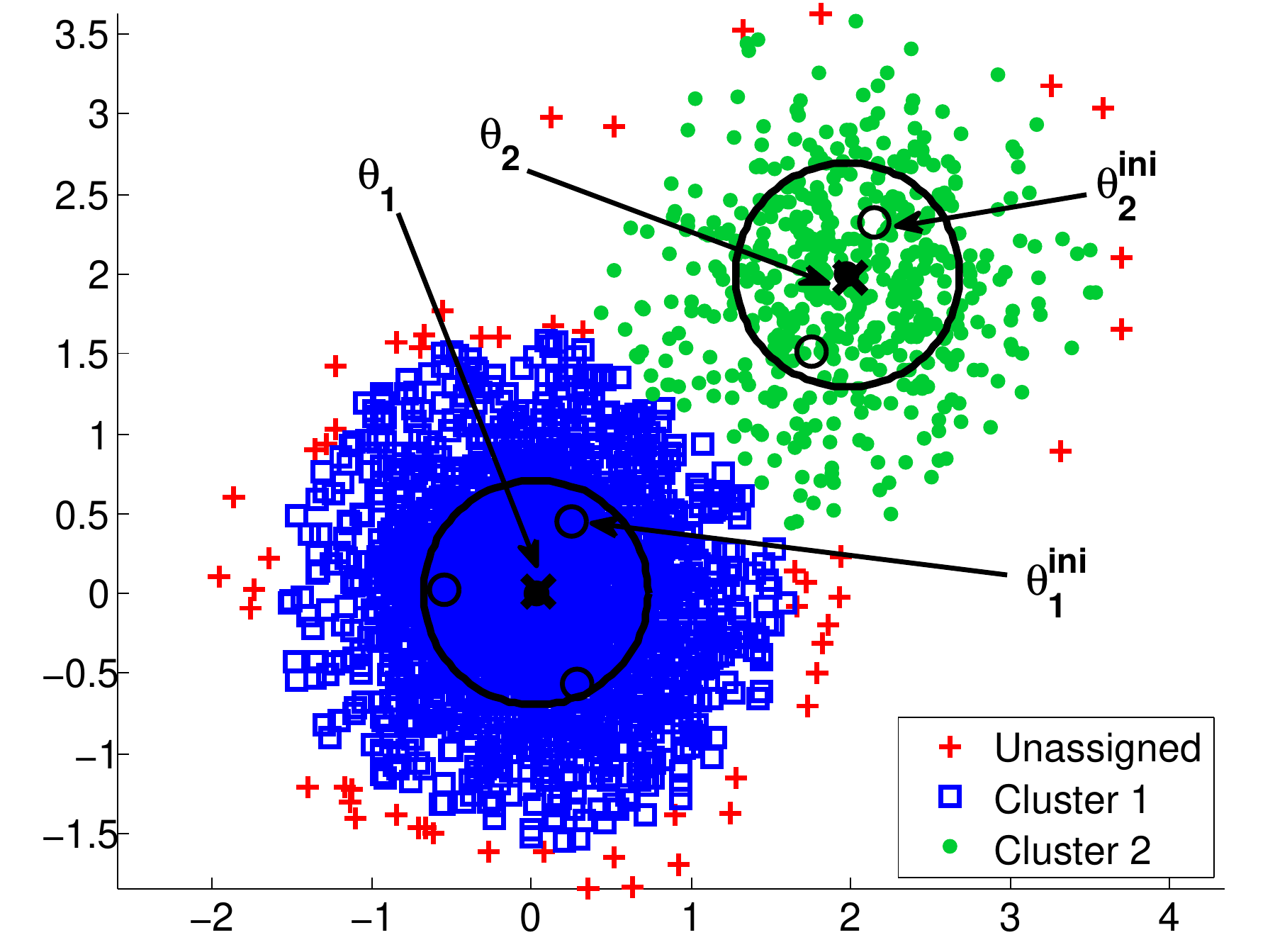}\label{ex4-SAPCM}}
\hfil
\centering{\caption{The clustering results of Example 4 for (a) APCM, $m_{ini}=5$ and $\alpha=1.5$ (b) SAPCM, $m_{ini}=5$ and $\alpha=1$. See also Fig.~\ref{example1PCM} caption.}\label{Example4}}
\end{figure}

\begin{table}[htpb!]
\centering
\caption{Performance of APCM and SAPCM for the data sets of Examples 3 and 4.}
{\small
\begin{tabular}{>{\arraybackslash}m{0.3\linewidth} | >{\centering\arraybackslash}m{0.15\linewidth} |>{\centering\arraybackslash}m{0.06\linewidth}| >{\centering\arraybackslash}m{0.07\linewidth} |>{\centering\arraybackslash}m{0.07\linewidth} |>{\centering\arraybackslash}m{0.07\linewidth} |>{\centering\arraybackslash}m{0.07\linewidth}}
\hline  
	& \centering Data Set & \centering $m_{ini}$ & \centering $m_{final}$ & \centering RM & \centering SR & {\centering MD} \\
\hline
APCM  ($\alpha=1.6$) & Example 3 & 5 & 1 & 67.99 & 80.00 & 1.0368 \\
SAPCM ($\alpha=2$)   & Example 3 & 5 & 2 & 90.07 & 94.76 & 0.0673 \\
\hline
APCM 	($\alpha=1.5$) & Example 4 & 5 & 2 & 97.86 & 98.92 & 0.0324 \\
SAPCM ($\alpha=1$)   & Example 4 & 5 & 2 & 97.78 & 98.88 & 0.0142 \\
\hline
\end{tabular}}
\label{table:Example34}
\end{table}

{\bf Example 4:} Consider now the same two-dimensional data set of Example 3, where now the means of the two normal distributions are $\mathbf{c}_1=[0, 0]^T$ and $\mathbf{c}_2=[2, 2]^T$, respectively, as shown in Fig.~\ref{ex4data}. Table~\ref{table:Example34} shows the clustering results of APCM and SAPCM and Figs.~\ref{ex4-APCM} and \ref{ex4-SAPCM} depict the performance of APCM and SAPCM, respectively. As it can be deduced, APCM is now able to uncover the underlying clustering structure. However, SAPCM manages to detect even more accurately the true centers of the clusters (as MD index indicates).

{\it Remark:} In SAPCM the parameter $\lambda$ is chosen as in SPCM, as eq.~\eqref{lambda} indicates. Note that in SAPCM, the parameters $\gamma_j$'s are updated during the execution of the algorithm, thus the parameter $\lambda$ should also be updated after the adaptation of $\gamma_j$'s (see line 29 in Algorithm 2). Moreover, in SAPCM the parameter $K$ should take much smaller values than in SPCM, due to the definition of $\gamma_j$'s. This has to do with the fact that in SAPCM the adaptation of the parameters $\gamma_j$'s leads to more accurate estimates for the variances of the clusters (see the radius of the circles ($\sqrt{\gamma_j}$) in Figs.~\ref{ex1-SPCM},~\ref{ex2-SPCM} for SPCM and the corresponding ones for SAPCM in Figs.~\ref{ex3-SAPCM},~\ref{ex4-SAPCM} and \cite{Xen15}). Taking into accound that (a) the choice of eq.~\eqref{lambda} imposes sparsity for all the points at distance greater than $\min\limits_{j=1,\ldots,m} \gamma_j$ from a given representative and (b) $\gamma_j$'s in SAPCM are of much smaller sizes with respect to their corresponding ones in SPCM, values of $K$ close to 1 would lead to such a large degree of sparsity (as indicated by $f(u_{ij})$ in eq.~\eqref{f2u}), where the cluster representatives could hardly move (through eq.~\eqref{theta}, see line 10 in Alg.~2). Therefore, in all SAPCM experiments we set $K=0.1$.

\section{Experimental results}
In this section, we assess the performance of the proposed methods in several experimental settings and illustrate the results. More specifically, we use several two-dimensional simulated data sets as well as real-world data sets (Iris \cite{UCILib} and a hyperspectral image data set \cite{HsiSal}) to evaluate the performance of SPCM and SAPCM in comparison with several other related algorithms.

{\bf Experiment 1}:
This experiment illustrates the rationale of SPCM, which has been approached in Example 1 more qualitatively. Let us consider a two-dimensional data set consisting of $N=17$ points, which form two clusters $C_1$ and $C_2$ with 12 and 5 data points, respectively (see Fig.~\ref{example1}). The means of the clusters are $\mathbf{c}_1=[1.75, 2.75]$ and $\mathbf{c}_2=[4.25, 2.75]$. In this experiment, we consider only the PCM and the SPCM algorithms, both with $m=2$. Figs.~\ref{ex1PCMinit} and \ref{ex1SPCMinit} show the initial positions of the cluster representatives that are taken from FCM and the circles with radius equal to $\sqrt{\gamma_j}$'s resulting from eq.~(\ref{Ketaj}) (for $B=1$) for both PCM and SPCM. Similarly, Figs.~\ref{ex1PCM1st} and \ref{ex1SPCM1st} show the new locations of $\boldsymbol{\theta}_j$'s after the first iteration of the algorithms and Figs.~\ref{ex1PCM8th},~\ref{ex1SPCM5th} show the locations of $\boldsymbol{\theta}_j$'s after the 8th and 5th (final) iterations for PCM and SPCM, respectively. Table~\ref{table:ex1} shows the degrees of compatibility $u_{ij}$'s of all data points $\mathbf{x}_i$'s with the cluster representatives $\boldsymbol{\theta}_j$'s at the three iterations considered in Fig.~\ref{example1} for both PCM and SPCM. 

\begin{figure}[htpb!]
\centering
\subfloat[Initial step of PCM]{\includegraphics[width=0.32\textwidth]{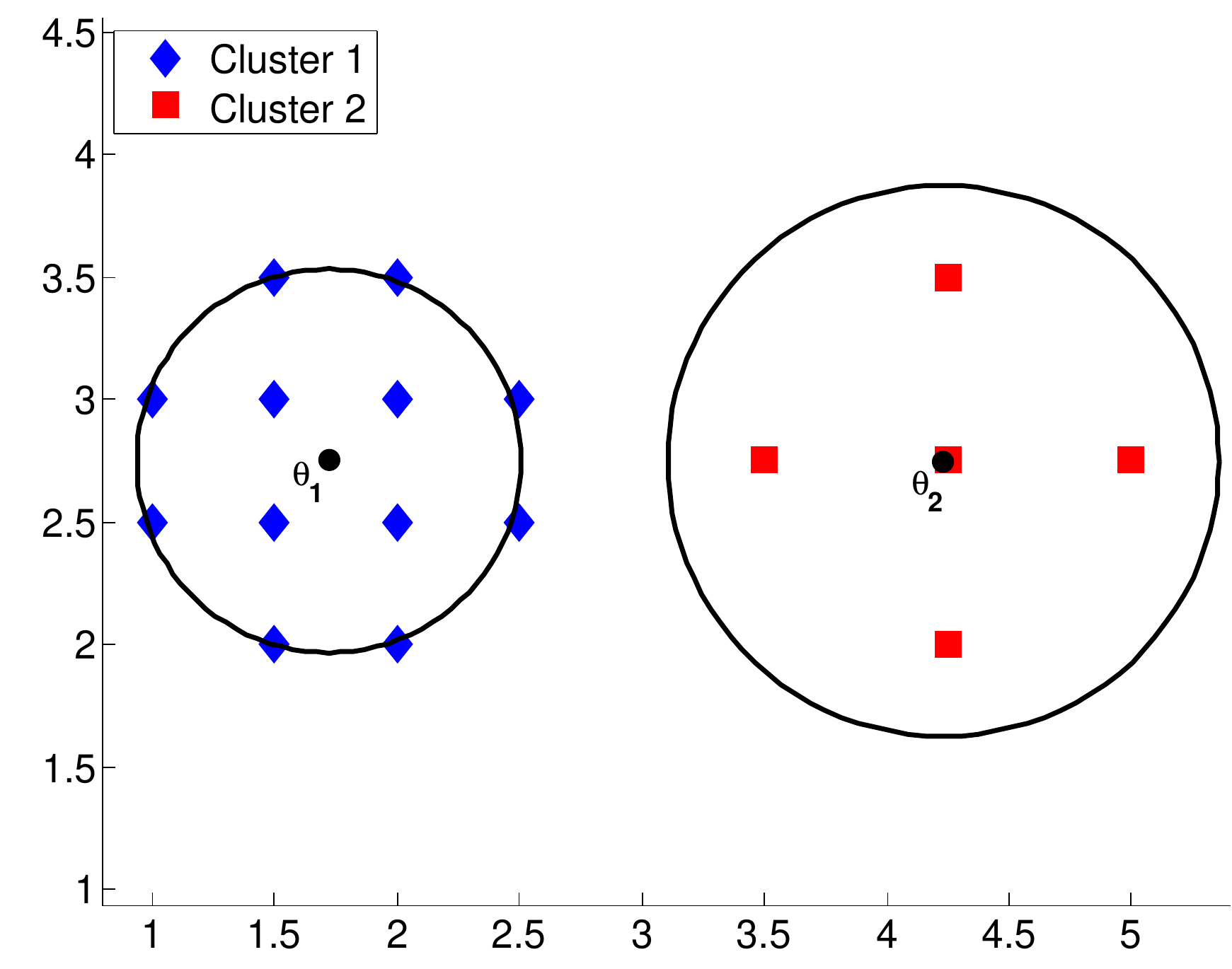}\hspace{6pt}\label{ex1PCMinit}}
\hfil
\centering
\subfloat[$1^{st}$ iteration of PCM]{\includegraphics[width=0.32\textwidth]{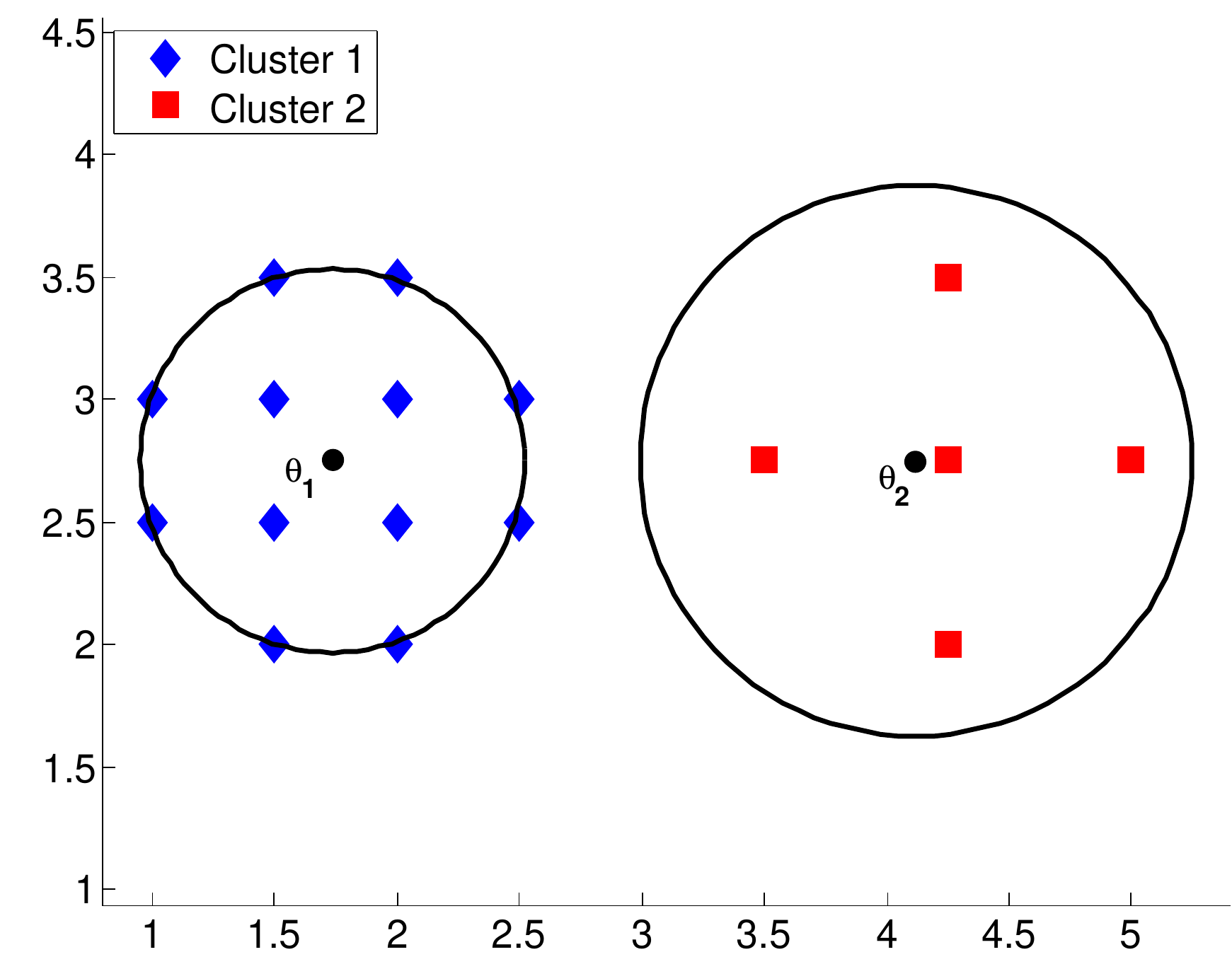}\hspace{6pt}\label{ex1PCM1st}}
\hfil
\centering
\subfloat[$8^{th}$ iteration of PCM]{\includegraphics[width=0.32\textwidth]{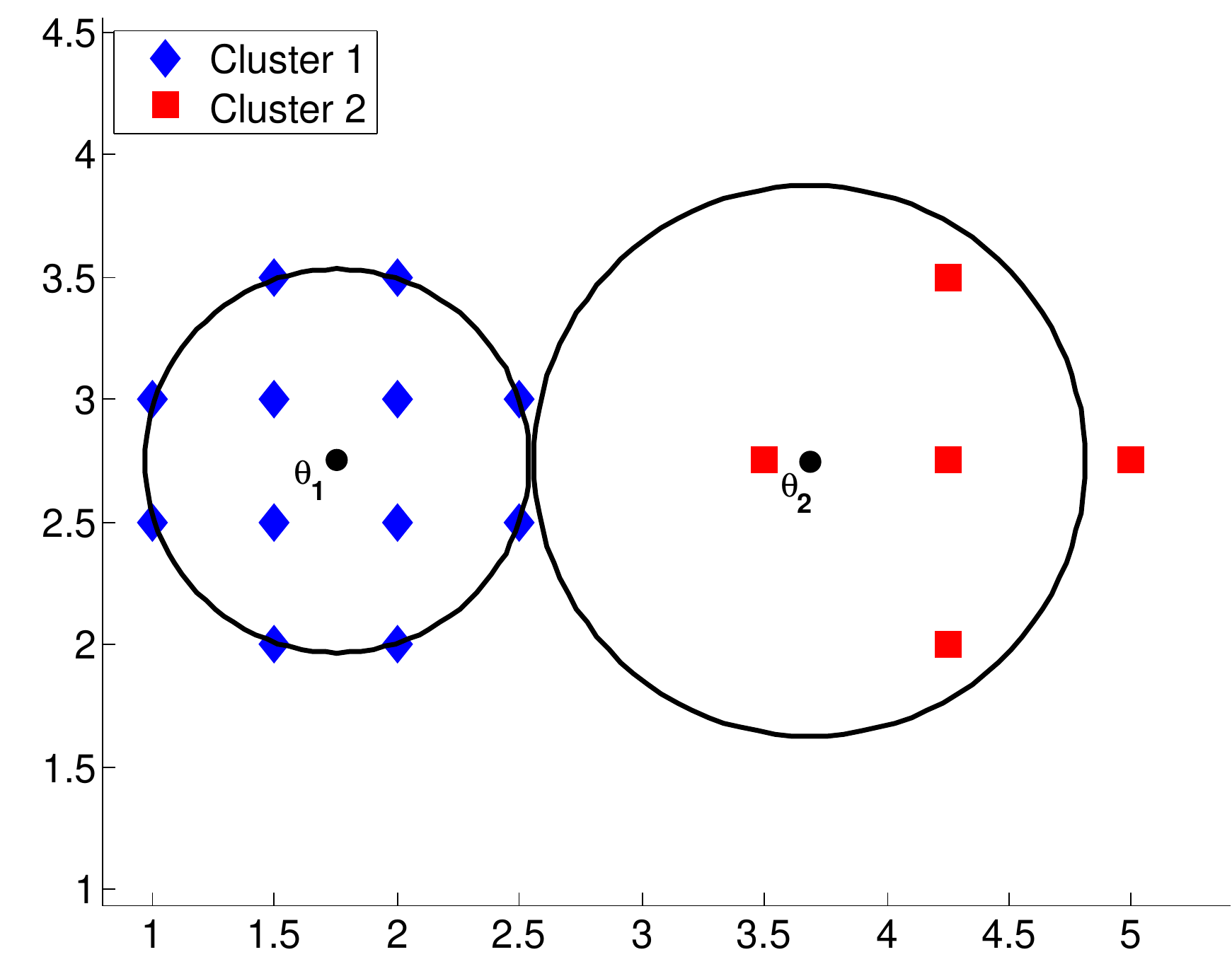}\hspace{6pt}\label{ex1PCM8th}}
\hfil
\centering
\subfloat[Initial step of SPCM]{\includegraphics[width=0.33\textwidth]{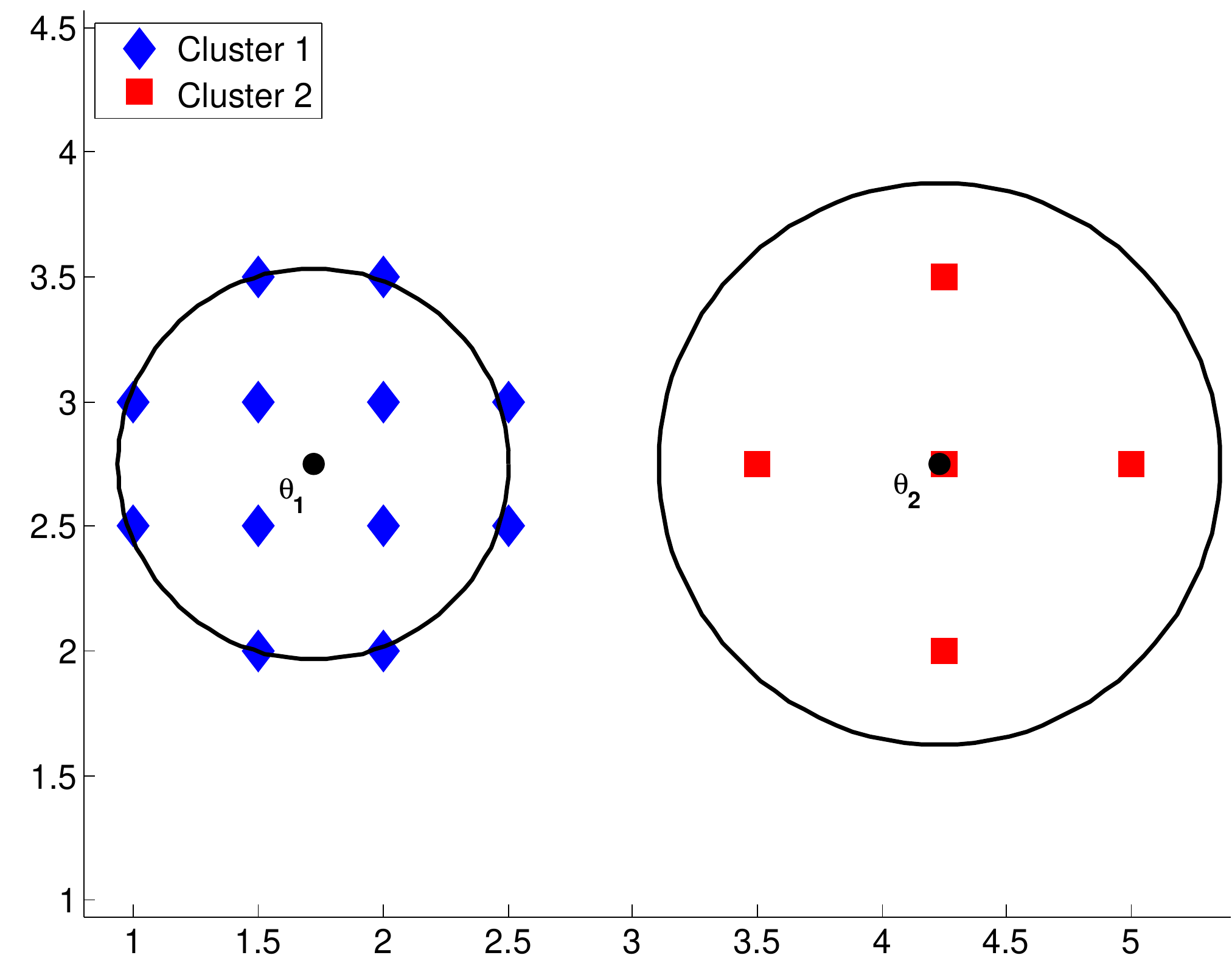}\hspace{-6pt}\label{ex1SPCMinit}}
\hfil
\centering
\subfloat[$1^{st}$ iteration of SPCM]{\includegraphics[width=0.33\textwidth]{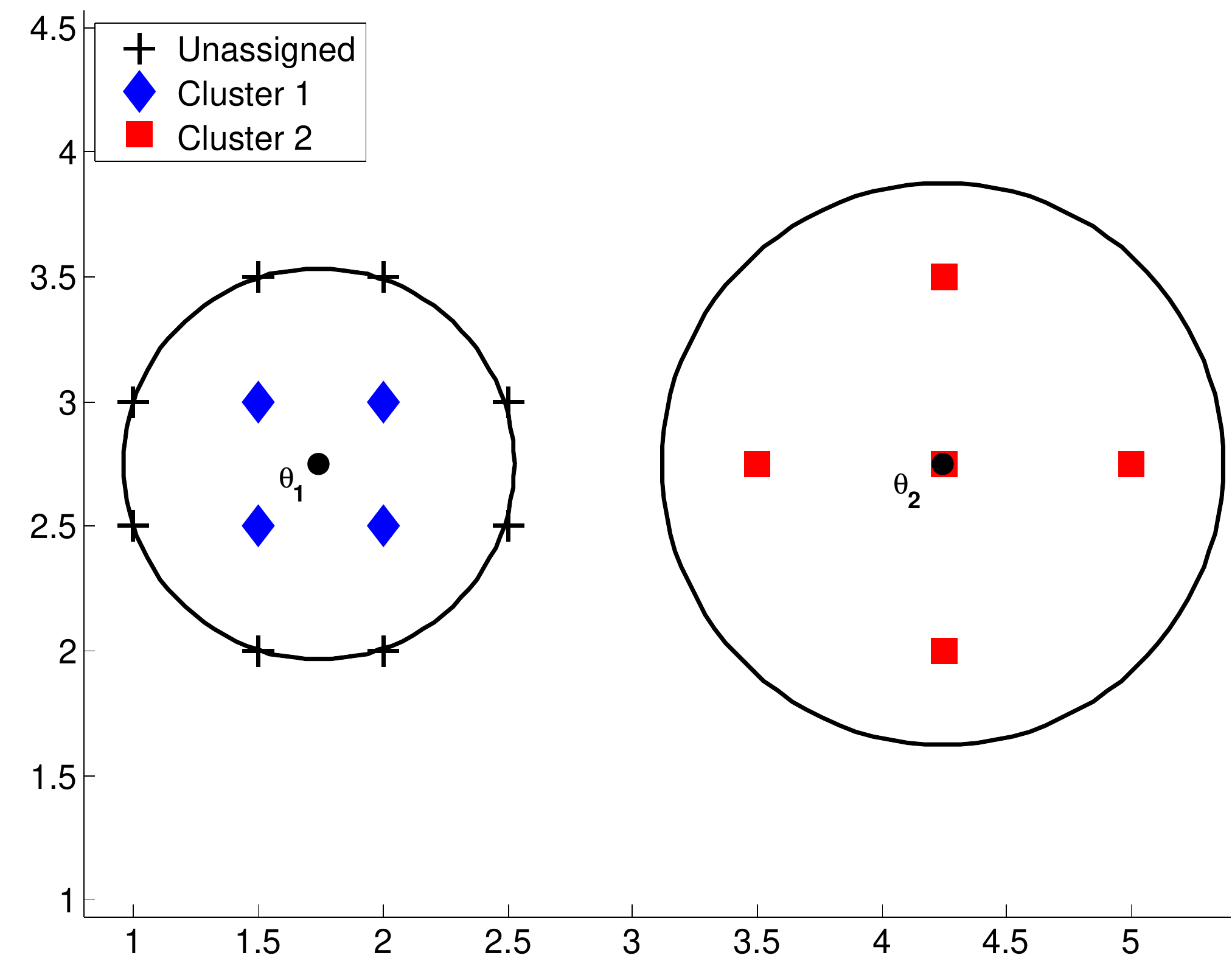}\hspace{-7pt}\label{ex1SPCM1st}}
\hfil
\centering
\subfloat[$5^{th}$ iteration of SPCM]{\includegraphics[width=0.33\textwidth]{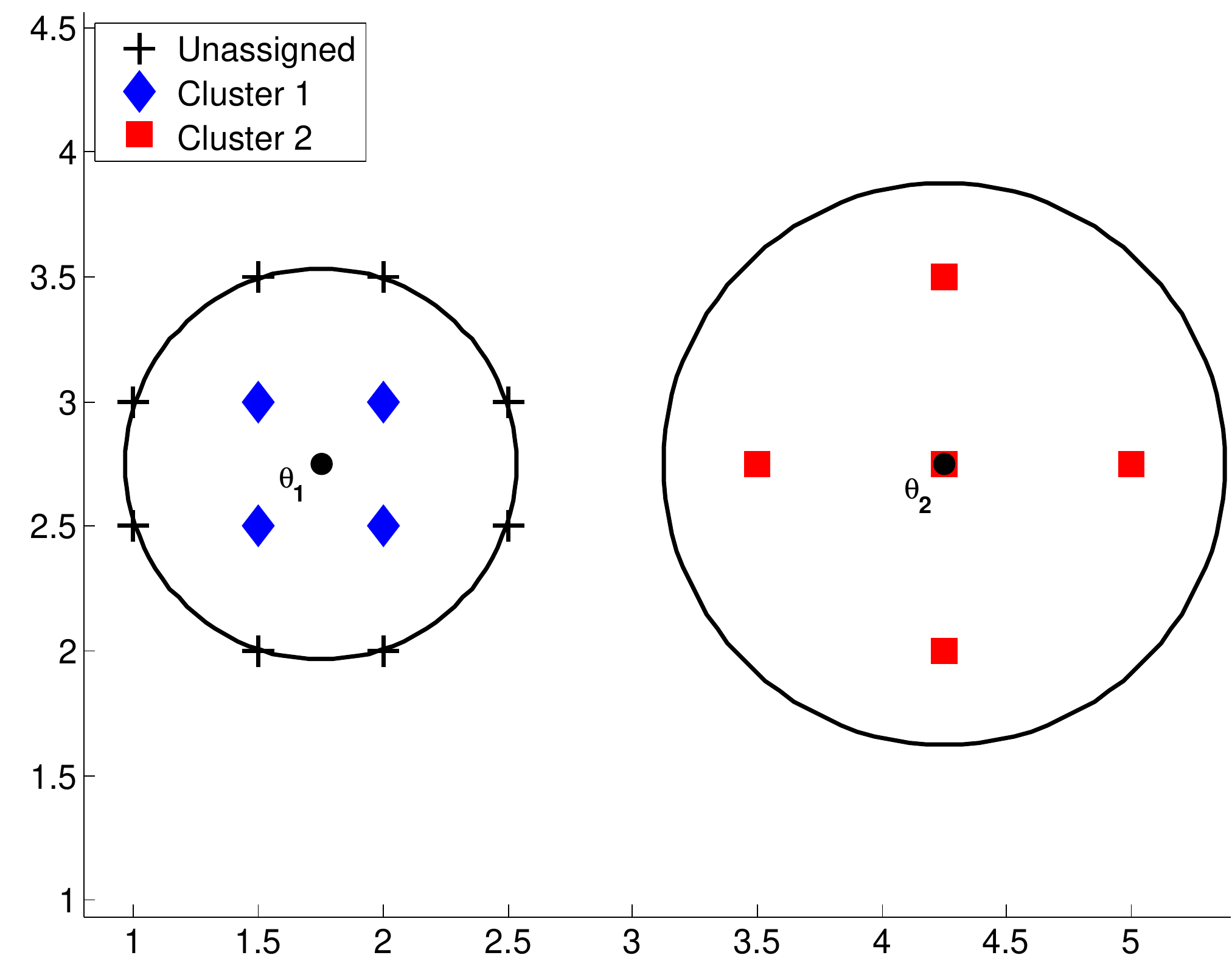}\hspace{-8pt}\label{ex1SPCM5th}}
\hfil
\centering{\caption{PCM and SPCM snapshots at their initialization step, their first iteration and 8th iteration for PCM and 5th (final) iteration for SPCM (Experiment 1).}\label{example1}}
\end{figure}

As it can be deduced from Table~\ref{table:ex1} and Fig.~\ref{example1}, the degrees of compatibility of the data points of $C_1$ with the cluster representative $\boldsymbol{\theta}_2$ increase as PCM evolves, leading gradually $\boldsymbol{\theta}_2$ towards the region of the cluster $C_1$ and thus, ending up with two coincident clusters, although $\boldsymbol{\theta}_1$ and $\boldsymbol{\theta}_2$ are initialized properly through the FCM algorithm (see Fig.~\ref{ex1PCMinit}). However, this is not the case in SPCM algorithm, as both the cluster representatives remain in the centers of the actual clusters. It is of great interest to mention that in SPCM $\boldsymbol{\theta}_1$ and $\boldsymbol{\theta}_2$ conclude closest to the actual centers compared to its initial state through the FCM algorithm (see Fig.~\ref{ex1SPCM5th}). Obviously, the superior performance of SPCM is due to the sparsity imposed on $\mathbf{u}_i$'s leading several $u_{ij}$'s to 0 for points $\mathbf{x}_i$ that do not ``belong" to $C_j$ (see Table~\ref{table:ex1}), thus preventing these points from contributing to the estimation of $\boldsymbol{\theta}_j$. This experiment indicates that, in principle, SPCM can handle successfully cases where relatively closely located clusters with different densities are involved.

\begin{table}[htpb!]
\centering
\caption{The degrees of compatibility of the data points of Experiment 1 for PCM and SPCM algorithms, after: (a) initialization (common to both algorithms), (b) first iteration (for both algorithms) and (c) 8th iteration for PCM and 5th (final) iteration for SPCM.}
{\footnotesize
\begin{tabular}{>{\centering\arraybackslash}m{0.085\textwidth} | >{\centering\arraybackslash}m{0.065\textwidth} | >{\centering\arraybackslash}m{0.05\textwidth} ||>{\centering\arraybackslash}m{0.065\textwidth} | >{\centering\arraybackslash}m{0.065\textwidth} || >{\centering\arraybackslash}m{0.065\textwidth} |>{\centering\arraybackslash}m{0.065\textwidth}  || >{\centering\arraybackslash}m{0.065\textwidth} |>{\centering\arraybackslash}m{0.05\textwidth}  || >{\centering\arraybackslash}m{0.065\textwidth} |>{\centering\arraybackslash}m{0.065\textwidth}}
\hline  
	& \multicolumn{2}{c||}{Initialization}  & \multicolumn{4}{c||}{$1^{st}$ iteration} & \multicolumn{2}{c||}{$8^{th}$ iteration} & \multicolumn{2}{c}{$5^{th}$ iteration} \\
\cline{2-11}
	& \multicolumn{2}{c||}{PCM/SPCM} &  \multicolumn{2}{c||}{PCM} & \multicolumn{2}{c||}{SPCM} & \multicolumn{2}{c||}{PCM} & \multicolumn{2}{c}{SPCM} \\
\cline{2-11}
	\centering $\mathbf{x}_i$ & \centering $C_1$ & \centering $C_2$ & \centering $C_1$ & \centering $C_2$ & \centering $C_1$ & {\centering $C_2$} & \centering $C_1$ & \centering $C_2$ & \centering $C_1$ & {\centering $C_2$}\\
\hline
{\footnotesize$(1.5, 3.5)$}  	 & \centering 0.9292 & \centering 0.0708 & \centering 0.3701 & \centering 0.0018 & \centering 0 & {\centering 0} & \centering 0.3606 & \centering 0.0118 & \centering 0 & {\centering 0} \\
{\footnotesize$(2.0, 3.5)$}    & \centering 0.8963 & \centering 0.1037 & \centering 0.3526 & \centering 0.0127   & \centering 0 & {\centering 0} & \centering 0.3630 & \centering 0.0570 & \centering 0 & {\centering 0} \\
{\footnotesize$(1.0, 3.0)$}    & \centering 0.9475 & \centering 0.0525 & \centering 0.3884 & \centering 2.5e-04 & \centering 0 & {\centering 0} & \centering 0.3583 & \centering 0.0024 & \centering 0 & {\centering 0} \\
{\footnotesize$(1.5, 3.0)$}    & \centering 0.9854 & \centering 0.0146 & \centering 0.8348 & \centering 0.0027   & \centering 0.4625 & {\centering 0} & \centering 0.8134 & \centering 0.0174 & \centering 0.4478 & {\centering 0}\\
{\footnotesize$(2.0, 3.0)$}    & \centering 0.9728 & \centering 0.0272 & \centering 0.7954 & \centering 0.0188   & \centering 0.4316 & {\centering 0} & \centering 0.8186 & \centering 0.0846 & \centering 0.4476 & {\centering 0}\\
{\footnotesize$(2.5, 3.0)$}    & \centering 0.8201 & \centering 0.1799 & \centering 0.3360 & \centering 0.0897   & \centering 0 & {\centering 0} & \centering 0.3653 & \centering 0.2766 & \centering 0 & {\centering 0}\\
{\footnotesize$(1.0, 2.5)$}    & \centering 0.9475 & \centering 0.0525 & \centering 0.3884 & \centering 2.5e-04 & \centering 0 & {\centering 0} & \centering 0.3583 & \centering 0.0024 & \centering 0 & {\centering 0}\\
{\footnotesize$(1.5, 2.5)$}    & \centering 0.9854 & \centering 0.0146 & \centering 0.8348 & \centering 0.0027   & \centering 0.4625 & {\centering 0} & \centering 0.8134 & \centering 0.0174 & \centering 0.4478 & {\centering 0}\\
{\footnotesize$(2.0, 2.5)$}    & \centering 0.9728 & \centering 0.0272 & \centering 0.7954 & \centering 0.0188   & \centering 0.4316 & {\centering 0} & \centering 0.8186 & \centering 0.0846 & \centering 0.4476 & {\centering 0}\\
{\footnotesize$(2.5, 2.5)$} & \centering 0.8201 & \centering 0.1799 & \centering 0.3360 & \centering 0.0897   & \centering 0 & {\centering 0} & \centering 0.3653 & \centering 0.2766 & \centering 0 & {\centering 0}\\
{\footnotesize$(1.5, 2.0)$} & \centering 0.9292 & \centering 0.0708 & \centering 0.3701 & \centering 0.0018   & \centering 0 & {\centering 0} & \centering 0.3606 & \centering 0.0118 & \centering 0 & {\centering 0}\\
{\footnotesize$(2.0, 2.0)$} & \centering 0.8963 & \centering 0.1037 & \centering 0.3526 & \centering 0.0127   & \centering 0 & {\centering 0} & \centering 0.3630 & \centering 0.0570 & \centering 0 & {\centering 0}\\
\hline
{\footnotesize$(4.25, 3.5)$} & \centering 0.0748   & \centering 0.9252 & \centering 1.2e-05 & \centering 0.6415 & \centering 0 & {\centering 0.4850} & \centering 1.6e-05 & \centering 0.5276 & \centering 0 & {\centering 0.4852}\\
{\footnotesize$(3.5, 2.75)$} & \centering 0.1441   & \centering 0.8559 & \centering 0.0058   & \centering 0.6566 & \centering 0   & {\centering 0.4983} & \centering 0.0070 & \centering 0.9512 & \centering 0 & {\centering 0.4854}\\
{\footnotesize$(4.25, 2.75)$} & \centering 6.1e-05 & \centering 0.9999 & \centering 3.0e-05 & \centering 0.9997 & \centering 0 & {\centering 0.8046} & \centering 4.0e-05 & \centering 0.8222 & \centering 0 & {\centering 0.8049}\\
{\footnotesize$(5.0, 2.75)$} & \centering 0.0522   & \centering 0.9478 & \centering 2.5e-08 & \centering 0.6267 & \centering 0 & {\centering 0.4720} & \centering 3.6e-08 & \centering 0.2926 & \centering 0 & {\centering 0.4849}\\
{\footnotesize$(4.25, 2.0)$} & \centering 0.0748   & \centering 0.9252 & \centering 1.2e-05 & \centering 0.6415 & \centering 0 & {\centering 0.4850} & \centering 1.6e-05 & \centering 0.5276 & \centering 0 & {\centering 0.4852}\\
\hline
\end{tabular}}
\label{table:ex1}
\end{table}

In the sequel, we compare the clustering performance of SPCM and SAPCM with that of the k-means, the FCM, the PCM \cite{Kris96}, the UPC \cite{Yang06}, the UPFC \cite{Wu10}, the PFCM \cite{Pal05}, the SPCM-$L_1$ \cite{Hama12} and the APCM \cite{Xen15} algorithms, which all result from cost optimization schemes. For a fair comparison, the representatives $\boldsymbol{\theta}_j$'s of all algorithms (except for SPCM-$L_1$) are initialized based on the FCM scheme and the parameters of each algorithm are first fine tuned. Moreover, in PCM, UPC, UPFC, PFCM and SPCM, duplicate clusters are removed. In order to compare a clustering with the true data label information, we utilize again the RM, SR and the MD indices defined previously. In particular, in Experiments 2 and 3 the SR of each physical cluster (SR$_{c_j}, j=1,...,m$) is presented, which measures the percentage of the points of each physical cluster that have been correctly labeled by each algorithm. Finally, the number of iterations and the total time required for the convergence of each algorithm, is provided. All algorithms are executed using MATLAB R2013a on Intel i7-4790 machine with 16 GB RAM and 3.60 GHz speed.

\begin{figure}[htpb!]
\centering
\subfloat[The data set]{\includegraphics[width=0.32\textwidth]{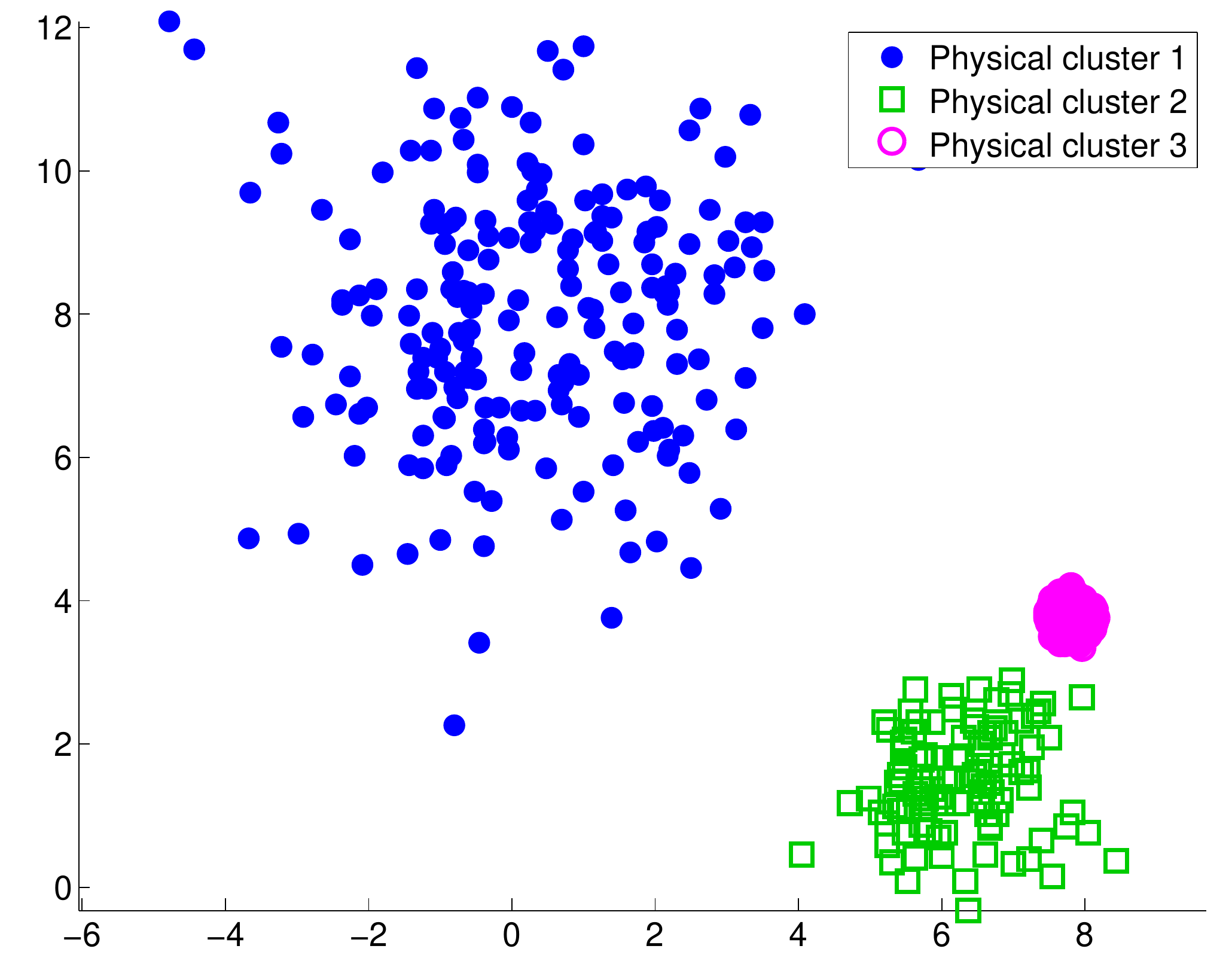}\label{exp2dataset}}
\hfil
\centering
\subfloat[k-means]{\includegraphics[width=0.32\textwidth]{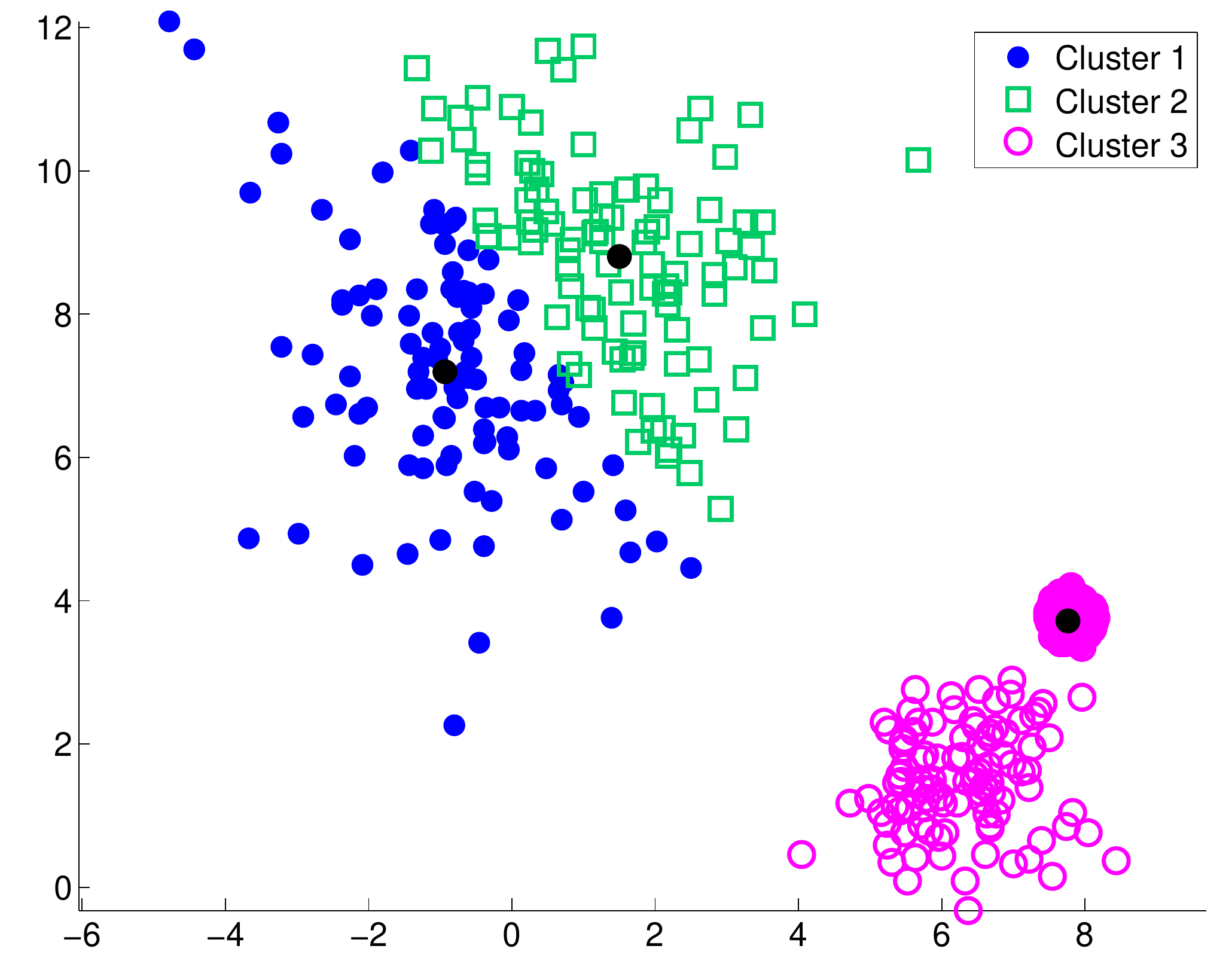}\label{exp2kmeans}}
\hfil
\centering
\subfloat[FCM]{\includegraphics[width=0.32\textwidth]{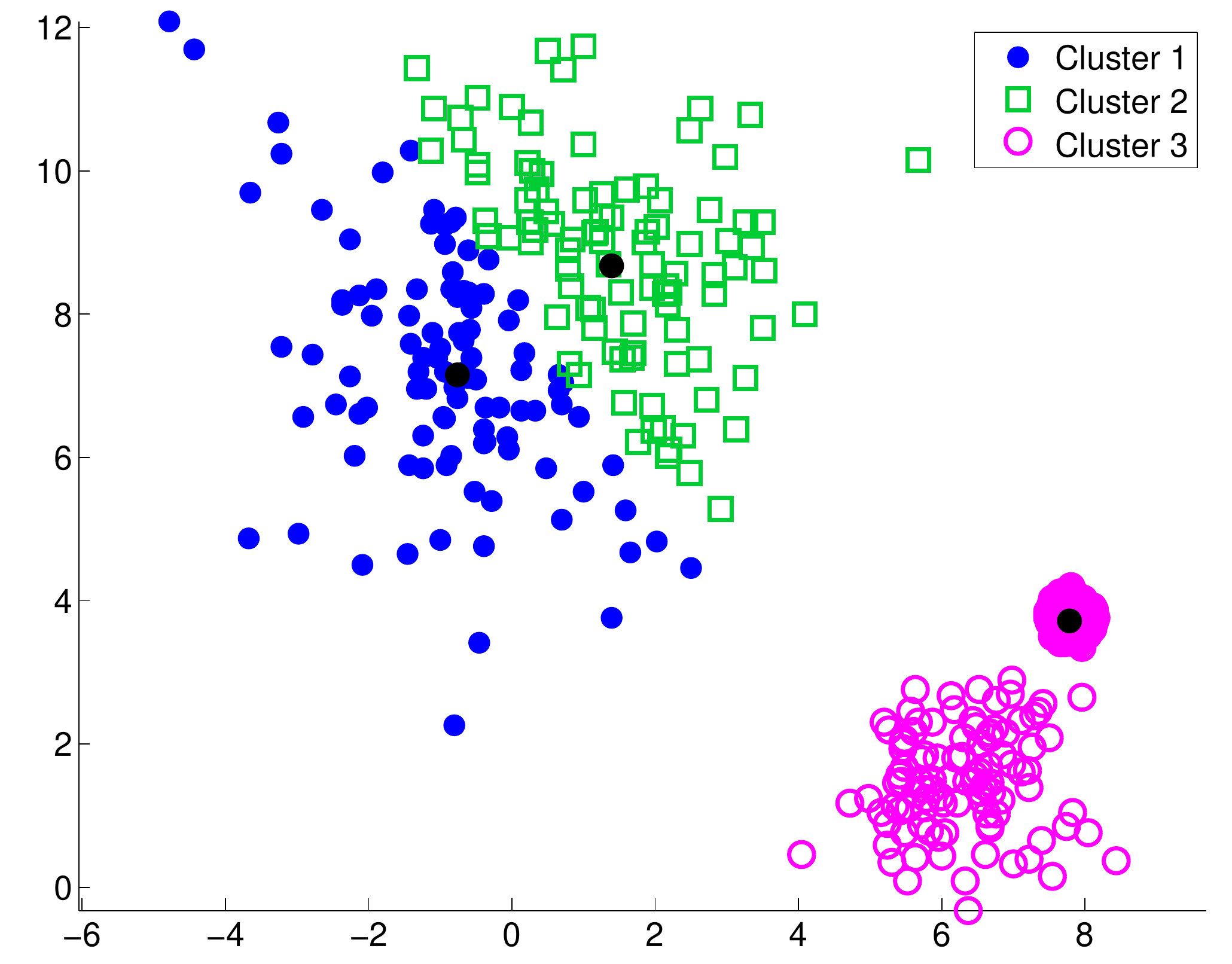}\label{exp2fcm}}
\hfil
\centering
\subfloat[PCM]{\includegraphics[width=0.33\textwidth]{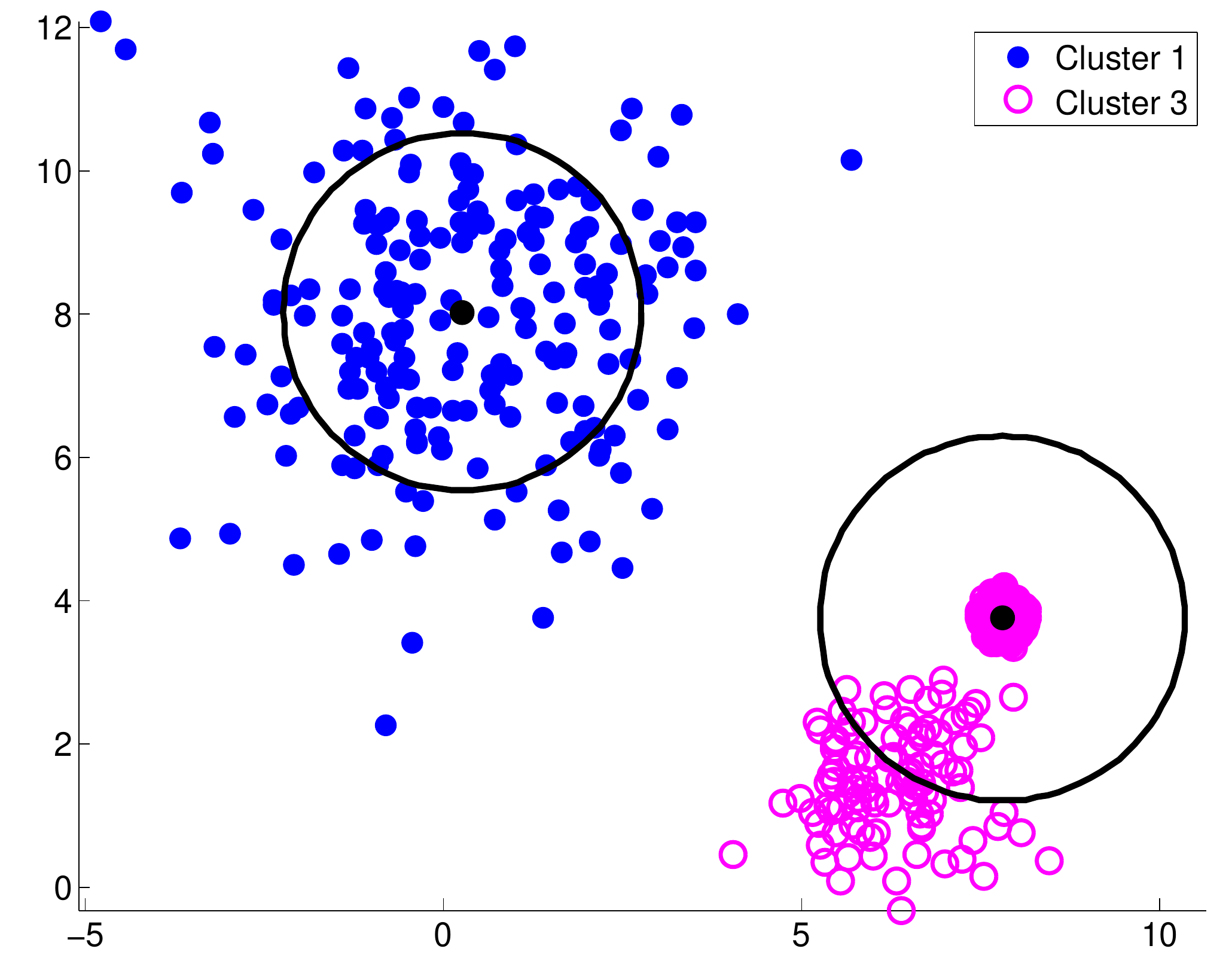}\label{exp2pcm}}
\hfil
\centering
\subfloat[UPC]{\includegraphics[width=0.33\textwidth]{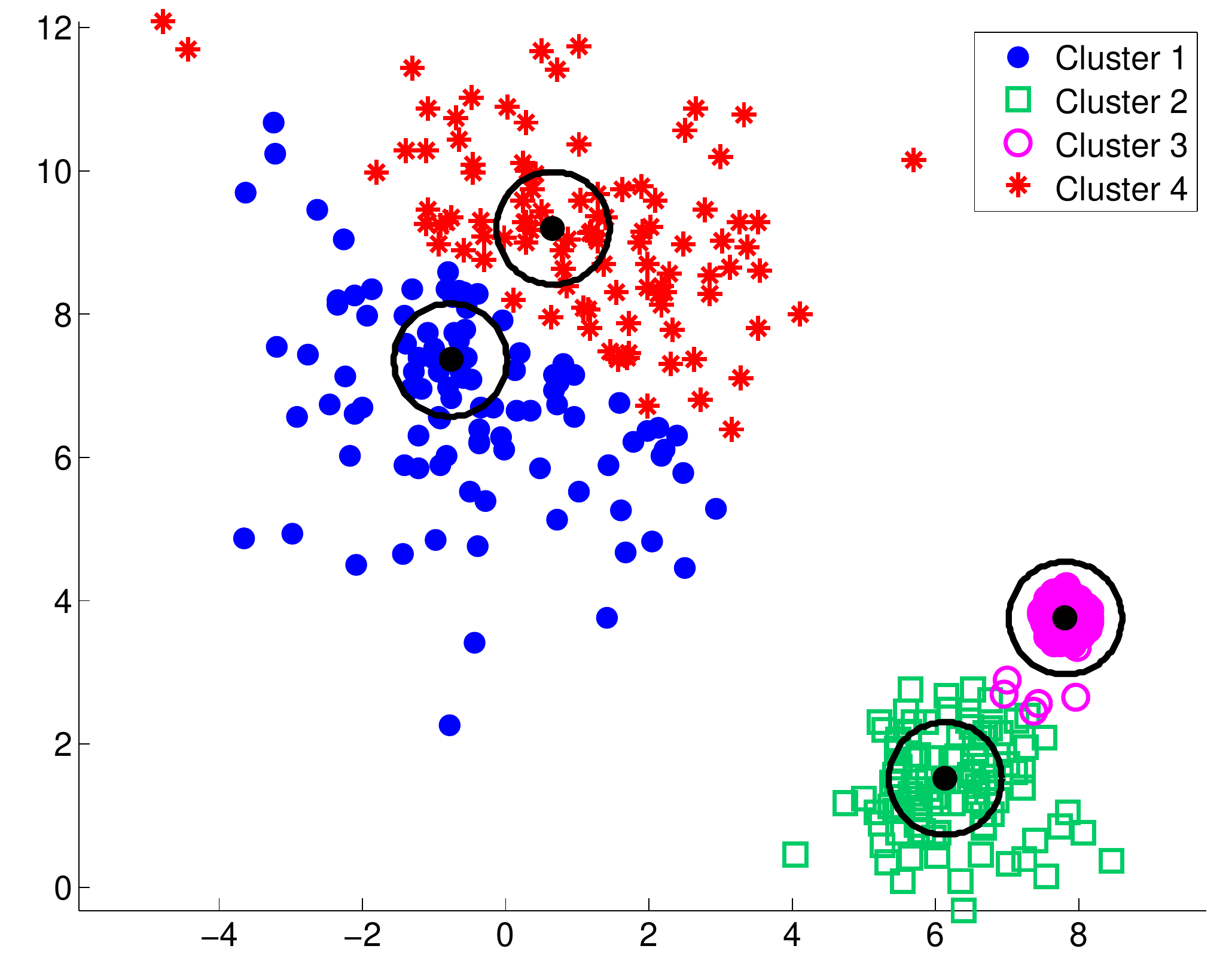}\label{exp2upc}}
\hfil
\centering
\subfloat[UPFC]{\includegraphics[width=0.33\textwidth]{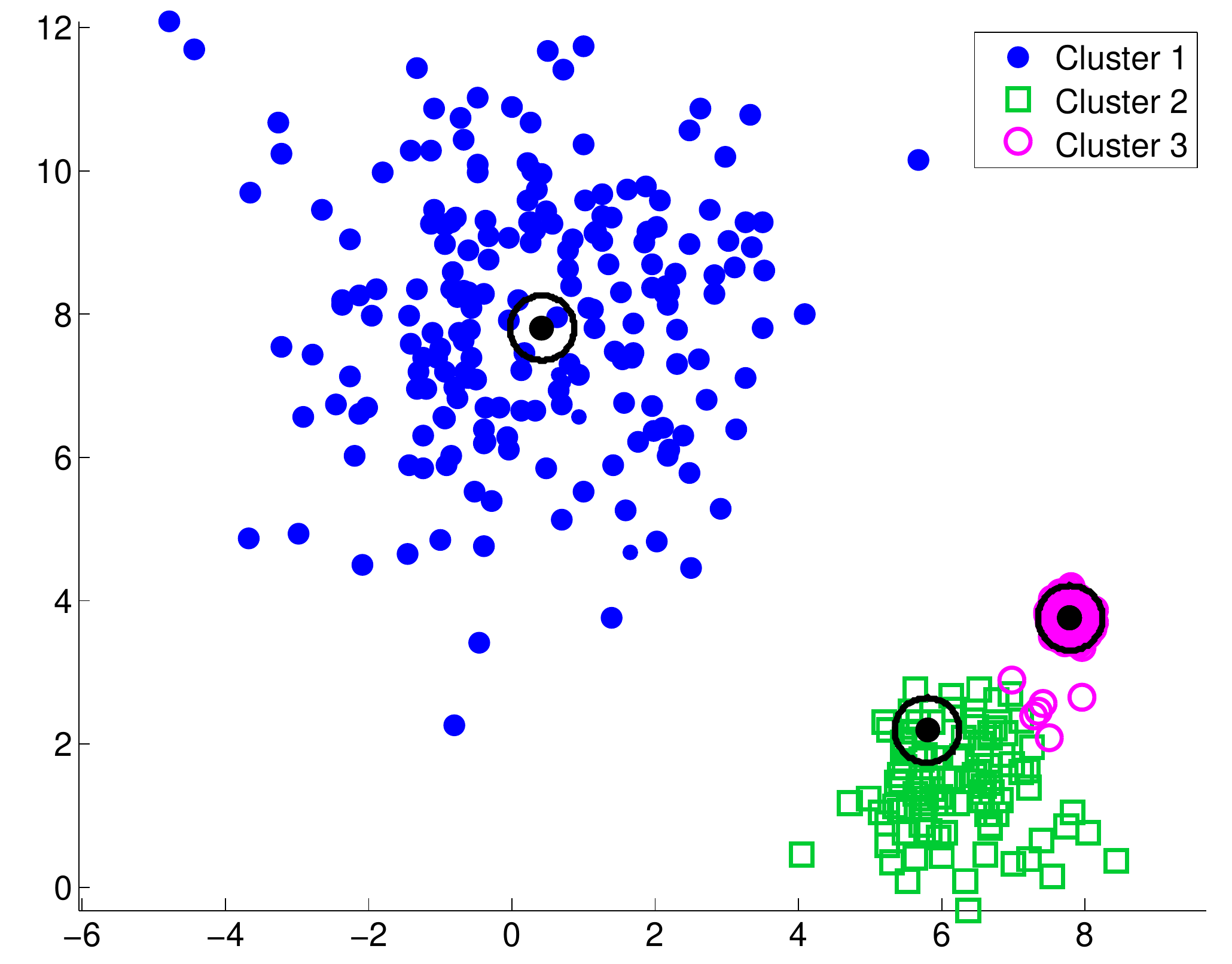}\label{exp2upfc}}
\hfil
\centering
\subfloat[PFCM]{\includegraphics[width=0.33\textwidth]{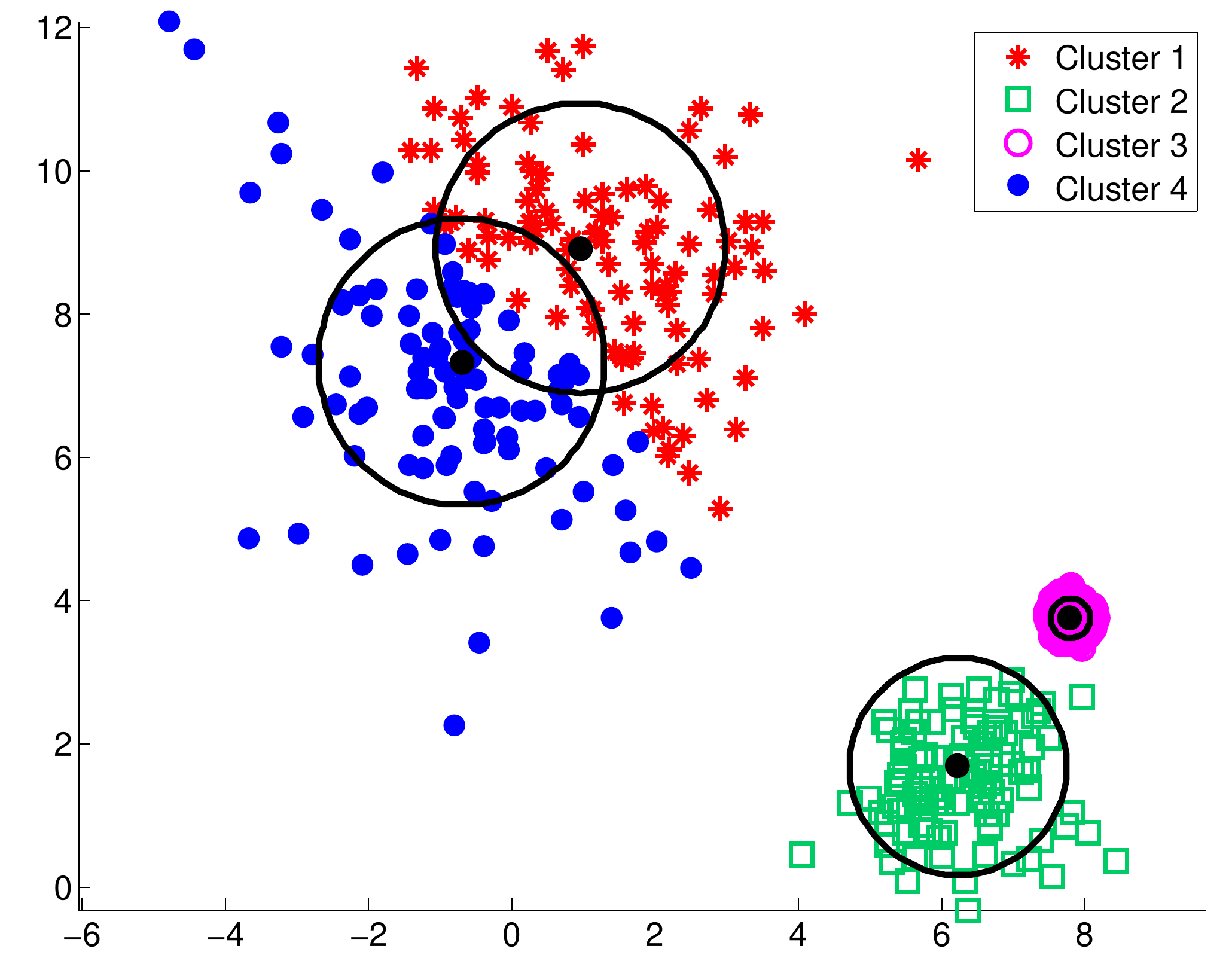}\label{exp2pfcm}}
\hfil
\centering
\subfloat[SPCM-L$_1$]{\includegraphics[width=0.33\textwidth]{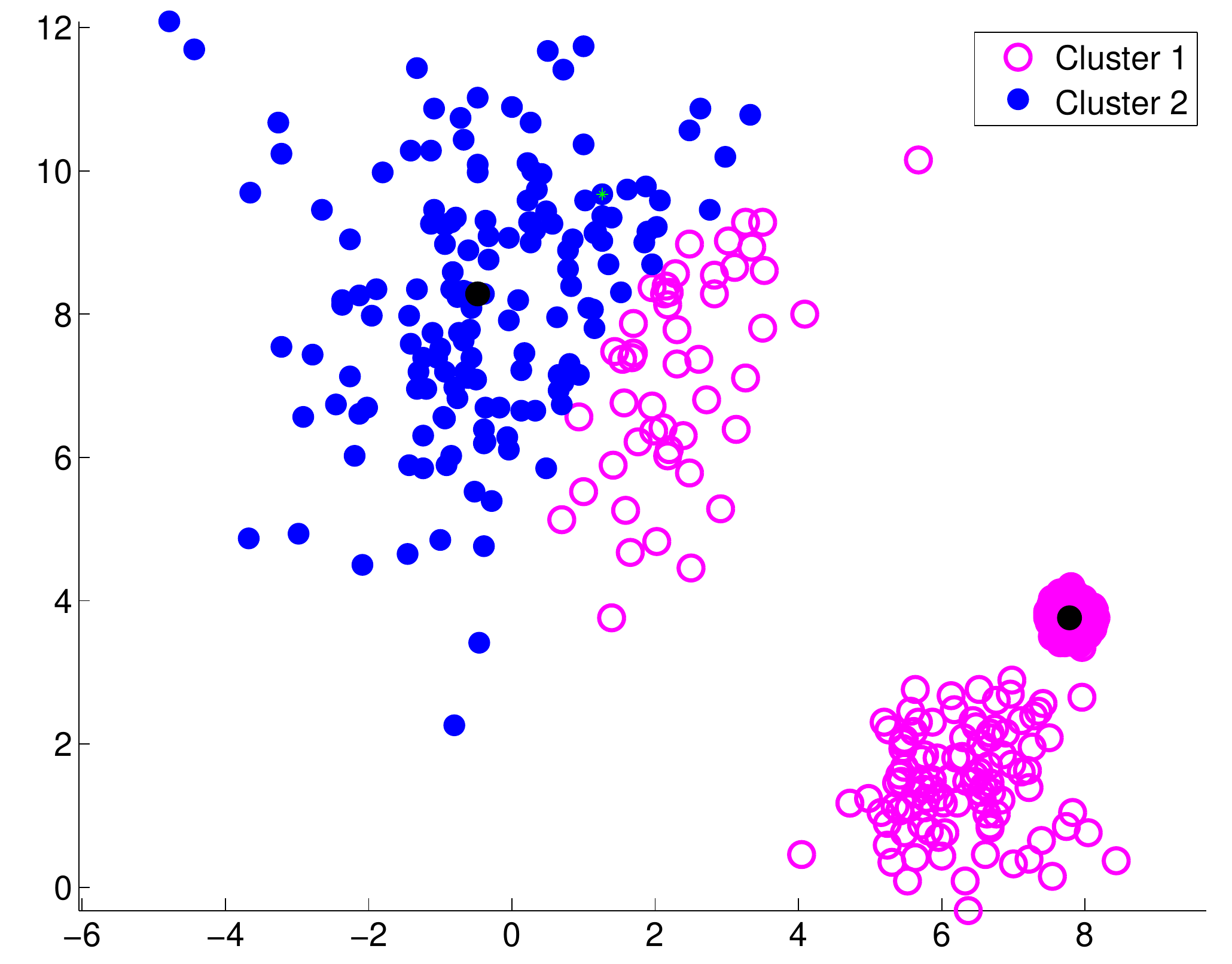}\label{exp2spcml1}}
\hfil
\centering
\subfloat[APCM]{\includegraphics[width=0.33\textwidth]{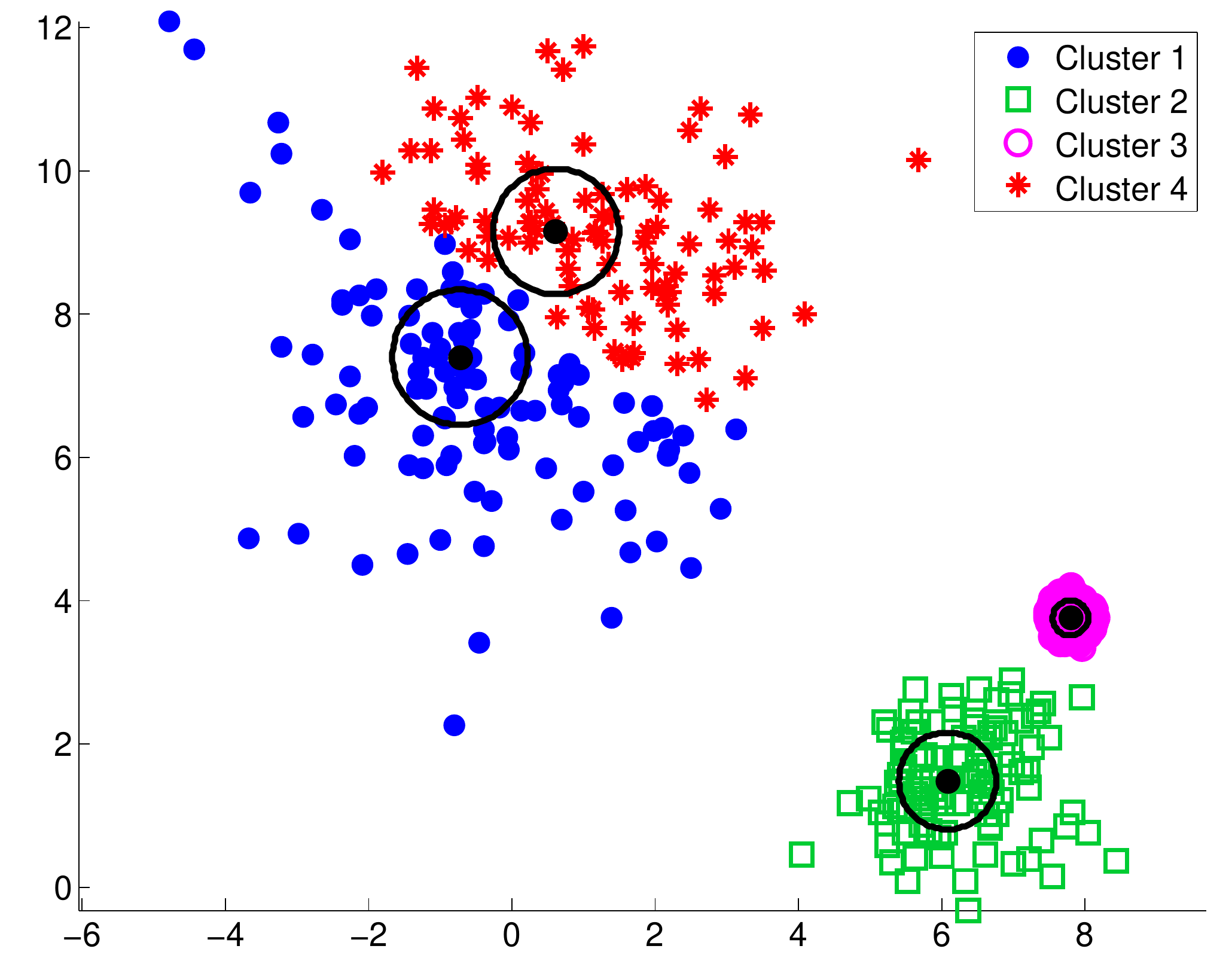}\label{exp2apcm}}
\hfil
\centering
\subfloat[SPCM]{\includegraphics[width=0.33\textwidth]{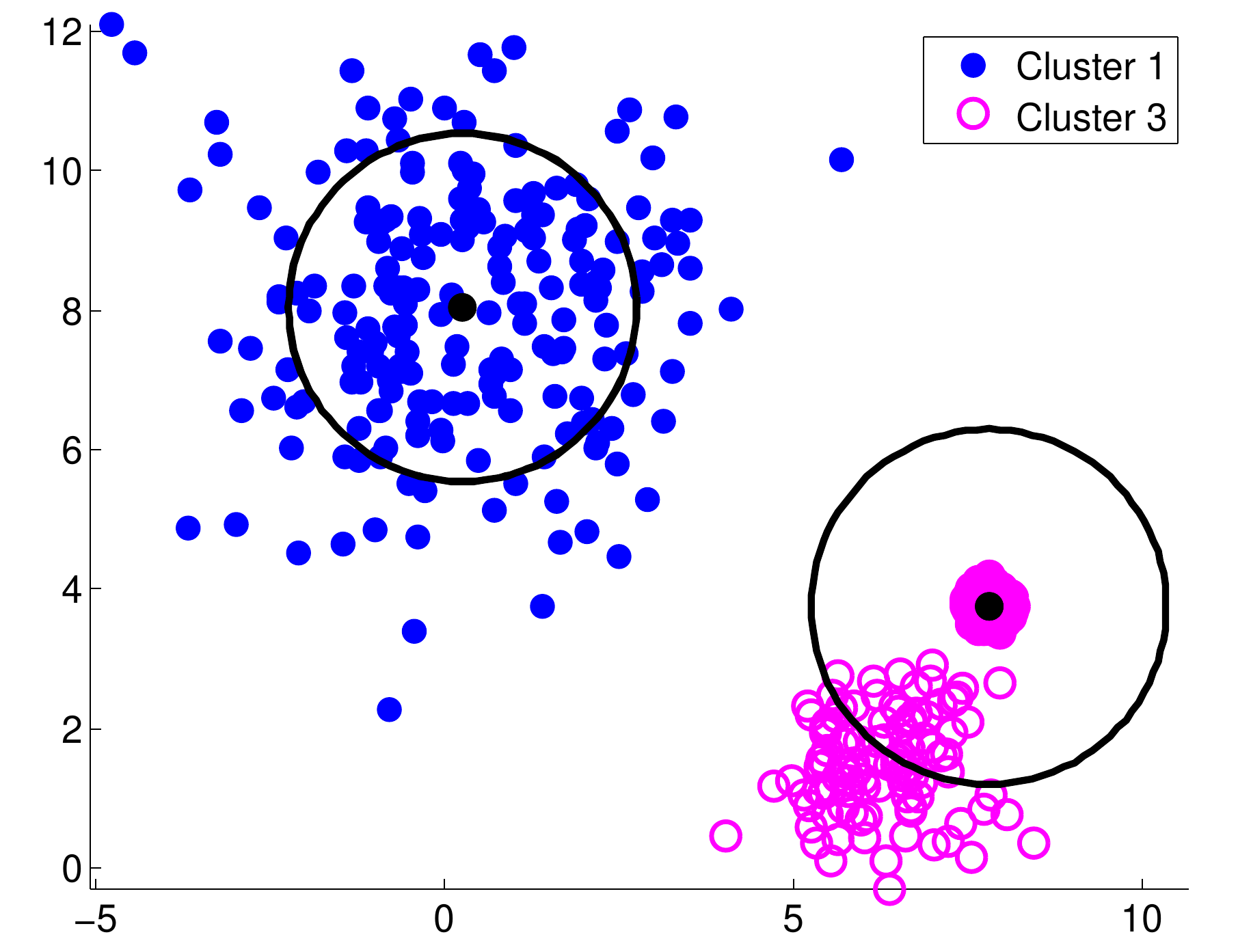}\label{exp2spcm}}
\hfil
\centering
\subfloat[SAPCM]{\includegraphics[width=0.33\textwidth]{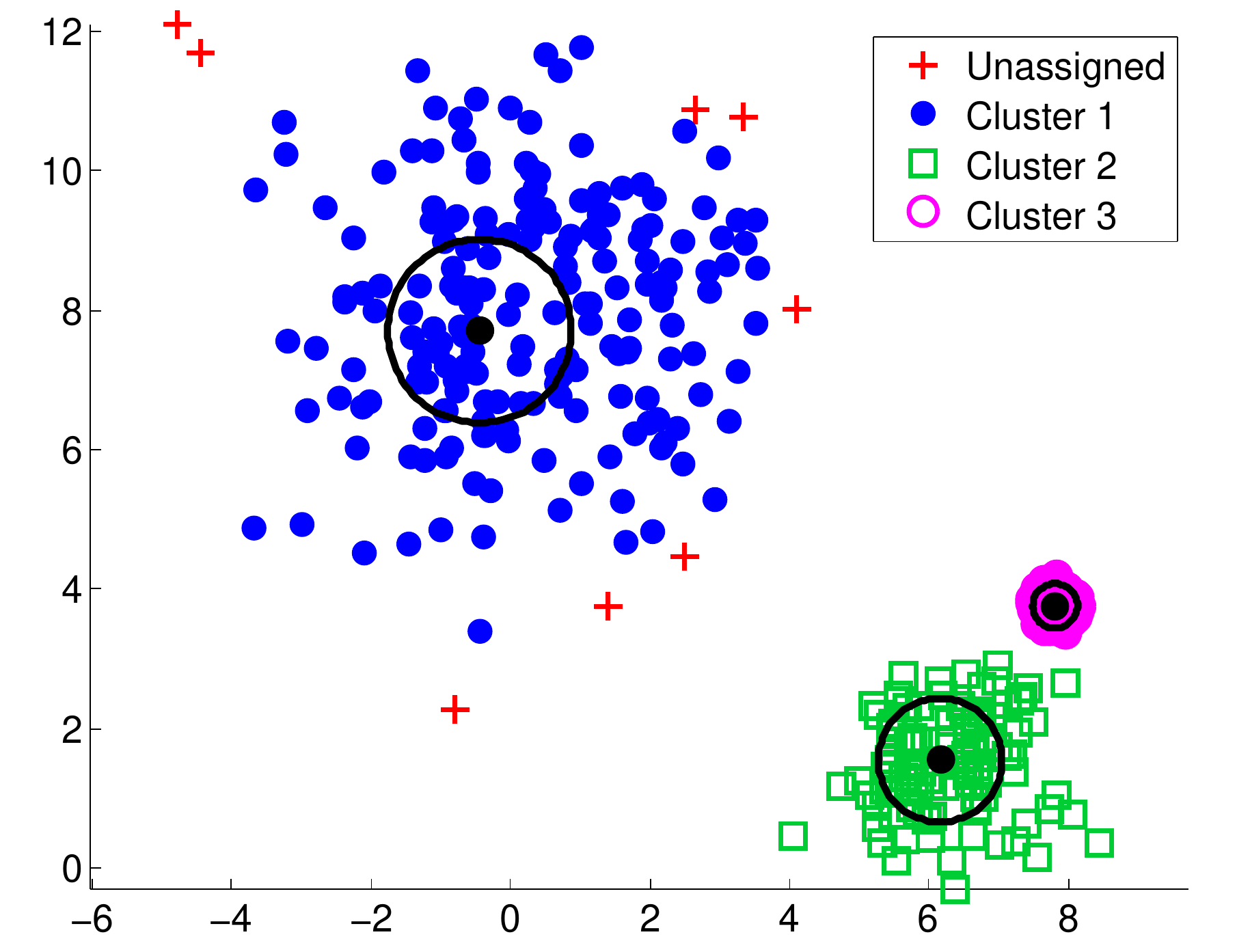}\label{exp2sapcm}}
\hfil
\centering{\caption{(a) The data set of Experiment 2. Clustering results for (b) k-means, $m_{ini}=3$, (c) FCM, $m_{ini}=3$, (d) PCM, $m_{ini}=5$, (e) UPC, $m_{ini}=5$, $q=1.5$, (f) UPFC, $m_{ini}=10$, $\alpha=5$, $\beta=1$, $q=2.2$, $n=3$, (g) PFCM, $m_{ini}=5$, $K=1$, $\alpha=1$, $\beta=5$, $q=1.5$, $n=1.5$,}\label{exp2}}
\end{figure}
\begin{figure}[htpb!]
  \contcaption{(h) SPCM-$L_1$, $\lambda=15$, $q=2$ (i) APCM, $m_{ini}=5$, $\alpha=0.3$, (j) SPCM, $m_{ini}=5$, and (k) SAPCM, $m_{ini}=10$ and $\alpha=0.15$.}
\end{figure}

{\bf Experiment 2}:
Consider a two-dimensional data set consisting of $N=5300$ points, where three clusters $C_1$, $C_2$ and $C_3$ are formed. Each cluster is modelled by a normal distribution. The means of the distributions are $\mathbf{c}_1=[0.27, 7.99]^T$, $\mathbf{c}_2=[6.28, 1.49]^T$ and $\mathbf{c}_3=[7.81, 3.76]^T$, respectively, while their covariance matrices are set to $3\cdot I_2$, $0.5\cdot I_2$ and $0.01\cdot I_2$, respectively. A number of 200 points are generated by the first distribution, 100 points are generated by the second one and 5000 points are generated by the third one. Note that $C_2$ and $C_3$ clusters are very close to each other and they have a big difference in their variances (see Fig.~\ref{exp2dataset}). Also, note the difference in the density among the three clusters.

Table~\ref{table:synth2} shows the clustering results of all algorithms for Experiment 2. Fig.~\ref{exp2kmeans} and Fig.~\ref{exp2fcm} show the clustering result obtained using the k-means and FCM algorithms, respectively, both for $m_{ini}=3$. Figs.~\ref{exp2pcm}, ~\ref{exp2upc},~\ref{exp2upfc}, ~\ref{exp2pfcm}, ~\ref{exp2spcml1}, and~\ref{exp2apcm}, depict the performance of PCM, UPC, UPFC, PFCM, SPCM-$L_1$ and APCM, respectively, with their parameters chosen (after fine-tuning) as stated in the figure caption. In addition, the circles, centered at each $\boldsymbol{\theta}_j$ and having radius $\sqrt{\gamma_j}$ (as they have been computed after the convergence of the algorithms), are also drawn.

\begin{table}[htpb!]
\centering
\caption{Performance of clustering algorithms for the Experiment 2 data set.}
{\small
\begin{tabular}{>{\arraybackslash}m{0.43\linewidth} | >{\centering\arraybackslash}m{0.035\linewidth} |>{\centering\arraybackslash}m{0.055\linewidth}| >{\centering\arraybackslash}m{0.045\linewidth} |>{\centering\arraybackslash}m{0.045\linewidth} |>{\centering\arraybackslash}m{0.045\linewidth} |>{\centering\arraybackslash}m{0.055\linewidth} |>{\centering\arraybackslash}m{0.03\linewidth} |>{\centering\arraybackslash}m{0.04\linewidth}}
\hline  
	& \centering $m_{ini}$ & \centering $m_{final}$ & \centering SR$_{c_1}$ & \centering SR$_{c_2}$ & \centering SR$_{c_3}$ & \centering MD & \centering Iter & {\centering Time}\\
\hline
k-means & 3  & 3  & 51 & 0 & 100 & 3.4066 & 2 & 0.265 \\
k-means & 5  & 5  & 51 & 94 & 51.48 & 0.5369 & 20 & 0.202 \\
\hline
FCM & 3  & 3  & 51 & 0 & 100 & 3.3432 & 110 & 0.140 \\
FCM & 5  & 5  & 50.50 & 93 & 51.62 & 0.5537 & 86 & 0.218 \\
\hline
PCM & 5  & 2  & 100 & 0 & 100 & 0.9242 & 15 & 0.514 \\
PCM & 10 & 2  & 100 & 0 & 100 & 0.9254 & 18 & 1.185 \\
\hline
UPC ($q=1.5$) & 5  & 4  & 50 & 95 & 100 & 0.4589 & 65 & 0.390 \\
UPC ($q=1.2$) & 10  & 4  & 50 & 95 & 100 & 0.4480 & 89 & 0.910 \\
\hline
UPFC ($a=5$, $b=1$, $q=2$, $n=1.5$) & 5  & 4  & 50.50 & 96 & 100 & 0.4170 & 41 & 0.390 \\
UPFC ($a=5$, $b=1$, $q=2.2$, $n=3$) & 10  & 3  & 100 & 94 & 100 & 0.3601 & 190 &  2.940\\
\hline
PFCM ($K=1$, $a=1$, $b=5$, $q=1.5$, $n=1.5$) & 5  & 4  & 51.50 & 100 & 100 & 0.4573 & 38 & 0.380 \\
PFCM ($K=1$, $a=2$, $b=1$,  $q=2$, $n=1.2$) & 10  & 5  & 44 & 97 & 100 & 0.4011 & 60 & 0.880 \\
\hline
SPCM-L$_1$ ($\lambda=15$, $q=2$) & -  & 2  & 76 & 0 & 100 & 1.1831 & 6 & 0.031 \\
\hline
APCM ($\alpha=0.3$) & 5  & 4  & 53 & 100 & 100 & 0.4469 & 73 & 0.421 \\
APCM ($\alpha=0.3$) & 10  & 4  & 52.50 & 100 & 100 & 0.4748 & 87 & 0.890 \\
\hline
SPCM & 5  & 2  & 100 & 0 & 100 & 0.9256 & 15 & 3.338 \\
SPCM & 10  & 2  & 100 & 0 & 100 & 0.9263 & 19 & 8.034 \\
\hline
SAPCM ($\alpha=0.18$)  & 5  & 3  & 100 & 100 & 100 & 0.3222 & 91 & 14.40 \\
SAPCM ($\alpha=0.15$)  & 10  & 3  & 100 & 100 & 100 & 0.3020 & 102 & 20.28 \\
\hline
\end{tabular}}
\label{table:synth2}
\end{table}

As it can be deduced from Fig.~\ref{exp2} and Table.~\ref{table:synth2}, even when the k-means and the FCM are initialized with the (unknown in practice) true number of clusters ($m=3$), they fail to unravel the underlying clustering structure mainly due to the big difference in the variances and densities between clusters. The classical PCM also fails to detect the physical cluster 2, because of its position that is next to the densest physical cluster. The UPC algorithm has been fine tuned so that the parameters $\gamma_j$'s, which remain fixed during its execution and are the same for all clusters, get small enough values, in order to identify the cluster $C_2$. However, it splits the high variance/low density cluster $C_1$ in two clusters. The same seems to hold for the PFCM algorithm, after fine tuning of its several parameters. The UPFC algorithm produces 3 clusters, at the cost of a computationally demanding fine tuning of the (several) parameters it involves. However, the final estimates of $\boldsymbol{\theta}_j$'s are not closely located to the true cluster centers (see the MD measure in Table~\ref{table:synth2}). The APCM algorithm also splits the big variance cluster in two subclusters, failing to detect the underlying clustering structure. On the other hand, SPCM identifies two clusters with high accuracy of the center of the actual clusters, but misses the third one. Finally, as it is deduced from Table~\ref{table:synth2}, the SAPCM algorithm manages to identify all clusters, achieving the best RM and SR results and detecting very accurately the true centers of the clusters, since it exhibits the minimum MD among all algorithms. 

{\bf Experiment 3}:
\begin{figure}[htpb!]
\centering
\subfloat[The data set]{\includegraphics[width=0.32\textwidth]{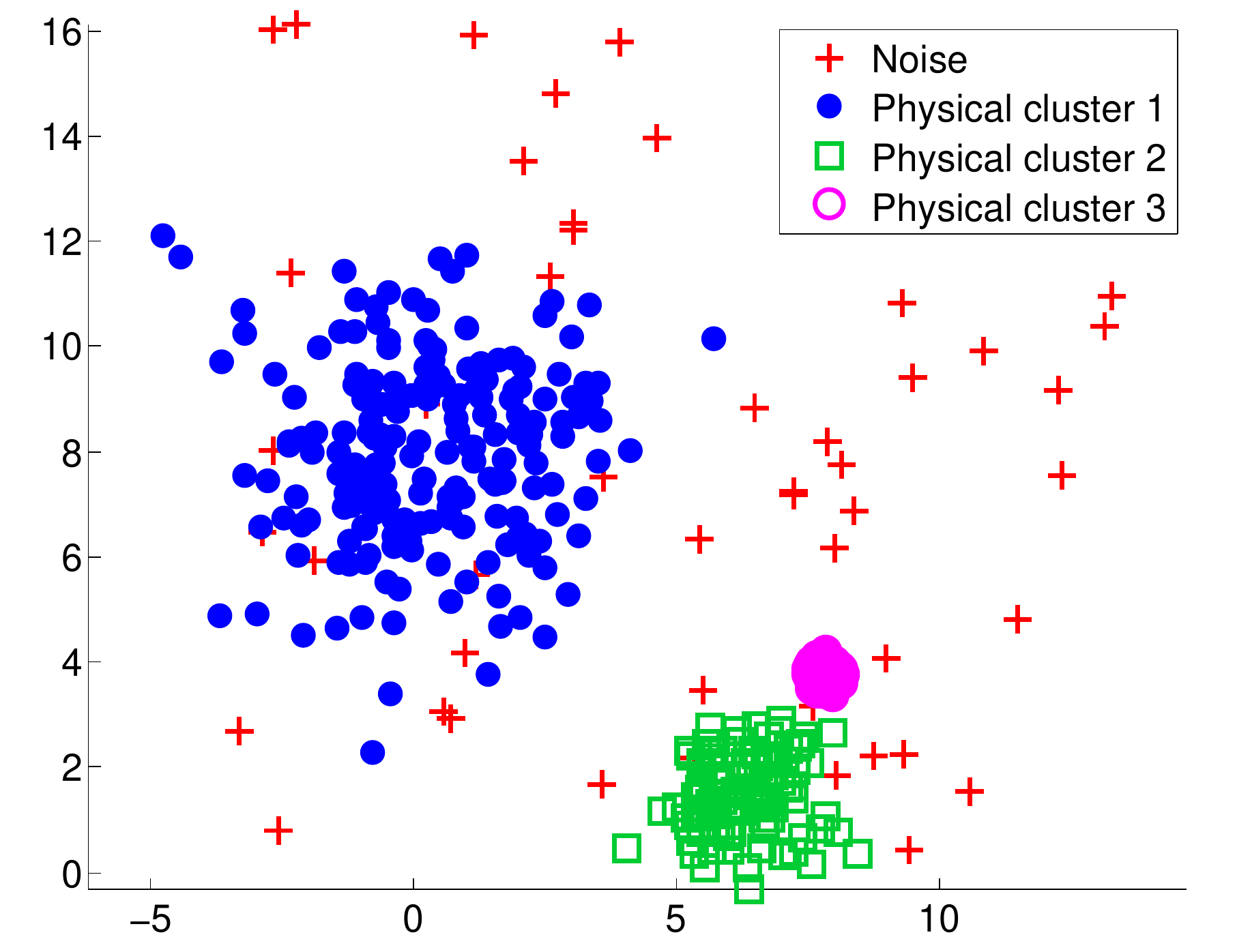}\label{exp3dataset}}
\hfil
\centering
\subfloat[k-means]{\includegraphics[width=0.32\textwidth]{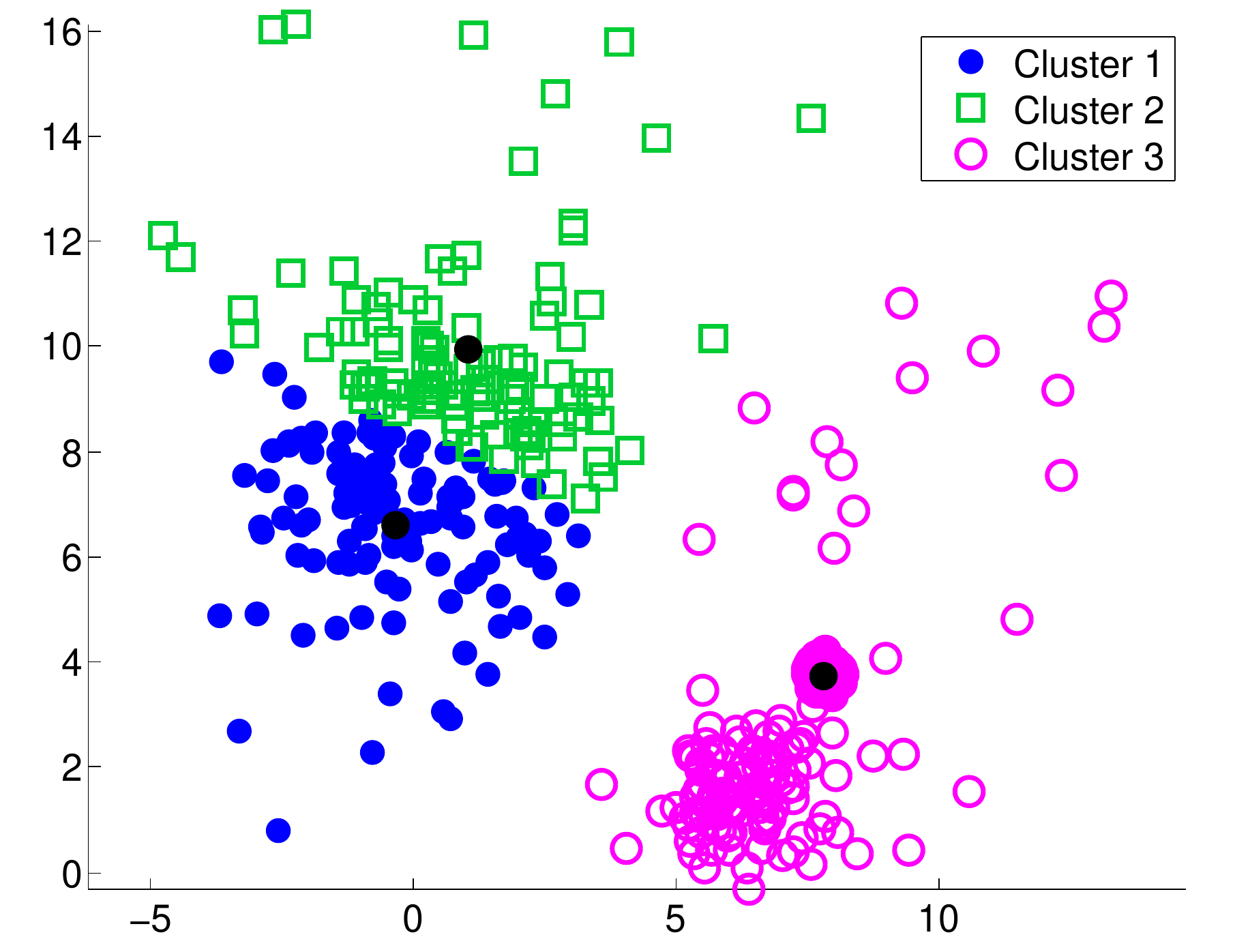}\label{exp3kmeans}}
\hfil
\centering
\subfloat[FCM]{\includegraphics[width=0.32\textwidth]{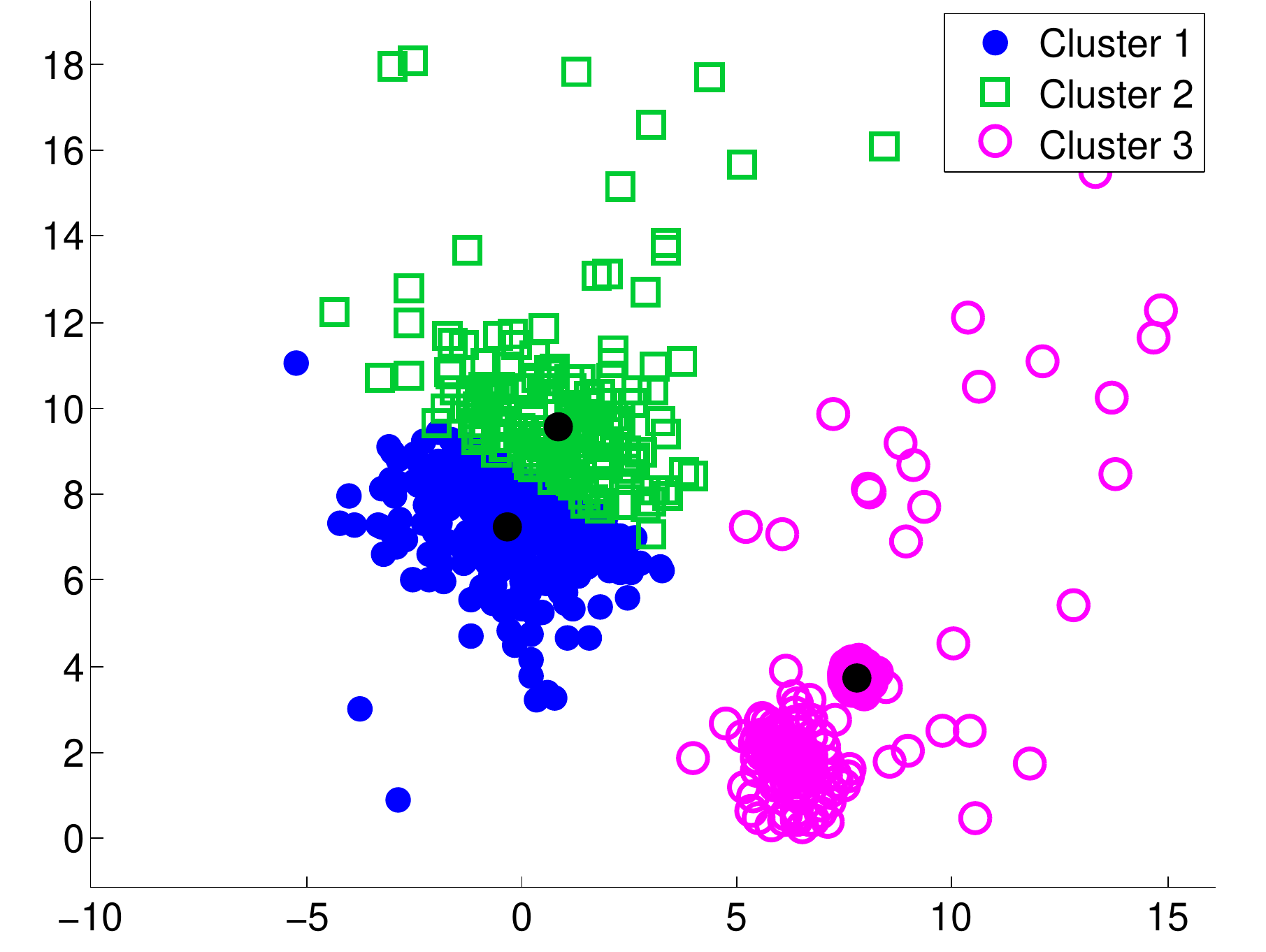}\label{exp3fcm}}
\hfil
\centering
\subfloat[PCM]{\includegraphics[width=0.33\textwidth]{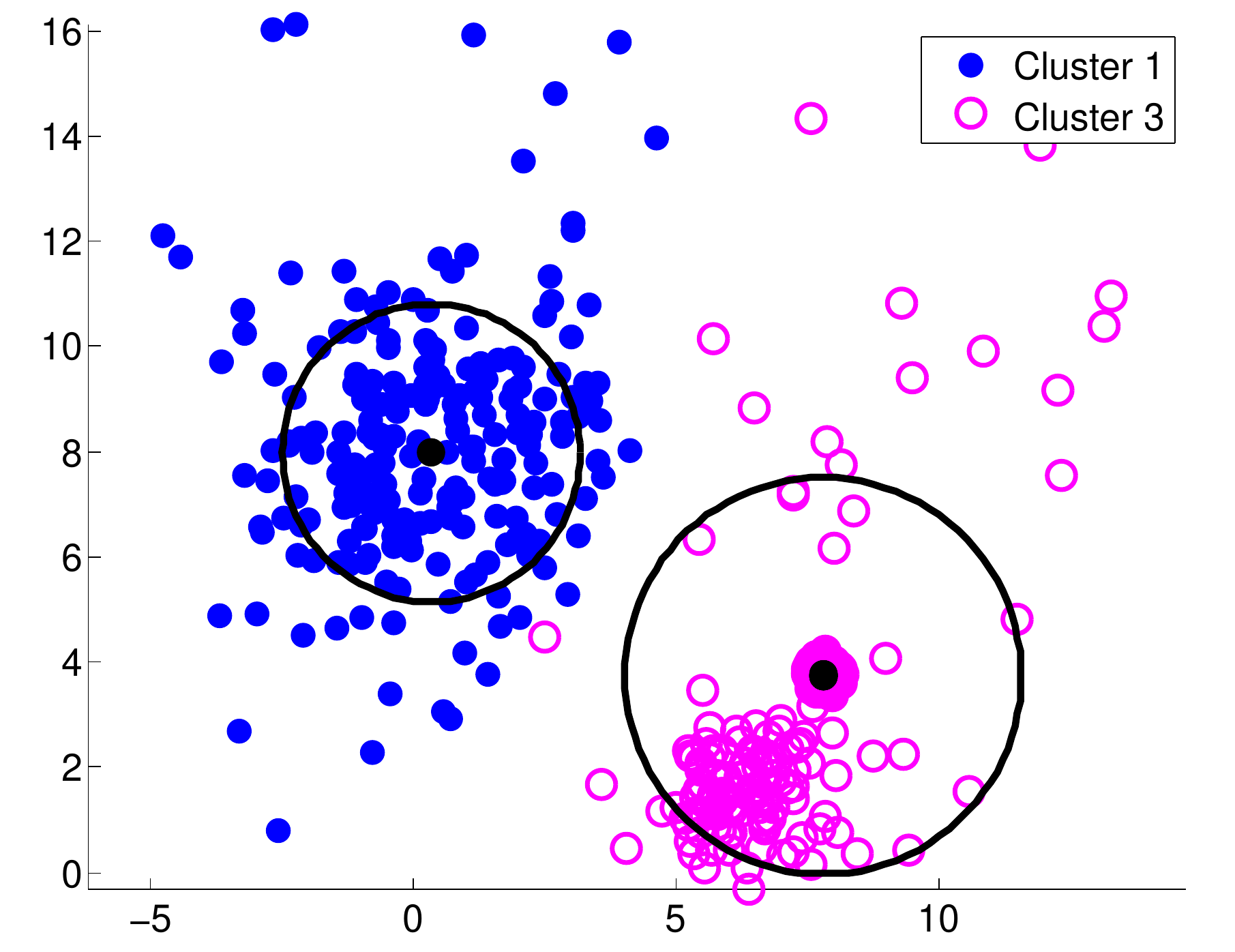}\label{exp3pcm}}
\hfil
\centering
\subfloat[UPC]{\includegraphics[width=0.33\textwidth]{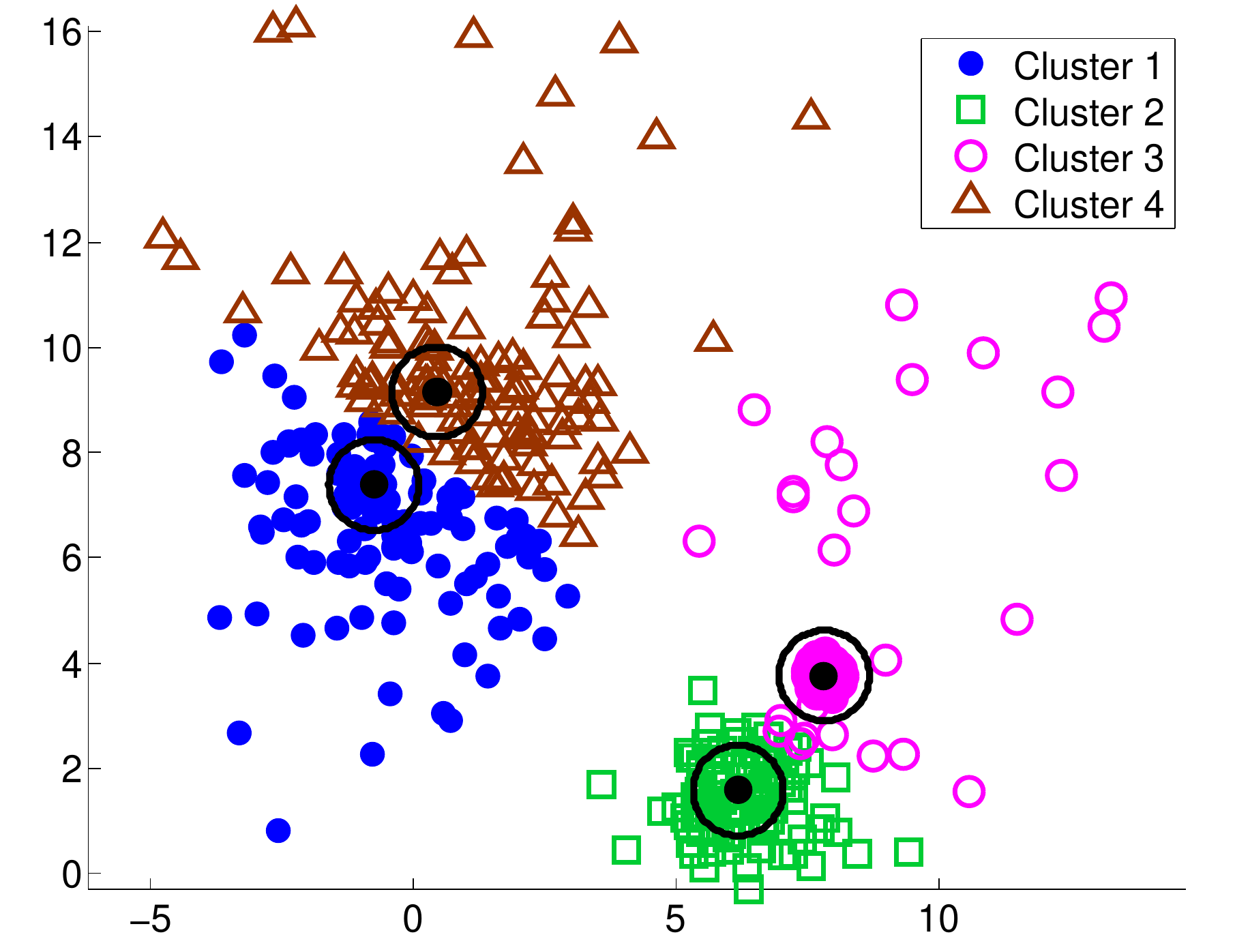}\label{exp3upc}}
\hfil
\centering
\subfloat[UPFC]{\includegraphics[width=0.33\textwidth]{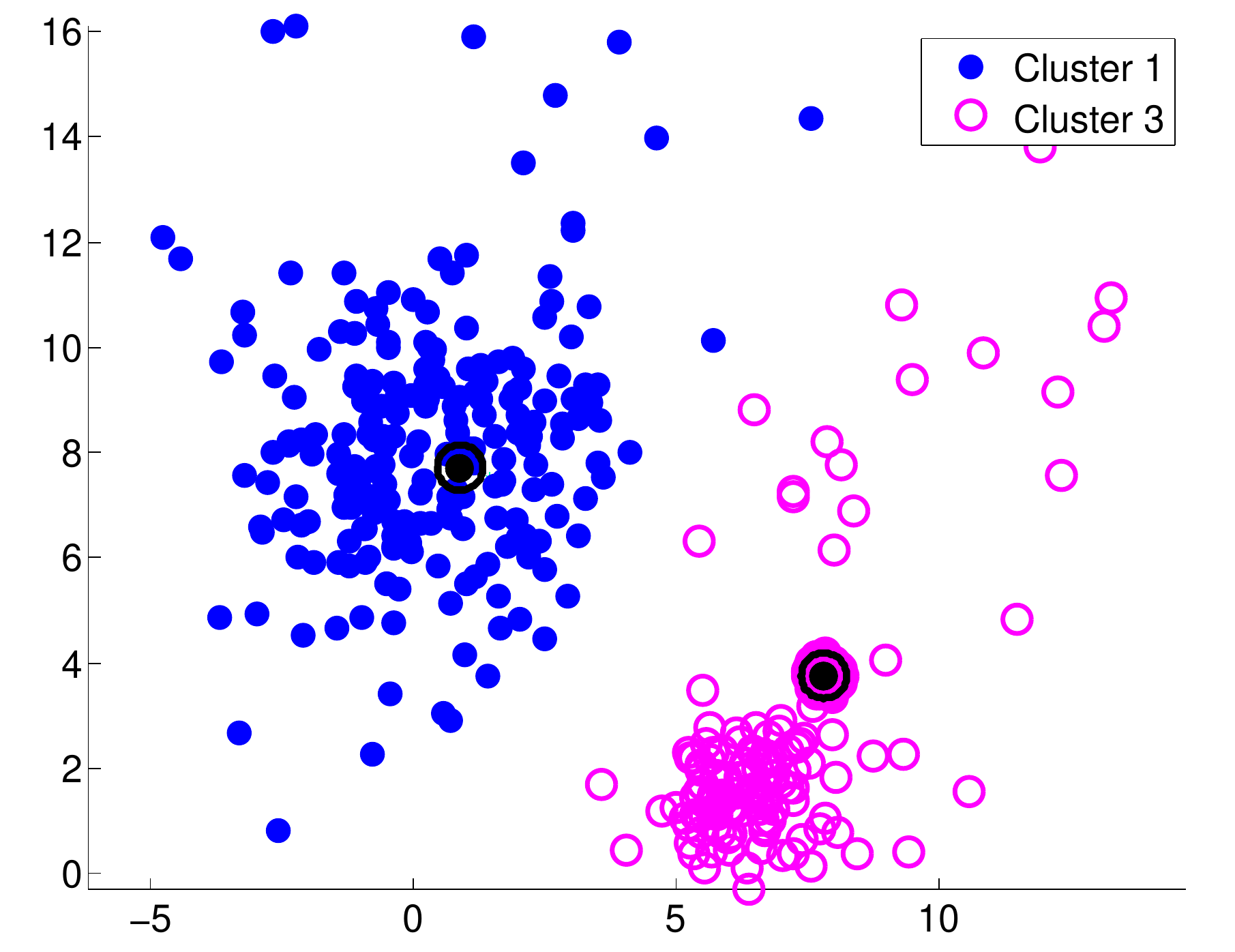}\label{exp3upfc}}
\hfil
\centering
\subfloat[PFCM]{\includegraphics[width=0.33\textwidth]{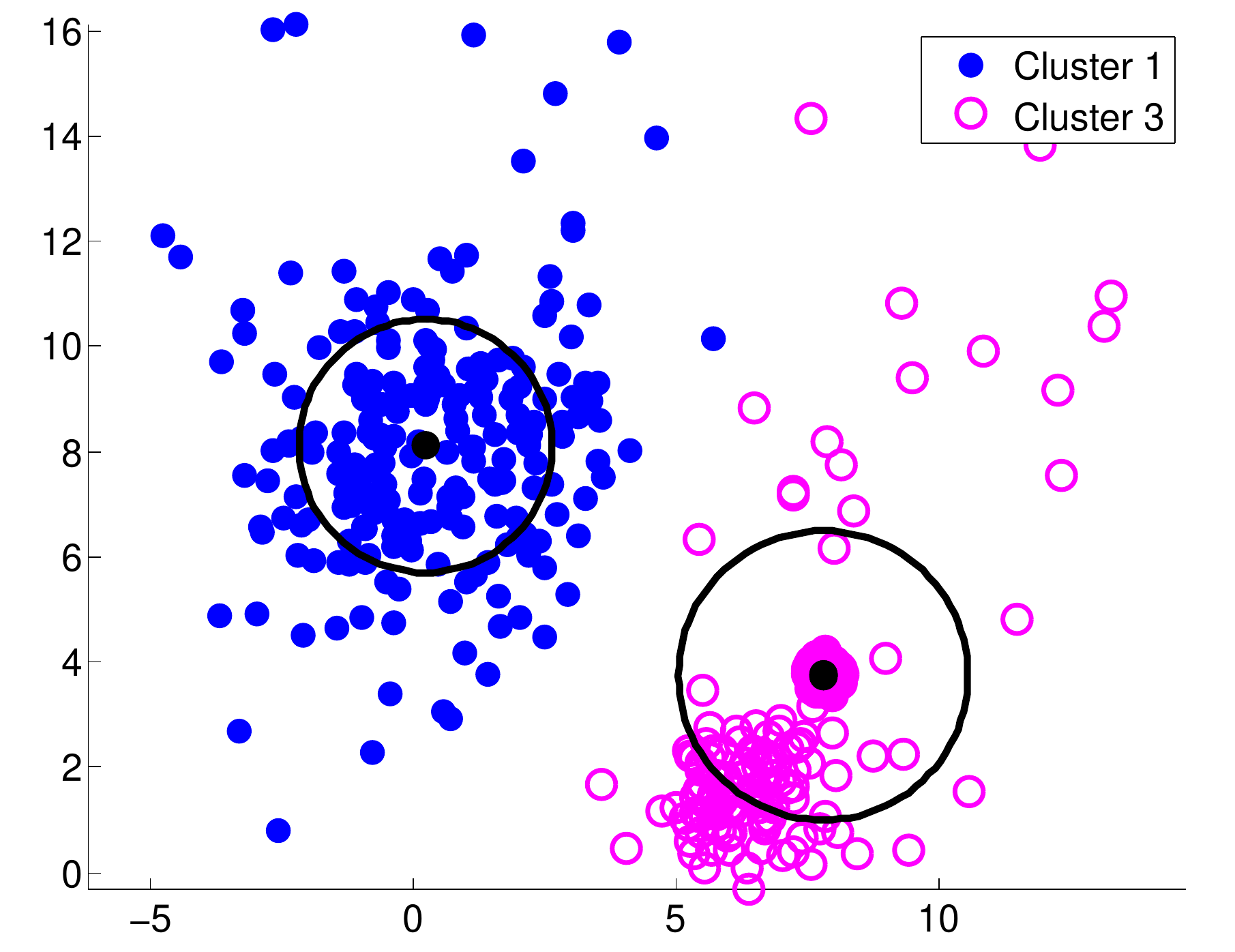}\label{exp3pfcm}}
\hfil
\centering
\subfloat[SPCM-L$_1$]{\includegraphics[width=0.33\textwidth]{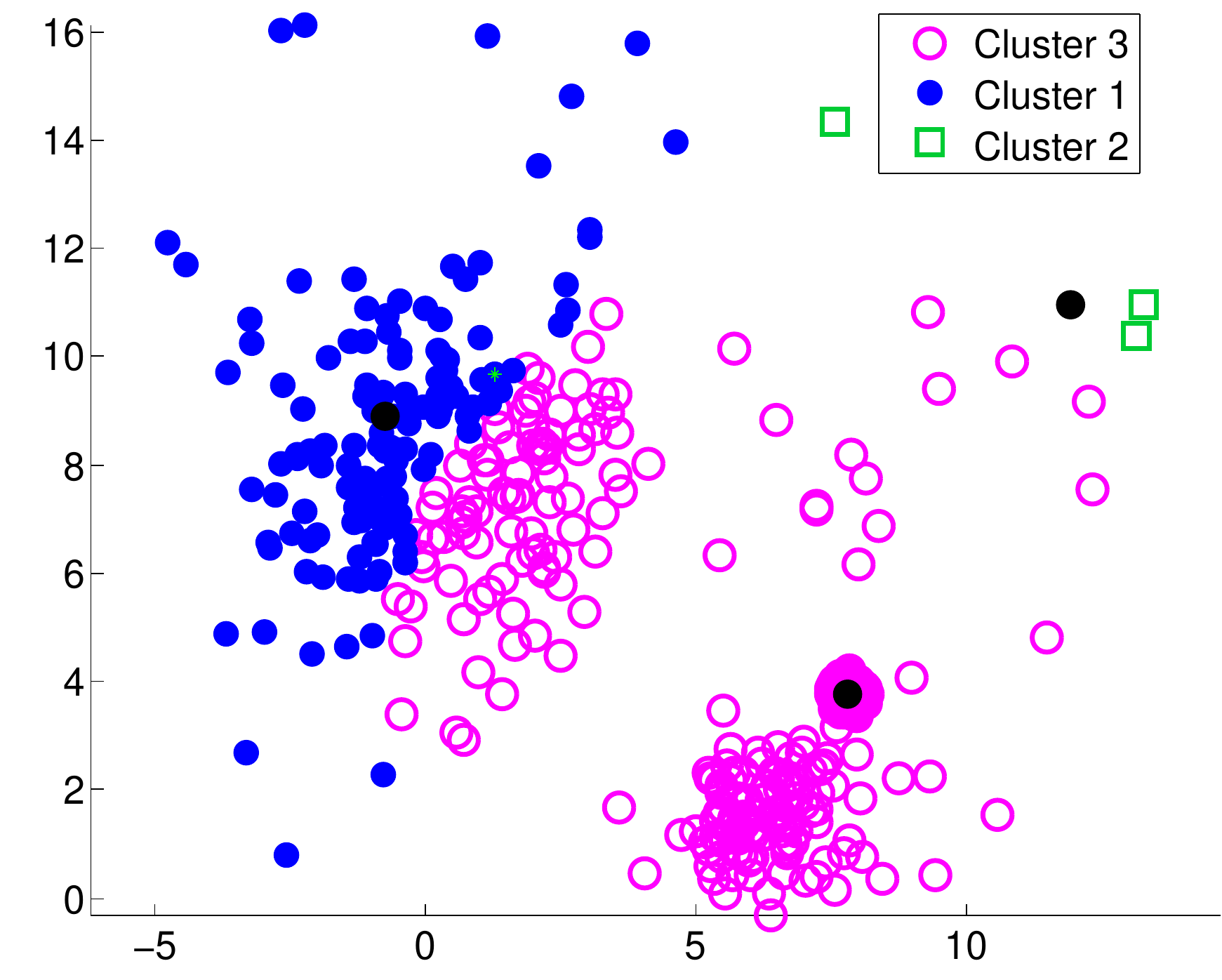}\label{exp3spcml1}}
\hfil
\centering
\subfloat[APCM]{\includegraphics[width=0.33\textwidth]{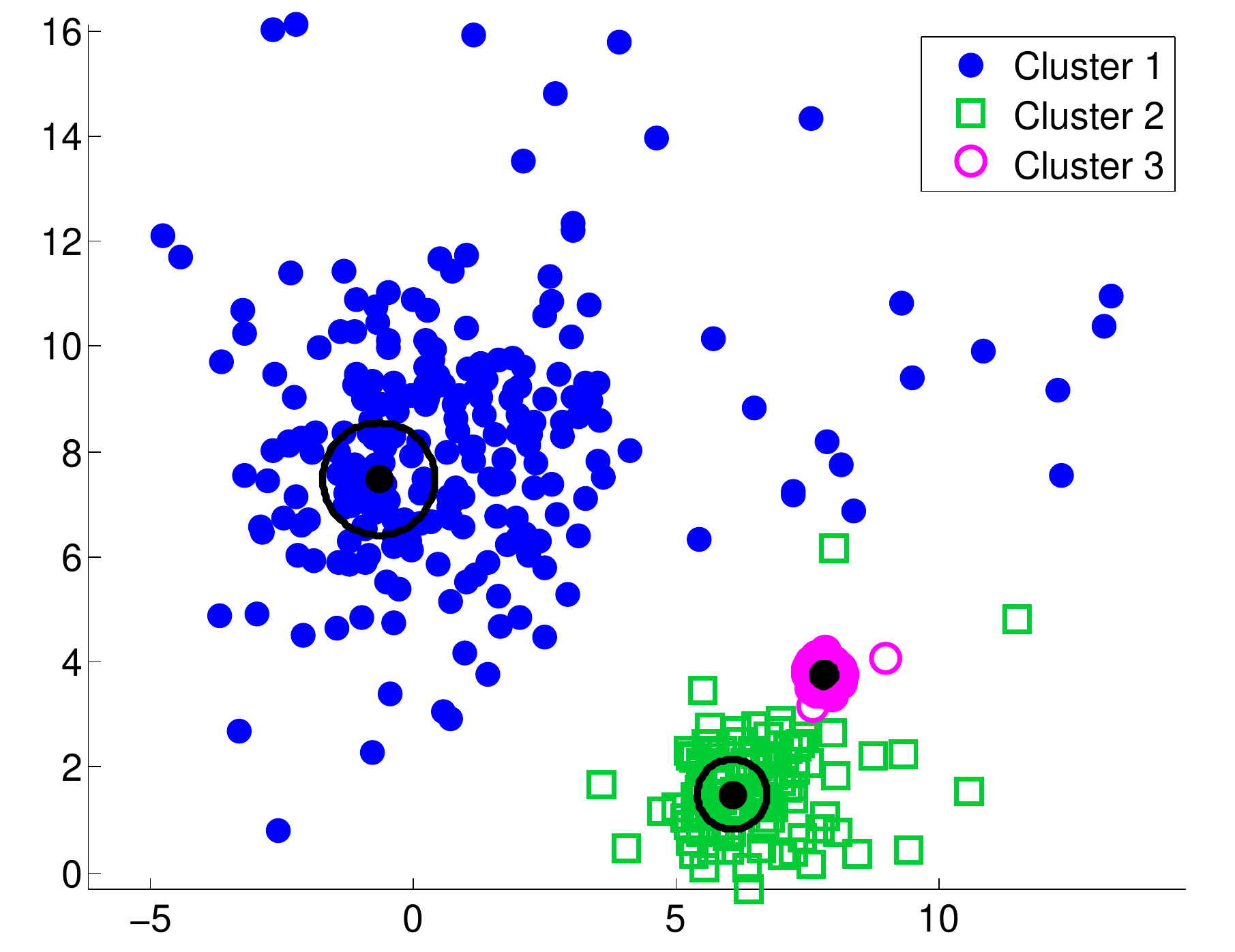}\label{exp3apcm}}
\hfil
\centering
\subfloat[SPCM]{\includegraphics[width=0.33\textwidth]{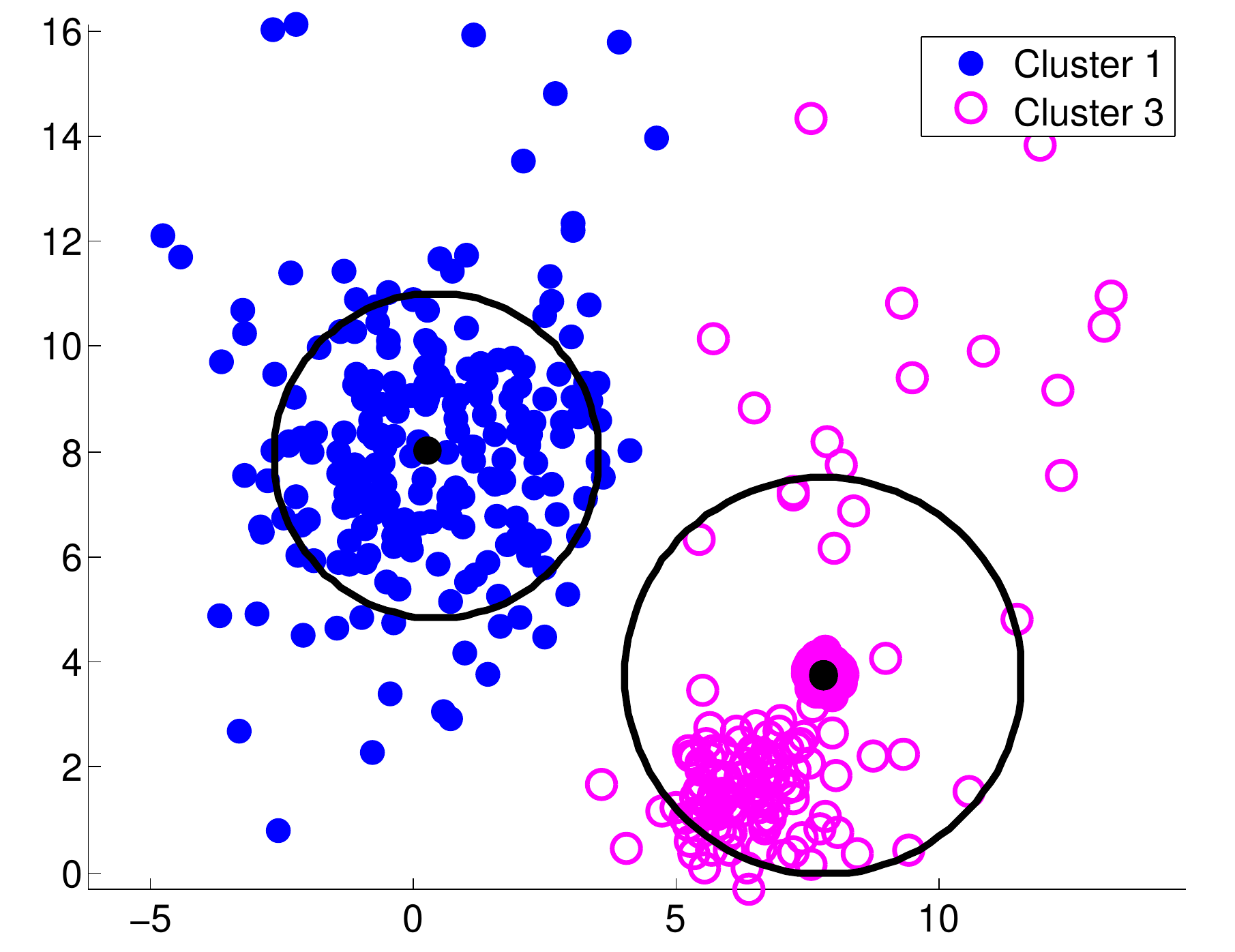}\label{exp3spcm}}
\hfil
\centering
\subfloat[SAPCM]{\includegraphics[width=0.33\textwidth]{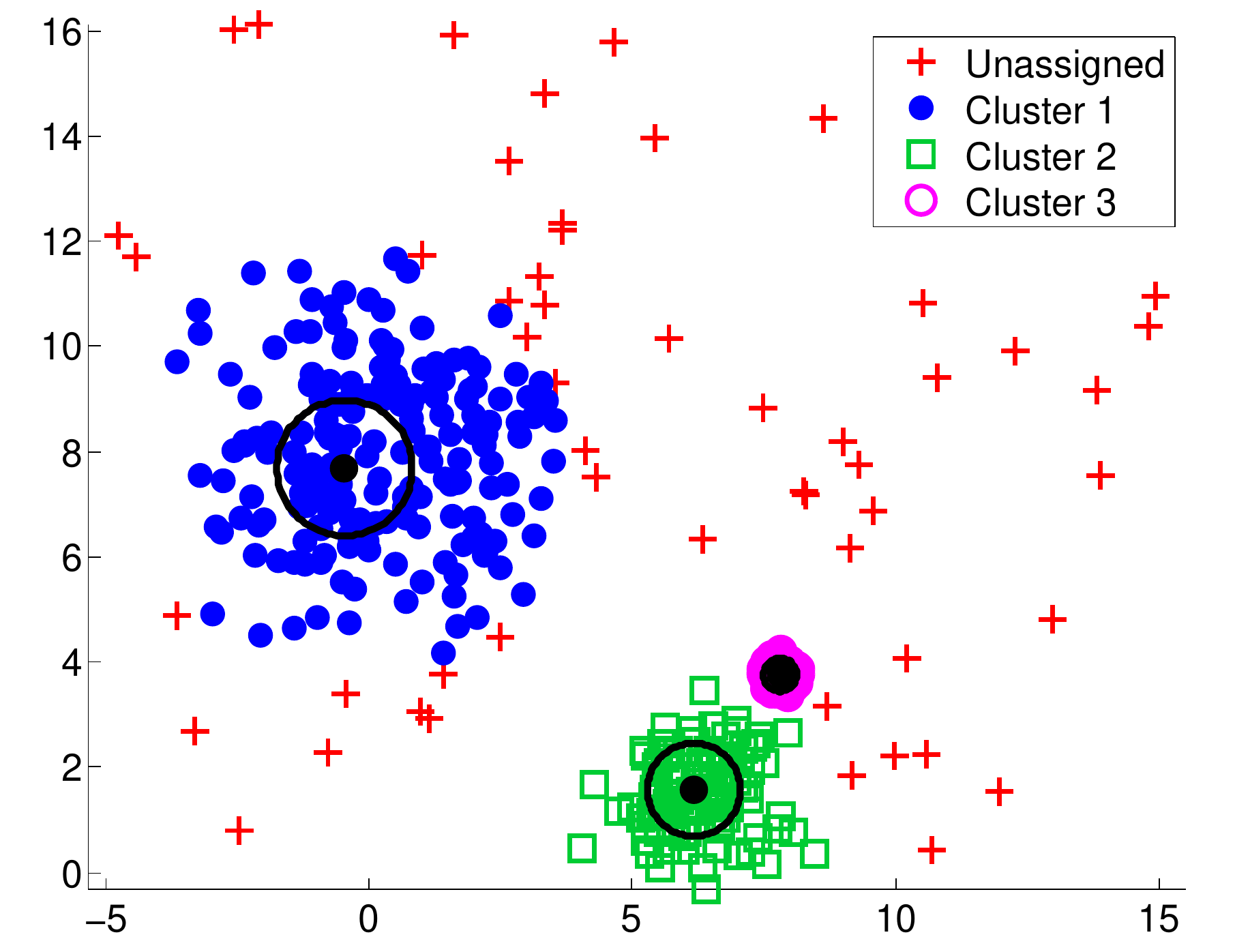}\label{exp3sapcm}}
\hfil
\centering{\caption{(a) The data set of Experiment 3. Clustering results for (b) k-means, $m_{ini}=3$, (c) FCM, $m_{ini}=3$, (d) PCM, $m_{ini}=10$, (e) UPC, $m_{ini}=5$, $q=1.5$, (f) UPFC, $m_{ini}=10$, $\alpha=5$, $\beta=1$, $q=2.5$, $n=2$, (g) PFCM, $m_{ini}=5$, $K=1$, $\alpha=1$, $\beta=1$, $q=1.5$, $n=1.5$,}\label{exp3}}
\end{figure}
\begin{figure}[htpb!]
  \contcaption{(h) SPCM-$L_1$, $\lambda=17$, $q=2$ (i) APCM, $m_{ini}=5$, $\alpha=0.4$, (j) SPCM, $m_{ini}=10$, and (k) SAPCM, $m_{ini}=10$ and $\alpha=0.18$.}
\end{figure}
Consider now the same dataset as in Experiment 2, where 50 data points are now added randomly as noise in the region where data live (see Fig.~\ref{exp3dataset}). It can be seen that APCM and SAPCM algorithms are the only algorithms that distinguish all clusters. However, SAPCM detects with higher accuracy their actual centers, compared to APCM. In addition, it keeps MD at low values, whereas all other algorithms conclude to higher MD values than in Experiment 2 (see Table~\ref{table:synth3}). Finally, as it can be seen in Fig.~\ref{exp3}, SAPCM is the only algorithm that identifies the noisy points of the data set and ignores them in the updating of the location of the cluster representatives.

\begin{table}[htpb!]
\centering
\caption{Performance of clustering algorithms for the Experiment 3 data set.}
{\small
\begin{tabular}{>{\arraybackslash}m{0.43\linewidth} | >{\centering\arraybackslash}m{0.035\linewidth} |>{\centering\arraybackslash}m{0.055\linewidth}| >{\centering\arraybackslash}m{0.045\linewidth} |>{\centering\arraybackslash}m{0.045\linewidth} |>{\centering\arraybackslash}m{0.045\linewidth} |>{\centering\arraybackslash}m{0.055\linewidth} |>{\centering\arraybackslash}m{0.03\linewidth} |>{\centering\arraybackslash}m{0.04\linewidth}}
\hline  
	& \centering $m_{ini}$ & \centering $m_{final}$ & \centering SR$_{c_1}$ & \centering SR$_{c_2}$ & \centering SR$_{c_3}$ & \centering MD & \centering Iter & {\centering Time}\\
\hline
k-means & 3  & 3  & 54.50 & 0 & 100 & 3.8296 & 8 & 0.156 \\
k-means & 5  & 5  & 99.50 & 94 & 50.96 & 0.0843 & 35 & 0.203 \\
\hline
FCM & 3  & 3  & 56 & 0 & 100 & 3.4345 & 75 & 0.110 \\
FCM & 5  & 5  & 99.50 & 92 & 38.92 & 0.3334 & 129 & 0.375 \\
\hline
PCM & 5  & 1  & 0 & 0 & 100 & 3.7899 & 9 & 0.421 \\
PCM & 10 & 2  & 99 & 0 & 97.60 & 0.9254 & 29 & 1.943 \\
\hline
UPC ($q=1.5$) & 5  & 4  & 50 & 95 & 100 & 0.4424 & 80 & 0.328 \\
UPC ($q=1.3$) & 10  & 4  & 50 & 95 & 100 & 0.4517 & 113 & 1.186 \\
\hline
UPFC ($a=1$, $b=1$, $q=2.5$, $n=2$) & 5  & 2  & 100 & 0 & 100 & 1.1388 & 60 & 0.421 \\
UPFC ($a=5$, $b=1$, $q=2.5$, $n=2$) & 10  & 2  & 100 & 0 & 100 & 1.1346 & 151 &  2.044\\
\hline
PFCM ($K=1$, $a=1$, $b=1$, $q=1.5$, $n=1.5$) & 5  & 2  & 100 & 0 & 100 & 0.9519 & 45 & 0.343 \\
PFCM ($K=1$, $a=1$, $b=1$,  $q=1.2$, $n=1.5$) & 10  & 2  & 98.50 & 0 & 100 & 0.9575 & 61 & 1.358 \\
\hline
SPCM-L$_1$ ($\lambda=17$, $q=2$) & -  & 3  & 58.50 & 0 & 100 & 4.1291 & 9 & 0.016 \\
\hline
APCM ($\alpha=0.4$) & 5  & 3  & 100 & 100 & 100 & 0.4259 & 152 & 0.453 \\
APCM ($\alpha=0.3$) & 10  & 4  & 97 & 100 & 100 & 0.3730 & 101 & 1.274 \\
\hline
SPCM & 5  & 1  & 0 & 0 & 100 & 3.7899 & 9 & 2.359 \\
SPCM & 10  & 2  & 99.5 & 0 & 100 & 0.9100 & 16 & 8.223 \\
\hline
SAPCM ($\alpha=0.22$)  & 5  & 3  & 100 & 100 & 100 & 0.3820 & 236 & 29.83 \\
SAPCM ($\alpha=0.18$)  & 10  & 3  & 100 & 100 & 100 & 0.3312 & 145 & 22.52 \\
\hline
\end{tabular}}
\label{table:synth3}
\end{table}

{\bf Experiment 4}:
\begin{table}[htpb!]
\centering
\caption{Performance of clustering algorithms for the Iris data set.}
{\small
\begin{tabular}{>{\arraybackslash}m{0.41\linewidth} | >{\centering\arraybackslash}m{0.03\linewidth} |>{\centering\arraybackslash}m{0.05\linewidth}| >{\centering\arraybackslash}m{0.06\linewidth} |>{\centering\arraybackslash}m{0.06\linewidth} |>{\centering\arraybackslash}m{0.07\linewidth} |>{\centering\arraybackslash}m{0.05\linewidth}
|>{\centering\arraybackslash}m{0.05\linewidth}}
\hline  
	& \centering $m_{ini}$ & \centering $m_{final}$ & \centering RM & \centering SR & \centering MD & {\centering Iter} & {\centering Time}\\
\hline
k-means & 3   & 3   & 87.97 & 89.33 & 0.1271 & 3 & 0.30\\
k-means & 10  & 10  & 76.64 & 40.00 & 0.7785 & 4 & 0.13\\
\hline
FCM & 3  & 3   & 87.97 & 89.33 & 0.1287 & 19 & 0.02\\
FCM & 10 & 10  & 76.16 & 36.00 & 0.7793 & 35 & 0.02\\
\hline
PCM & 3  & 2 & 77.19 & 66.67 & 0.5428 & 19 & 0.11\\
PCM & 10 & 2 & 77.63 & 66.67 & 0.5286 & 28 & 0.11\\
\hline
UPC ($q=4$)   & 3   & 3 & 91.24 & 92.67 & 0.1438 & 26 & 0.03 \\
UPC ($q=2.4$) & 10  & 3 & 81.96 & 81.33 & 0.5569 & 150 & 0.11\\
\hline
UPFC ($a=1$, $b=5$,  $q=4$,   $n=2$) & 3  & 3 & 91.24 & 92.67 & 0.1642 & 32 & 0.03 \\
UPFC ($a=1$, $b=1.5$, $q=2.5$, $n=2$) & 10 & 3 & 81.96 & 81.33 & 0.5566 & 180 & 0.16\\
\hline
PFCM ($K=1$, $a=1$, $b=10$,  $q=7$, $n=2$) & 3  & 3 & 90.55 & 92.00 & 0.1833 & 17 & 0.03\\
PFCM ($K=1$, $a=1$, $b=1.5$, $q=2$, $n=2$) & 10 & 3 & 84.64 & 85.33 & 0.5411 & 92 & 0.05 \\
\hline
SPCM-L$_1$ ($\lambda=4.5$, $q=2$) & -  & 3 & 66.65 & 58.67 & 0.69.04 & 13 & 0.02 \\
\hline
APCM ($\alpha=3$) & 3   & 3 & 91.24 & 92.67 & 0.1406 & 26 & 0.06 \\
APCM ($\alpha=1$) & 10  & 3 & 84.15 & 84.67 & 0.4030 & 67 & 0.09\\
\hline
SPCM  & 3   & 2 & 77.19 & 66.67 & 0.4870 & 23 & 0.10 \\
SPCM  & 10  & 2 & 77.63 & 66.67 & 0.4719 & 41 & 0.47\\
\hline
SAPCM ($\alpha=2.2$) & 3   & 3 & 91.24 & 92.67 & 0.1419 & 33 & 0.18 \\
SAPCM ($\alpha=0.8$) & 10  & 3 & 84.15 & 84.67 & 0.4224 & 60 & 0.37\\
\hline
\end{tabular}}
\label{table:real1}
\end{table}
Let us consider the Iris data set (\cite{UCILib}) consisting of $N=150$, $4$-dimensional data points that form three classes, each one having 50 points. In this data set, two classes are overlapped, thus one can argue whether the true number of clusters $m$ is 2 or 3. As it is shown in Table~\ref{table:real1}, k-means and FCM work well, only if they are initialized with the true number of clusters ($m_{ini}=3$). The classical PCM and SPCM fail to end up with $m_{final}=3$ clusters, independently of the initial number of clusters. On the contrary, the UPC, the PFCM, the UPFC, the APCM and the SAPCM algorithms, after appropriate fine tuning of their parameters, produce very accurate results in terms of the RM, SR and MD metrics. However, the APCM and SAPCM algorithms detect more accurately the centers of the true clusters compared to the other algorithms. It is noted again that the main drawback of PFCM and UPFC is the requirement for fine tuning of several parameters, which increases excessively the computational load required for detecting the appropriate combination of parameters that achieves the best clustering performance. Finally, the SPCM-L$_1$ algorithm concludes also to three clusters, however with degraded clustering performance.

{\bf Experiment 5}:
%
\begin{figure}[htpb!]
\centering
{\includegraphics[width=0.95\textwidth]{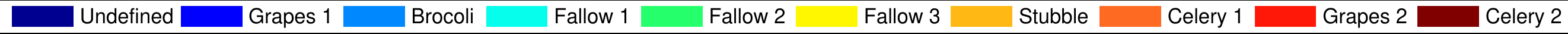}\label{salinas_legend}}
\hfil
\centering
\subfloat[{\small 4th PC component}]{\includegraphics[width=0.24\textwidth]{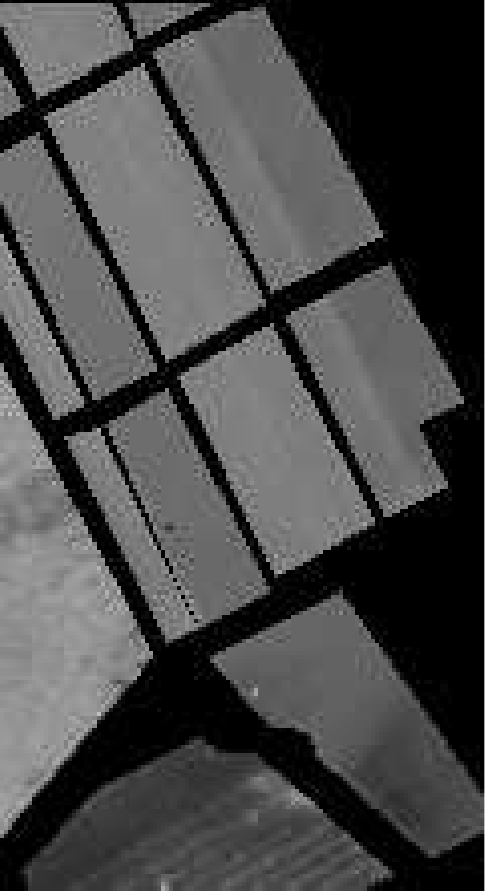}\label{salinas_4thPCA}}
\hfil
\hspace{1pt}
\centering
\subfloat[The ground truth]{\includegraphics[width=0.24\textwidth]{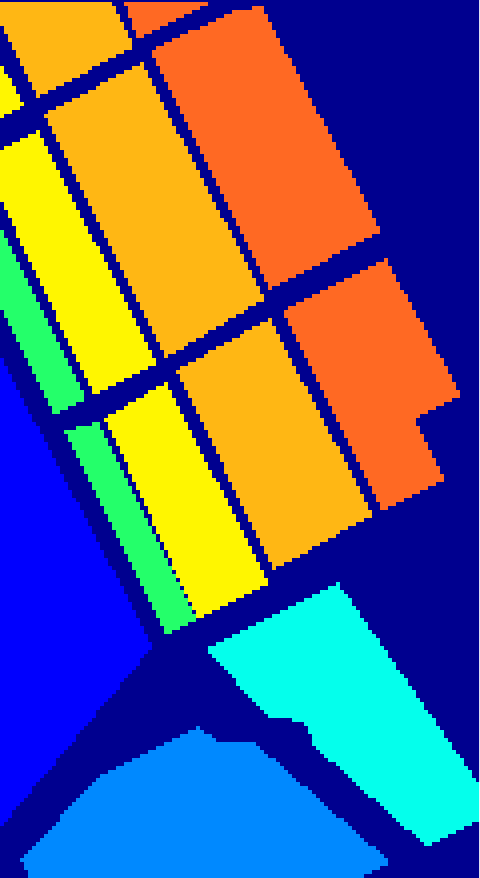}\label{salinas_gt}}
\hfil
\hspace{1pt}
\centering
\subfloat[k-means]{\includegraphics[width=0.24\textwidth]{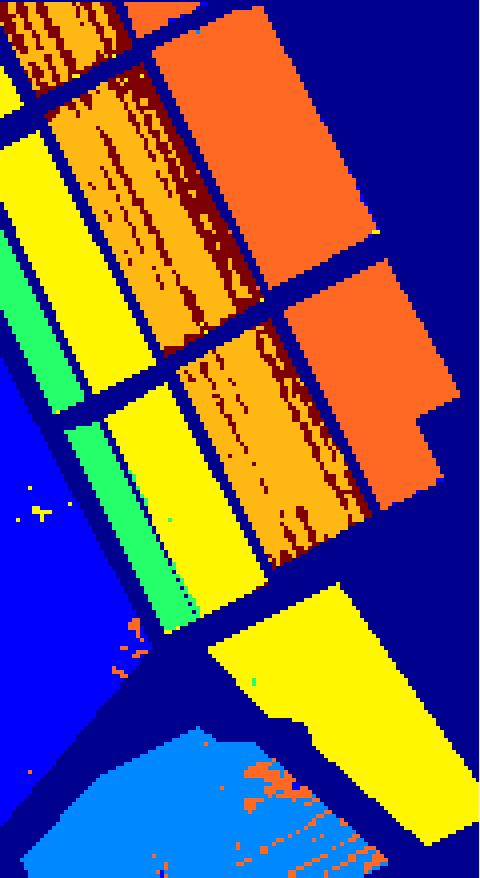}\label{salinas_kmeans_m7}}
\hfil
\hspace{1pt}
\centering
\subfloat[FCM]{\includegraphics[width=0.24\textwidth]{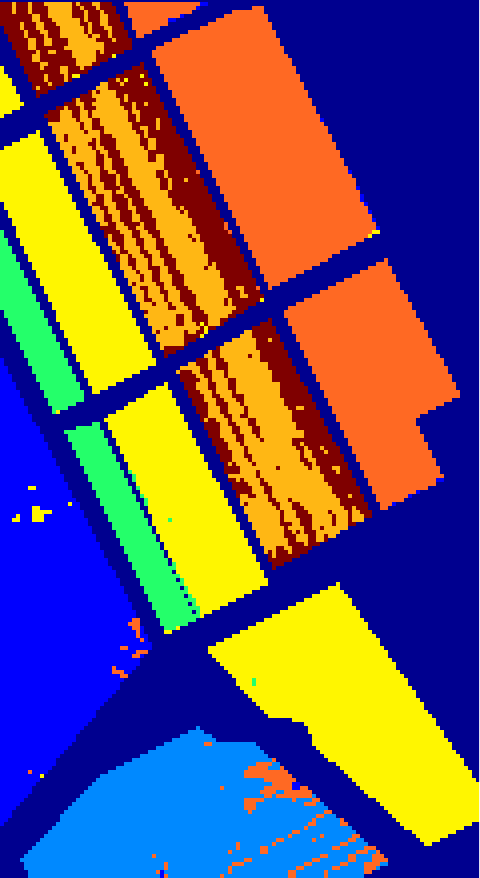}\label{salinas_FCM_m7}}
\hfil
\hspace{1pt}
\centering
\subfloat[PCM]{\includegraphics[width=0.24\textwidth]{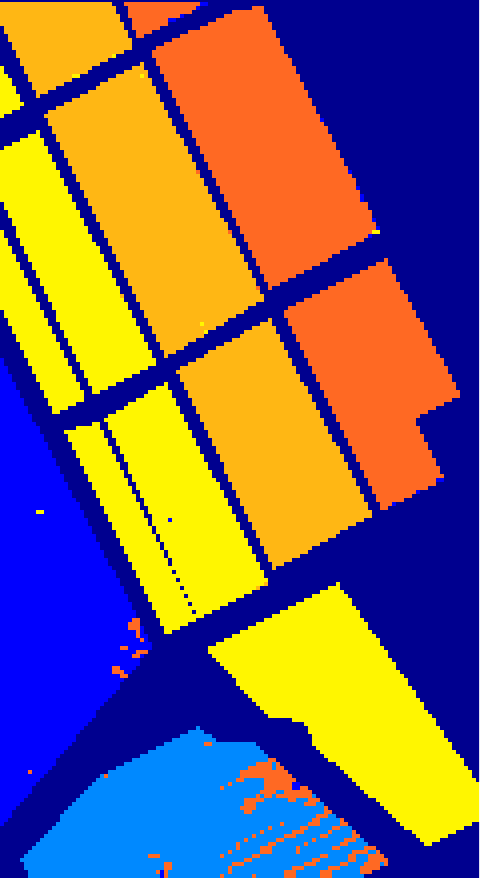}\label{salinas_PCM_m5}}
\hfil
\hspace{1pt}
\centering
\subfloat[UPC]{\includegraphics[width=0.24\textwidth]{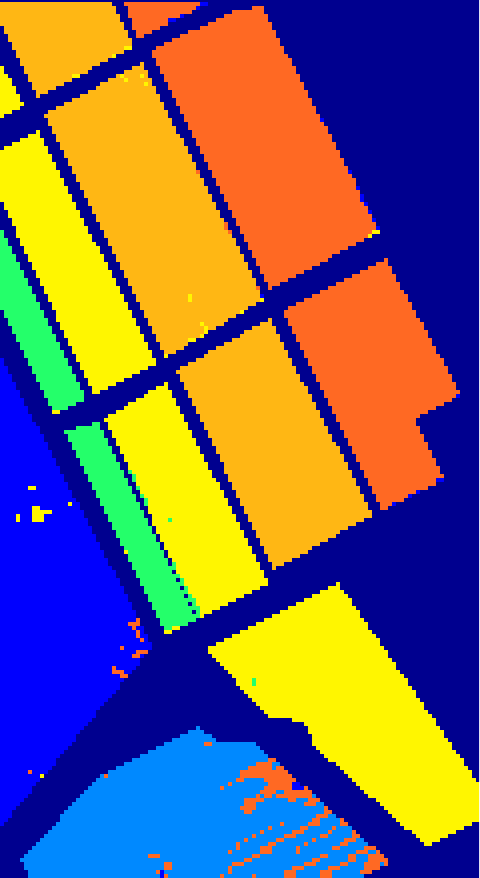}\label{salinas_UPC_m6}}
\hfil
\hspace{1pt}
\centering
\subfloat[UPFC]{\includegraphics[width=0.24\textwidth]{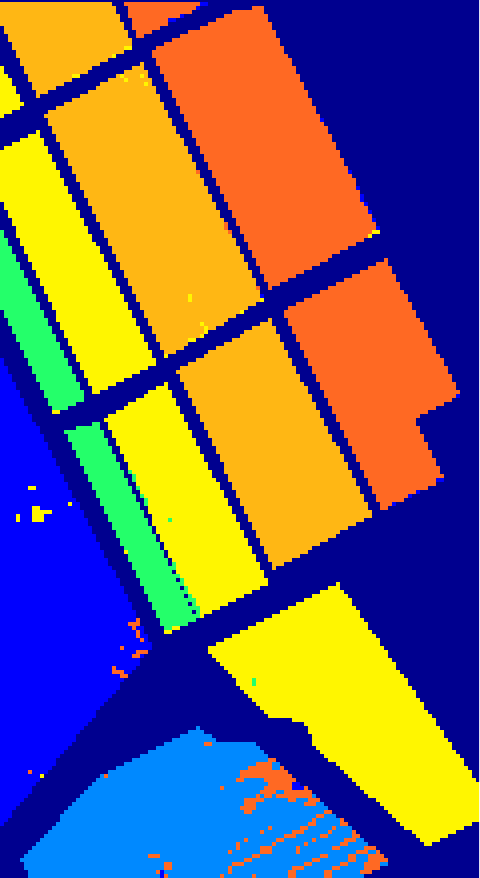}\label{salinas_UPFC_m6}}
\hfil
\hspace{1pt}
\centering
\subfloat[PFCM]{\includegraphics[width=0.24\textwidth]{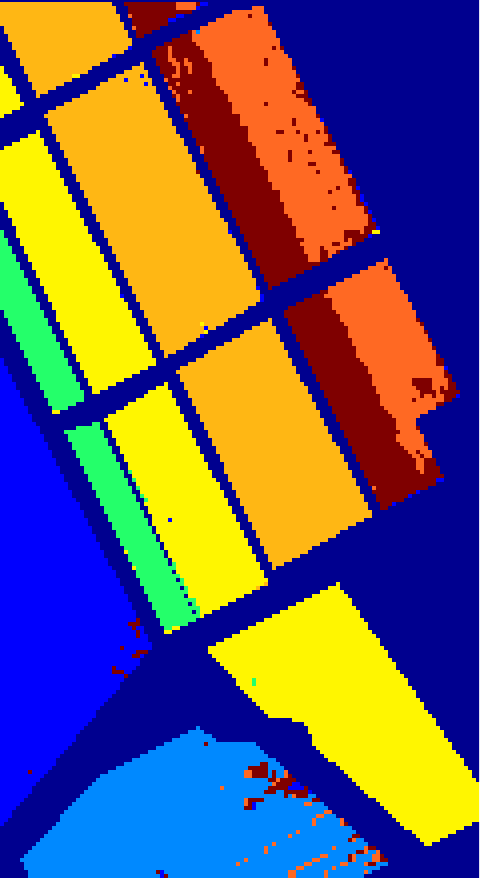}\label{salinas_PFCM_m7}}
\hfil
\hspace{1pt}
\centering
\subfloat[APCM]{\includegraphics[width=0.24\textwidth]{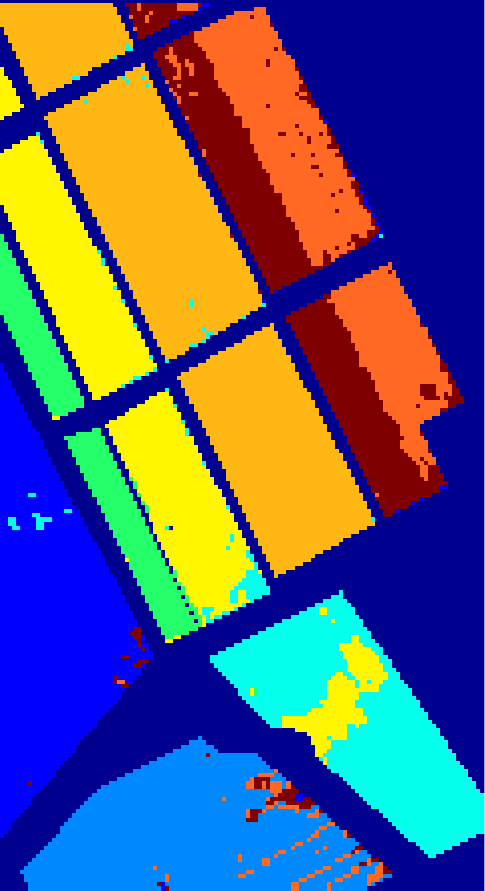}\label{salinas_APCM_m8}}
\hfil
\hspace{1pt}
\centering
\subfloat[SPCM]{\includegraphics[width=0.24\textwidth]{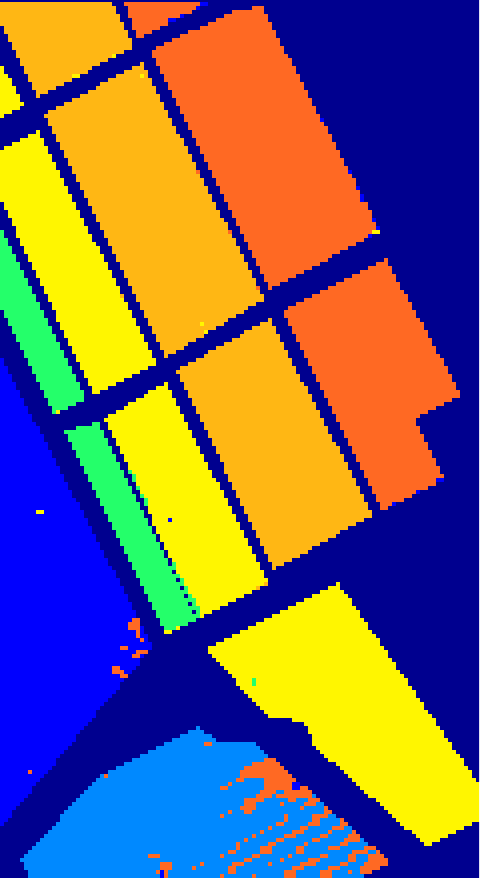}\label{salinas_SPCM_m6}}
\hfil
\hspace{1pt}
\centering
\subfloat[SAPCM]{\includegraphics[width=0.24\textwidth]{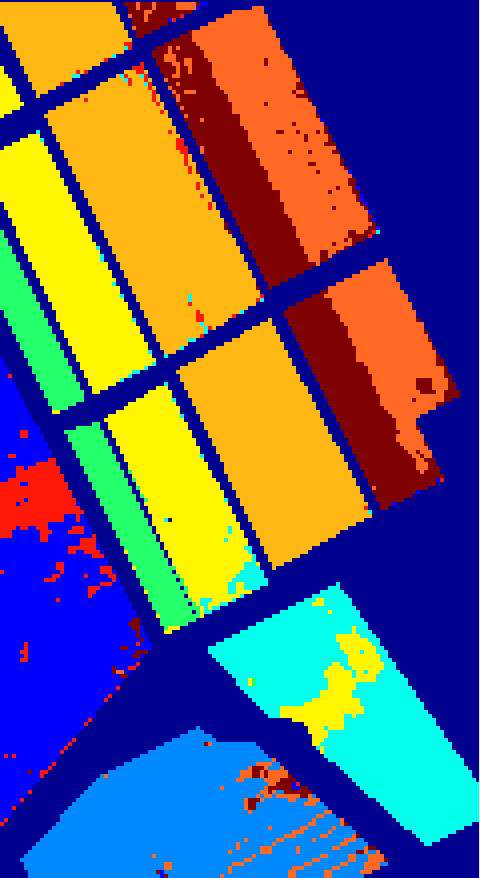}\label{salinas_SAPCM_m9}}
\hfil
\end{figure}
\addtocounter{figure}{+1}
\begin{figure}[t]
  \contcaption{{{(a) The 4th PC component of Salinas HSI and (b) the corresponding ground truth labeling. Clustering results of experiment 6 obtained from (c) k-means, $m_{ini}=7$, (d) FCM, $m_{ini}=7$, (e) PCM, $m_{ini}=15$, (f) UPC, $m_{ini}=15$ and $q=4$, (g) UFCM, $m_{ini}=15$, $\alpha=1$, $\beta=3$, $q=5$ and $n=2$, (h) PFCM, $m_{ini}=15$, $K=1$, $\alpha=1$, $\beta=2$, $q=3$ and $n=2$, (i) APCM, $m_{ini}=15$ and $\alpha=3$, (j) SPCM, $m_{ini}=30$ and (k) SAPCM, $m_{ini}=15$ and $\alpha=1.8$.}\label{exp5}}}
\end{figure}

In this experiment, the data set under study is a hyperspectral image (HSI), which depicts a subscene of size $220\times 120$ of the flightline acquired by the AVIRIS sensor over Salinas Valley, California \cite{HsiSal}. The AVIRIS sensor generates 224 spectral bands across the range from 0.2 to 2.4 $\mu m$. The number of bands is reduced to 204 by removing 20 water absorption bands. The aim in this experiment is to identify homogeneous regions in the Salinas HSI. A total size of $N=26400$ samples-pixels are used, stemming from 7 ground-truth classes: ``Grapes", ``Broccoli", three types of ``Fallow", `Stubble` and ``Celery", denoted by different colors in Fig.~\ref{salinas_gt}.  Note that there is no available ground truth information for the dark blue pixels in Fig.~\ref{salinas_gt}. It is also noted that Fig.~\ref{exp5} depicts the best mapping obtained by each algorithm\footnote{The results for the SPCM-L$_1$ algorithm were rather poor, thus they are not provided.} taking into accound not only the ``dry" performance indices, but also their physical interpretation.

\begin{table}[htpb!]
\centering
\caption{Performance of clustering algorithms for the Salinas HSI data set.}
{\small
\begin{tabular}{>{\arraybackslash}m{0.41\linewidth} | >{\centering\arraybackslash}m{0.03\linewidth} |>{\centering\arraybackslash}m{0.05\linewidth}| >{\centering\arraybackslash}m{0.04\linewidth} |>{\centering\arraybackslash}m{0.04\linewidth} |>{\centering\arraybackslash}m{0.075\linewidth} |>{\centering\arraybackslash}m{0.025\linewidth}
|>{\centering\arraybackslash}m{0.07\linewidth}}
\hline  
	& \centering $m_{ini}$ & \centering $m_{final}$ & \centering RM & \centering SR & \centering MD & {\centering Iter} & {\centering Time} \\
\hline
k-means & 7 & 7 & 93.75 & 79.89 & 1.84e+03 & 25 & 0.18e+02\\
k-means & 9 & 9 & 91.03 & 68.55 & 1.91e+03 & 10 & 0.27e+02\\
k-means & 15 & 15 & 89.90 & 59.18 & 0.60e+03 & 28 & 0.60e+02 \\
\hline
FCM & 7  & 7  & 93.18 & 75.31 & 2.41e+03 & 99 & 0.23e+02\\
FCM & 9 & 9 & 90.93 & 67.92 & 1.96e+03 & 103 & 0.31e+02\\
FCM & 15 & 15 & 89.75 & 57.89 & 0.59e+03 & 137 & 0.67e+02\\
\hline
PCM & 7  & 4 & 88.09 & 69.37 & 1.92e+03 & 28 & 0.52e+02\\
PCM & 15 & 5 & 92.75 & 80.84 & 1.21e+03 & 29 & 1.00e+02\\
PCM & 30 & 5 & 93.39 & 81.62 & 1.19e+03 & 56 & 2.90e+02\\
\hline
UPC ($q=4$) & 7  & 3 & 80.97 & 57.43 & 3.14e+03 & 38  & 0.27e+02\\
UPC ($q=4$) & 15 & 6 & 95.61 & 86.21 & 0.61e+03 & 48 &  0.85e+02\\
UPC ($q=3$) & 30 & 6 & 95.65 & 86.28 & 0.58e+03 & 48  & 2.30e+02\\
\hline
UPFC ($a=1$, $b=5$, $q=4$, $n=2$) & 7  & 3 & 80.98 & 57.43 & 3.14e+03 & 38 & 0.31e+02\\
UPFC ($a=1$, $b=3$, $q=5$, $n=2$) & 15 & 6 & 95.67 & 86.31 & 0.57e+03 & 45 & 0.93e+02 \\
UPFC ($a=1$, $b=3$, $q=5$, $n=2$) & 30 & 6 & 95.61 & 86.21 & 0.62e+03 & 54 & 2.57e+02\\
\hline
PFCM ($K=1$, $a=1$, $b=7$, $q=2$, $n=2$) & 7  & 3  & 80.98 & 57.44 & 3.06e+03 & 349 & 1.48e+02\\
PFCM ($K=1$, $a=1$, $b=2$, $q=3$, $n=2$) & 15 & 7  & 94.17 & 76.86 & 2.82e+03 & 162 & 1.86e+02\\
PFCM ($K=1$, $a=1$, $b=2$, $q=3$, $n=2$) & 30 & 7 & 93.60 & 76.63 & 2.91e+03 & 206 & 4.96e+02\\
\hline
APCM ($\alpha=4$)   & 7  & 6 & 95.45 & 85.92 & 0.72e+03 & 82 & 0.51e+02\\
APCM ($\alpha=3$)   & 15 & 8 & 95.91 & 85.85 & 0.56e+03 & 191 & 1.60e+02\\
APCM ($\alpha=1.5$) & 30 & 8 & 95.92 & 85.84 & 0.53e+03 & 262 & 3.47e+02\\
\hline
SPCM  & 7  & 5 & 92.73 & 81.19 & 1.15e+03 & 35 & 0.52e+02\\
SPCM  & 15 & 5 & 93.33 & 81.79 & 1.21e+03 & 47 & 1.51e+02\\
SPCM  & 30 & 6 & 95.62 & 86.15 & 0.48e+03 & 36 & 3.34e+02\\
\hline
SAPCM ($\alpha=2$)   & 7  & 6 & 95.85 & 86.51 & 0.71e+03 & 71 & 0.84e+02\\
SAPCM ($\alpha=1.8$) & 15 & 9 & 95.25 & 83.40 & 0.55e+03 & 223 & 3.69e+02\\
SAPCM ($\alpha=1.3$) & 30 & 9 & 95.20 & 83.31 & 0.56e+03 & 286 & 6.67e+02\\
\hline
\end{tabular}}
\label{table:HSI}
\end{table}

As it can be deduced from Fig.~\ref{exp5} and Table~\ref{table:HSI}, when k-means and FCM are initialized with $m_{ini}=7$, they actually split the ``Stubble" class into two clusters and merge the ``Fallow 1" and ``Fallow 3" classes. The PCM algorithm fails to uncover more than 5 discrete clusters, merging the three different types of the ``Fallow" class. The UPC, UPFC and SPCM algorithms are able to detect up to 6 clusters, merging the ``Fallow 1" and ``Fallow 3" classes. PFCM, after exhaustive fine tuning of its parameters, manages additionally to distinguish two types of ``Celery", compared to UPC, UPFC and SPCM, although this information is not reflected to the ground-truth labeling. Finally, APCM and SAPCM are the only algorithms that manage to distinguish the ``Fallow 1" from the ``Fallow 3" class, while at the same time they do not merge any other of the existing classes. 

Let us focus for a while on the ``Celery" class. This class forms two closely located clusters in the feature space, although this is not reflected to the ground-truth labeling (note however that this can be deduced from the 4th PC component in Fig.~\ref{salinas_4thPCA}). It is important to note that, in contrast to PFCM, APCM and SAPCM, none of the other algorithms succeeds in identifying each one of them. The fact that this is not reflected in the ground-truth labeling causes a misleading decrease in the SR performance of these three algorithms. The same holds for the ``Grapes" class, which also appears in the 4th PC component in Fig.~\ref{salinas_4thPCA}. However, only the SAPCM algorithm succeeds in identifing this one. 
 
\section{Conclusion}
In this paper two novel possibilistic c-means algorithms are proposed, namely SPCM and SAPCM, which both impose a sparsity constraint on the degrees of compatibility of each data vector with the clusters. Both algorithms are initialized through FCM with the latter executed for an overestimated number of the actual number of clusters. SPCM, which results by extending the cost function of the original PCM with a sparsity promoting term, unravels the underying clustering structure much more accurately than PCM. This is achieved via the improvement on the estimation of the cluster representatives by excluding points that are distant from them in contributing to their estimation. Thus, it is able to identify closely located clusters with possibly different densities. In addition, SPCM exhibits immunity to noise and outliers. The second algorithm, termed SAPCM, further extends SPCM by adapting the parameters $\gamma_j$'s as the algorithm evolves, incorporating the relative adaptation mechanism described in \cite{Xen15}. The SAPCM algorithm is immune to noise/outliers, as its predecessor SPCM. In addition, SAPCM has the ability to improve even more the estimates of the cluster representatives, compared to SPCM, and in addition is capable of detecting the number of natural clusters. In extensive experiments, it is shown that SAPCM has a steadily superior performance, compared to other related algorithms, irrespective of the initial estimate of the number of clusters. Also, in principle, it has the ability to deal well with very closely located clusters of different variances and/or densities. Both algorithms compare favourably with relevant state-of-the-art algorithms, exhibiting in most cases a superior clustering performance. Finally, note that convergence issues of SPCM are considered in \cite{Xen15con}.
%


%

\appendices
\section{}
\begin{proof}[Proof of Proposition \ref{propstatpoint}]
Taking the derivative of $f(u_{ij})$ with respect to $u_{ij}$, we obtain
\begin{equation}
\frac{\partial f(u_{ij})}{\partial u_{ij}} = \gamma_ju_{ij}^{-1}\left[1-\frac{\lambda}{\gamma_j}p(1-p)u_{ij}^{p-1}\right].
\label{df2u}
\end{equation}
Solving $\frac{\partial f(u_{ij})}{\partial u_{ij}}=0$ with respect to $u_{ij}$ and taking into account that $u_{ij}>0$ (by definition), after some elementary algebraic manipulations we have the following solutions
\begin{equation}
\hat{u}_{ij} = {\left[{\frac{\lambda}{\gamma_j}p(1-p)}\right]}^{\frac{1}{1-p}} \ \text{and }\  \tilde{u}_{ij}=+\infty.
\label{u_hat}
\end{equation}
\end{proof}
\begin{proof}[Proof of Proposition \ref{propuniqmin}]
It suffices to show that $\frac{\partial f(u_{ij})}{\partial u_{ij}}\leq 0$ for $u_{ij}\in(0, \hat{u}_{ij}]$ and $\frac{\partial f(u_{ij})}{\partial u_{ij}}\geq 0$ for $u_{ij}\in[\hat{u}_{ij}, +\infty)$. Indeed, for $u_{ij}\in(0, \hat{u}_{ij}]$ we have $u_{ij}\leq\hat{u}_{ij}$, which implies that $u_{ij}^{1-p}\leq\frac{\lambda}{\gamma_j}p(1-p)$ (from eq.~(\ref{u_hat})) or $1\leq\frac{\lambda}{\gamma_j}p(1-p)u_{ij}^{p-1}$. From the latter and taking into account eq.~(\ref{df2u}) again, it follows that $\frac{\partial f(u_{ij})}{\partial u_{ij}}\leq 0$ in $u_{ij}\in(0, \hat{u}_{ij}]$. Similarly, for $u_{ij}\in[\hat{u}_{ij}, +\infty)$ we have $u_{ij}\geq\hat{u}_{ij}$, which, utilizing eq.~\eqref{u_hat}, implies that $u_{ij}^{1-p}\geq\frac{\lambda}{\gamma_j}p(1-p)$ or $1\geq\frac{\lambda}{\gamma_j}p(1-p)u_{ij}^{p-1}$. From the latter and taking into account eq.~(\ref{df2u}), it follows that $\frac{\partial f(u_{ij})}{\partial u_{ij}}\geq 0$ in $u_{ij}\in[\hat{u}_{ij}, +\infty)$. Consequently, $\hat{u}_{ij}$ is the unique minimum of $f(u_{ij})$, since in $[\hat{u}_{ij}, +\infty)$, $f(u_{ij})$ is increasing and, as a consequence, $\tilde{u}_{ij}$ is not a minimum of $f(u_{ij})$.
\end{proof}
\begin{proof}[Proof of Proposition \ref{propfpos}]

It is $f(1)=d_{ij}+\gamma_j\ln{1}+\lambda p\cdot 1^{p-1}=d_{ij}+\lambda p>0$. Moreover, it is $f(0)=\lim\limits_{u_{ij}\rightarrow 0^+}f(u_{ij})=\lim\limits_{u_{ij}\rightarrow 0^+}\left(d_{ij}+\gamma_j\ln{u_{ij}}+\lambda pu_{ij}^{p-1}\right)=d_{ij}+\lim\limits_{u_{ij}\rightarrow 0^+}\left[\frac{1}{u_{ij}^{1-p}}\left(\gamma_ju_{ij}^{1-p}\ln{u_{ij}}+\right.\right.$ $\left.\left.\lambda p\right)\right]=+\infty$, as it follows from the application of the L' Hospital rule, since $\lim\limits_{u_{ij}\rightarrow 0^+}\frac{1}{u_{ij}^{1-p}}=+\infty$ and $\lim\limits_{u_{ij}\rightarrow 0^+}\left(\gamma_ju_{ij}^{1-p}\ln{u_{ij}}\right)=0$.

Taking into account (a) that $f(0)>0$ and $f(\hat{u}_{ij})<0$, (b) the continuity of $f(u_{ij})$ and (c) the Bolzano's theorem, there is at least one $u_{ij}^1\in (0,\hat{u}_{ij}): f(u_{ij}^1)=0$. Moreover, based on Proposition \ref{propuniqmin}, $\frac{\partial f(u_{ij})}{\partial u_{ij}}<0$ for $u_{ij}\in(0, \hat{u}_{ij})$, thus $f(u_{ij})$ is decreasing on $(0,\hat{u}_{ij})$. Therefore, there is exactly one $u_{ij}^1\in (0,\hat{u}_{ij}): f(u_{ij}^1)=0$. Similarly, taking into account (a) that $f(\hat{u}_{ij})<0$ and $f(1)>0$, (b) the continuity of $f(u_{ij})$ and (c) the Bolzano's theorem, there is at least one $u_{ij}^2\in (\hat{u}_{ij},1): f(u_{ij}^2)=0$. Moreover, based on Proposition \ref{propuniqmin}, it is $\frac{\partial f(u_{ij})}{\partial u_{ij}}>0$ for $u_{ij}\in(\hat{u}_{ij},1)$, thus $f(u_{ij})$ is increasing on $(\hat{u}_{ij},1)$. Therefore, there is exactly one $u_{ij}^2\in (\hat{u}_{ij},1): f(u_{ij}^2)=0$. Consequently, there are exactly two $u_{ij}^1,u_{ij}^2\in (0,1)$ such that $f(u_{ij})=0$.
%
\end{proof}
\begin{proof}[Proof of Proposition \ref{proplargsol}]
As previously mentioned, if $f(u_{ij})=0$ has two solutions, then $f(\hat{u}_{ij})<0$. From Proposition~\ref{propfpos}, it is $u_{ij}^1<\hat{u}_{ij}<u_{ij}^2$.
Since $\frac{\partial f(u_{ij})}{\partial u_{ij}}\leq 0$ for $u_{ij}\in(0, \hat{u}_{ij}]$ and $\frac{\partial f(u_{ij})}{\partial u_{ij}}\geq 0$ for $u_{ij}\in[\hat{u}_{ij}, +\infty)$ (proof of Proposition \ref{propuniqmin}), it turns out that $f(u_{ij})$ is decreasing for $u_{ij}\in(0, \hat{u}_{ij}]$ and increasing for $u_{ij}\in[\hat{u}_{ij}, +\infty)$. 
In addition, it can be easily verified that $f(0)\geq 0$ and $f(+\infty)\geq 0$. Taking into account these facts, the continuity of $f$ and the fact that $f(u_{ij}^1)=f(u_{ij}^2)=0$, it follows that $f(u_{ij})$ is positive for $u_{ij}\in(0, u_{ij}^1)\cup(u_{ij}^2, +\infty)$ and negative for $u_{ij}\in(u_{ij}^1, u_{ij}^2)$. Thus, $u_{ij}^2$ is a turning point for $J_{SPCM}(\Theta,U)$ before which $J_{SPCM}(\Theta,U)$ decreases with respect to $u_{ij}$ and after which $J_{SPCM}(\Theta,U)$ increases with respect to $u_{ij}$. Therefore, $u_{ij}^2$ is a local minimum of $J_{SPCM}(\Theta,U)$, whereas, employing similar reasoning, it turns out that $u_{ij}^1$ is a local maximum of $J_{SPCM}(\Theta,U)$.
\end{proof}
\begin{proof}[Proof of Proposition \ref{propglobsol}]
Let $J_{SPCM}(\boldsymbol{\theta}_j,u_{ij})$ contain the terms of $J_{SPCM}(\Theta,U)$ that involve $\boldsymbol{\theta}_j$, $u_{ij}$. According to Propositions~\ref{propuniqmin},~\ref{propfpos} and~\ref{proplargsol}, it turns out that if $f(\hat{u}_{ij})<0$, then the global minimum of $J_{SPCM}(\boldsymbol{\theta}_j,u_{ij})$ with respect to $u_{ij}$ is $u^2_{ij}$, provided that $J_{SPCM}(\boldsymbol{\theta}_j,u^2_{ij})<J_{SPCM}(\boldsymbol{\theta}_j,0)$. However, the latter becomes $u^2_{ij}\left[d_{ij}+\gamma_j\ln{u^2_{ij}}-\gamma_j+\lambda (u^2_{ij})^{p-1}\right]<0$ and taking into account that $f(u^2_{ij})\equiv d_{ij}+\gamma_j\ln{u^2_{ij}}+\lambda p(u^2_{ij})^{p-1}=0$, it is equivalent to $u^2_{ij} \left[-\lambda p(u^2_{ij})^{p-1}\right.$ $\left.-\gamma_j + \lambda (u^2_{ij})^{p-1} \right]<0$ or $u^2_{ij}>\left(\frac{\lambda(1-p)}{\gamma_j}\right)^\frac{1}{1-p}$. Clearly, in the case where $f(\hat{u}_{ij})<0$ and $u^2_{ij}<\left(\frac{\lambda(1-p)}{\gamma_j}\right)^\frac{1}{1-p}$, it is $u_{ij}=0$. Finally, in the case where $f(\hat{u}_{ij})>0$, it is $f(u_{ij})>0$, for $u_{ij}\in(0, +\infty)$. Thus, $J_{SPCM}(\Theta,U)$ increases with respect to $u_{ij}$ in $(0, +\infty)$ and, as a consequence, its minimum is achieved at $u_{ij}=0$.
\end{proof}

\ifCLASSOPTIONcaptionsoff
  \newpage
\fi

\end{document}